\documentclass{article} 
\usepackage{iclr2027_conference,times}

\usepackage{amsmath,amsfonts,bm}

\def\eqref#1{equation~\ref{#1}}
\def\Eqref#1{Equation~\ref{#1}}

\def\1{\bm{1}}

\DeclareMathAlphabet{\mathsfit}{\encodingdefault}{\sfdefault}{m}{sl}
\SetMathAlphabet{\mathsfit}{bold}{\encodingdefault}{\sfdefault}{bx}{n}

\usepackage{hyperref}
\usepackage{graphicx}
\usepackage{url}
\usepackage{subfigure}
\usepackage{amsthm}
\usepackage{booktabs}
\usepackage{wrapfig}
\usepackage{xcolor}
\usepackage{fancybox}
\cornersize*{8pt}
\newtheorem{claim}{Claim}

\theoremstyle{remark}
\newtheorem{remark}{Remark}

\title{Vision-Language-Action Autonomous Driving Agent with Language-based Memory}

\author{
Kai Yan$^{1}$\thanks{Corresponding author: \texttt{kaiyan3@illinois.edu}.
Work done as an intern at NVIDIA.},
Xiangyu Chen$^{2}$, Yulong Cao$^{2}$, Alex Naumann$^{2}$, Peter Karkus$^{2}$, Yan Wang$^{2}$,\\
\textbf{Jef Packer$^{2}$, Alex Schwing$^{1}$, Yuxiong Wang$^{1}$, Boris Ivanovic$^{2}$,
Wenjie Luo$^{2}$, Marco Pavone$^{2}$} \\
$^{1}$University of Illinois Urbana-Champaign, $^{2}$NVIDIA\\
\texttt{https://kaiyan289.github.io/projects/ad-memo/}
}

\iclrfinalcopy 
\begin{document}

\maketitle

\begin{abstract}

Vision-Language-Action (VLA) foundation models have recently emerged as one of the prevailing solutions for autonomous driving, as they can utilize knowledge acquired during vision-language pretraining for accurate and interpretable driving. However, VLAs can take only a limited number of frames as visual input due to the high token cost of an image, which is problematic for memory-dependent tasks such as determining the arrival order at all-way stops and long-horizon driving scene understanding. Existing solutions use latent vector memories accessed through cross-attention, which are neither interpretable nor portable. In this paper, we propose AD-Memo, a general-purpose VLA driving agent with language-based memory. The agent outputs memory as an extension of its Chain-of-Thought (CoT) to record surrounding objects critical to driving; this memory becomes part of the agent's future input. We curate memory-based datasets and train VLAs with a two-stage recipe: Supervised Fine-Tuning (SFT) and \textit{Da Capo}, a novel semi-closed-loop Reinforcement Learning (RL) algorithm which uses trajectory-level advantage for memory and step-level advantage for driving, leading to better credit assignment. Across scenarios such as all-way stops and general driving, AD-Memo improves driving quality, enables better question answering on driving scenes, and produces plug-and-play memory for other models.
\end{abstract}

\section{Introduction}


In recent years autonomous driving has successfully transitioned from the lab to real-world commercial deployment at scale~\citep{kusano2025comparison,waymo2026scale}. During this period, Vision-Language-Action (VLA) models~\citep{brohan2023rt,o2024open} have rapidly emerged as one of the prevailing solutions~\citep{wang2025alpamayo,li2026unidrivevla,zhou2026qwen} for autonomous driving, as they acquire knowledge through large-scale vision-language pretraining and apply it after fine-tuning on driving data. Through this architectural synergy, VLA driving agents produce accurate and interpretable driving trajectories in an end-to-end manner.


However, the high token cost of multiple images and the quadratic computational cost of attention~\citep{dao2022flashattention} severely limit the time horizon of VLA inputs~\citep{shi2026memoryvla}. For example, for a $4$-camera VLA receiving visual input at $10$Hz with $200$ tokens per image, a $0.4$s history requires $0.4/0.1 \times 200 \times 4 = 3200$ tokens; expanding this history to $4$ seconds increases the input length to $32$K tokens and the quadratic attention cost by a factor of $100$, which is prohibitively expensive for time-sensitive autonomous driving tasks~\citep{liu2022understanding}. While one can simply shorten the input horizon for faster inference, the lack of historical context becomes problematic for memory-dependent driving tasks, such as determining the arrival order at all-way stops~\citep{liang2026planning}, identifying occluded road objects~\citep{kumar2025occluded,xie2025drivebench}, and long-horizon driving scene understanding~\citep{zeng2025vision,zeng2026streamforest} --- all of which are important for driving.


To address this issue, prior works have introduced memory mechanisms into VLA driving agents using two prevalent approaches: \textit{latent vectors}~\citep{fu2025orion,huang2026mindvla} and \textit{language in-context learning}~\citep{wen2024dilu, mei2024continuously}. The former stores latent vectors extracted from past inputs and accesses them via cross-attention~\citep{vaswani2017attention}, while the latter stores structured examples of prior driving experiences. However, both approaches have shortcomings. Latent vector memory is not interpretable by humans or portable to other agents; it does not make full use of the VLA's reasoning ability inherited from its backbone, which motivates the use of VLAs in the first place. In turn, the language in-context learning approaches do not provide \textit{in-episode} memory, i.e., they do not help retain information about critical objects or events that affect current driving. Thus, in this work we ask: \textit{Can we fully exploit the knowledge in a VLA's backbone by using language-based, in-episode memory that is brief, explainable, and portable?}

\begin{figure}
    \centering
    \includegraphics[width=0.9\linewidth]{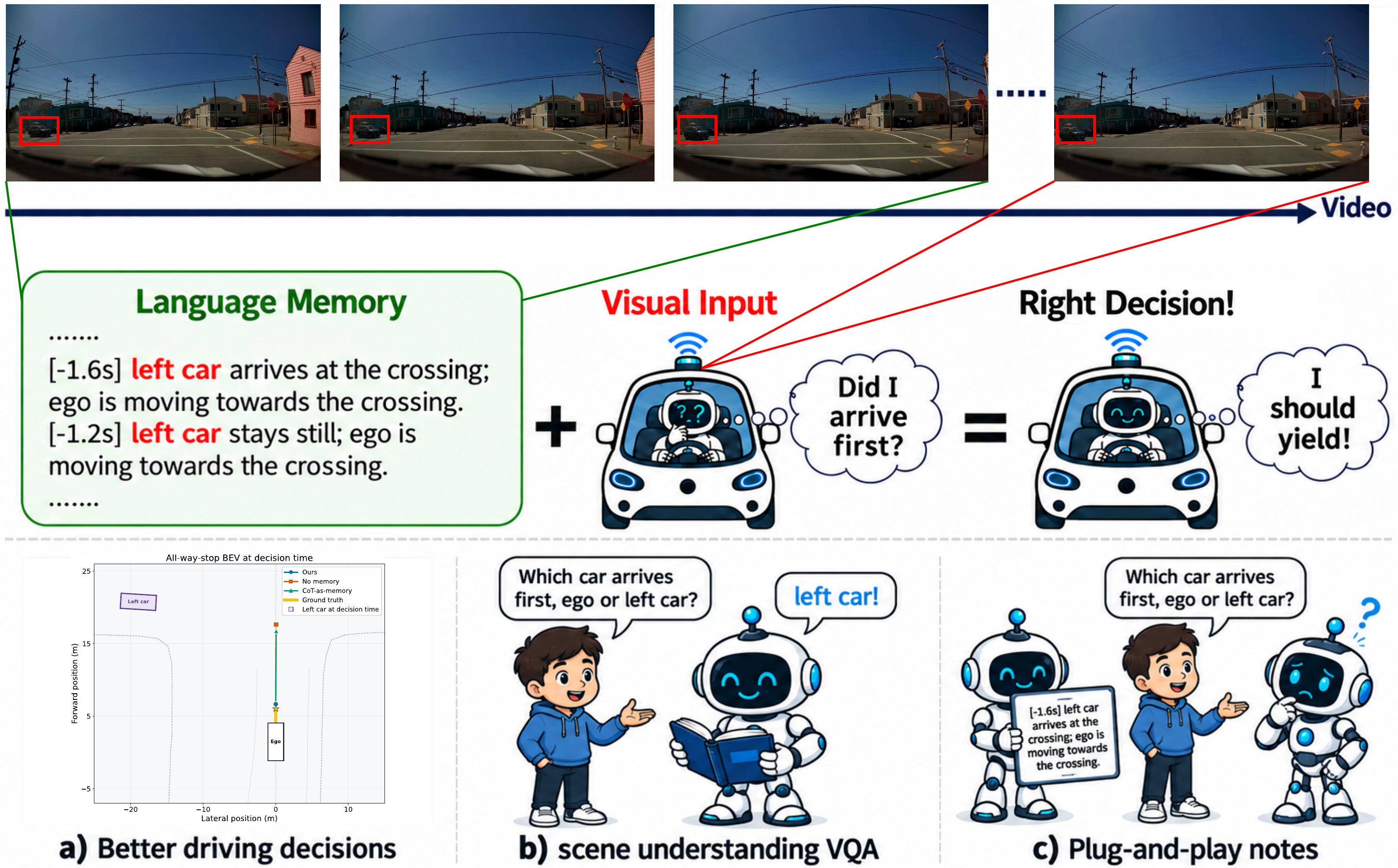}
    \caption{An illustration of AD-Memo. AD-Memo generates language-based memory periodically, which is beneficial for a diverse set of tasks illustrated in panel (a), (b) and (c), including driving, scene understanding question-answering tasks, and producing portable memory for other models. 
    }
    \label{fig:teaser}
\end{figure}


To achieve this, we propose AD-Memo, a general-purpose \textbf{A}utonomous \textbf{D}riving VLA agent with language-based, in-episode \textbf{memo}ry (see Fig.~\ref{fig:teaser} for an overview). The agent outputs memory as an extension of its Chain-of-Thought (CoT)~\citep{wei2022chain}, which then becomes part of the agent's future input. AD-Memo acquires this ability through a two-stage training recipe: Supervised Fine-Tuning (SFT) and \textit{Da Capo}, a novel semi-closed-loop RL algorithm. \textit{Da Capo} has two features: (1) it uses open-loop RL (i.e., ground-truth trajectory replay) for driving states to avoid the high cost of environment simulation, while using closed-loop memory (i.e., the agent's own memory generated at past steps) to overcome exposure bias~\citep{Bengio2015ScheduledSampling}; (2) inspired by stochastic computation graph theory~\citep{schulman2015gradient}, it assigns different reward components to driving tokens, which do not affect future steps under open-loop control, and memory tokens which do affect future steps, improving credit assignment. To apply this recipe, we curate a task-specific (all-way stop) dataset and a task-agnostic general driving dataset using rule-based and agentic frameworks, respectively. We then evaluate our model on the corresponding test sets, showing that AD-Memo not only improves driving quality but also better understands driving scenes, answers related questions, and generates plug-and-play memories that help other models with driving-related tasks.


Our contributions are as follows: (1) We introduce AD-Memo, the first general-purpose VLA driving agent that uses language-based, in-episode memory, which is scalable, explainable, and portable; (2) we propose \textit{Da Capo}, a novel semi-closed-loop RL algorithm that combines the fine-grained, per-step driving reward signal with the future-dependent, per-episode memory reward signal, overcoming exposure bias in the agent's memory while avoiding the expensive environment simulation required for fully closed-loop RL training; (3) we build novel agentic data pipelines and curate challenging, memory-dependent datasets; (4) we evaluate AD-Memo across multiple driving scenarios and empirically demonstrate that AD-Memo can not only improve driving performance, but also aid scene understanding and serve as a plug-and-play module for driving-related reasoning tasks.
\section{Preliminaries}

\textbf{Driving Vision-Language-Action (VLA) Models.} In this paper, we consider VLA driving agents with three input modalities: (1) \textit{past trajectory} $\tau^{\text{past}}=(\tau^1,\tau^2,\dots,\tau^{M_1})$, a special token sequence describing the waypoints of the past ego trajectory; (2) \textit{visual input} $I=(I^1, I^2,\dots,I^{M_2})$, a sequence of past camera images ordered chronologically; and (3) \textit{language input} $L=(L^1, L^2, \dots, L^{M_3})$, a sequence of language tokens. The agent $\pi_\theta$ first outputs a Chain-of-Thought (CoT)~\citep{wei2022chain} and then a future trajectory $\tau^{\text{future}}$. We let $y$ denote the combination of CoT and $\tau^{\text{future}}$. The primary goal of the agent is to minimize the difference between the human expert's future trajectory $\tau^{\text{gold}}$ and its prediction $\tau^{\text{future}}$ (see Appendix~\ref{sec:metrics} for detailed metrics) on video \textit{clips} $V=\{(\tau^{\text{past}}_1, I_1,L_1,y_1), (\tau^{\text{past}}_2,I_2,L_2,y_2),\dots, (\tau^{\text{past}}_n,I_n,L_n,y_n)\}$, each with $n$ steps corresponding to temporally ordered frames. Besides driving, we also expect the model to develop a good understanding of the scene and support multiple tasks. Thus, we consider two other types of tasks in this paper: (1) \textit{memory writing}, where the model takes a single image $I$ and past memory $L$ as input and outputs plug-and-play language memory $y$ that helps other Vision-Language Models (VLMs) understand the driving scene; (2) \textit{Visual Question Answering (VQA)}, where the model takes the same visual input $I$ as in driving, optional language input $L$, and a natural-language question $Q$, then generates a language answer $y^{\text{VQA}}$ about the driving history.


\textbf{Open-Loop and Closed-Loop RL for VLAs.} Open-loop RL~\citep{zhou2026autovla,xiong2026recogdrive} is an RL training paradigm in which the action at the current step does not affect future environmental states in the trajectory, i.e., ground-truth inputs from expert demonstrations are provided at every step. In contrast, closed-loop RL~\citep{shi2026clear} resembles MuJoCo-based RL tasks~\citep{todorov2012mujoco,yan2024reinforcement}, in which the future state depends on the current action. While the latter can better address exposure bias~\citep{Bengio2015ScheduledSampling} and enable more effective exploration~\citep{li2026simplevla}, the former remains a prevalent choice in autonomous driving~\citep{zhou2026autovla,sheng2026explorevla} and robotics~\citep{huang2026thinkact,huang2026mobilevla} due to the difficulty and high cost of simulating real-world environments~\citep{wang2024gensim,choi2026scale}. Our work strikes a balance between the two with semi-closed-loop RL, which combines open-loop replay of visual input $I$ and past trajectory $\tau^{\text{past}}$ with closed-loop memory $y$ generated from previous steps.


\textbf{Group Relative Policy Optimization (GRPO).} GRPO~\citep{shao2024deepseekmath} is a prevalent Reinforcement Learning with Verifiable Rewards (RLVR) method for building Large Language Model (LLM) agents. Given an input $x$ ($x=(\tau^{\text{past}},I,L)$ for driving, $(I,L,Q)$ for question-answering, and $(I,L)$ for memory writing) and a group of model responses $y_1,y_2,\dots,y_n$, GRPO optimizes the policy $\pi_\theta$ by minimizing
%
\begin{equation}
\label{eq:grpo}
\begin{aligned}
&-\frac{1}{n}\sum_{i=1}^{n}\frac{1}{|y_i|}
\sum_{j=1}^{|y_i|}
\min\!\bigl[
\rho_{i,j}A_i,\,
\operatorname{clip}(\rho_{i,j},1-\epsilon_1,1+\epsilon_2)A_i
\bigr]
+\beta\operatorname{KL}(\pi_\theta\|\pi_{\mathrm{old}}),
\end{aligned}
\end{equation}
%
%
where 
$\rho_{i,j}
=\frac{\pi_\theta(y_{i,j}\mid x,y_{i,<j})}
{\pi_{\theta_{\mathrm{old}}}(y_{i,j}\mid x,y_{i,<j})}$, $\beta\geq 0$ and $\epsilon_2>\epsilon_1>0$ following DAPO~\citep{yu2025dapo}. Here, $\pi_{\theta_{\text{old}}}$ is the old policy used to generate responses in the current sampling-training iteration, and $\pi_{\text{old}}$ is the fixed reference policy before GRPO training. The advantage $A_i$ is calculated relative to the group mean as $A_{i}=\frac{r(x,y_i)-\text{avg}(r(x,y_k))}{\text{std}(r(x,y_k))}$, where $\text{avg}$ and $\text{std}$ are the mean and standard deviation over $k\in\{1,2,\dots,n\}$.
\section{Methodology}

The proposed AD-Memo, a general-purpose VLA driving agent, learns driving, question-answering and memory-writing jointly. Sec.~\ref{sec:lbm} describes how the language-based memory works, while Sec.~\ref{sec:training} describes our training recipe.

\subsection{Language-Based Memory}
\label{sec:lbm}


At a high level, AD-Memo is a \textit{streaming video agent} that uses the current visual input to summarize objects of interest that could affect future driving. This design differs fundamentally from existing VLA approaches with language-based memory in robotics~\citep{haresh2026self,torne2026mem}, where notes mainly record the agent's state and subgoals. This is because \textit{in robotics, objects are often relatively stable but tasks are heterogeneous:} For example, making a cup of tea requires sequentially completing very different tasks, such as grabbing the cup, boiling water, and pouring from the kettle, but these objects do not move by themselves. In contrast, \textit{in autonomous driving, objects are often transient but tasks are homogeneous:} pedestrians ahead may disappear from view within seconds, but the task remains to output a safe and efficient trajectory --- regardless of whether the agent passed an intersection a minute ago. Thus, in this work, we focus on a memory horizon of around 10 seconds, which, as shown in Appendix~\ref{sec:datastats-aws} largely suffices for many scenarios of interest, but remains expensive to cover with visual input alone, requiring at least $100$ images at $10$Hz. We discuss the potential need for even longer-term memory, e.g., remembering speed limits, in Appendix~\ref{sec:limit}.

To expand AD-Memo's memory to the horizon discussed above, we divide the memory horizon by the visual input horizon $t$, and generate language-based memory for each interval such that the content of the memory covers the video without overlap. More specifically, suppose we have a driving clip $V$ with $n+1$ keyframes spaced $t$ seconds apart.\footnote{In this paper, we focus on performance at keyframes; for intermediate frames, we can simply reuse the memory input used at the previous keyframe and discard memory generated at non-keyframes.} On those keyframes, AD-Memo switches between two modes: driving and VQA. At the $i$-th keyframe ($i\in\{1,2,\dots,n\}$), AD-Memo operates in driving mode, in which it generates Chain-of-Thought (CoT)~\citep{wei2022chain} and a future trajectory $\tau^{\text{future}}_i$ for driving control. Within CoT, AD-Memo outputs \textit{memory} $\text{mem}_i$, which is a formatted substring appended to the input of future keyframes and can be also used by other models as input. At the $(n+1)$-th keyframe, i.e., the end of the clip, AD-Memo switches to VQA mode, in which it takes the current visual input, memories from a window of up to $T$ previous steps $\{\text{mem}_{\max(1,n+1-T)},\dots,\text{mem}_n\}$, and a question $Q$, then generates an answer. The question concerns a crucial object of interest that affects driving in the video (e.g., a pedestrian or traffic light) and is presented as a Multiple-Choice Question (MCQ) for verifiability. As AD-Memo does not have access to $Q$ during driving, it must identify and record all important objects on the road to answer the question, which also helps its VLM backbone to reason about and improve driving. When deployed on the road, AD-Memo remains in driving mode but continues recording important objects on the road, as it does when preparing for questions during training.


As no public autonomous driving dataset with language-based memory annotations exists, we curate two datasets from NVIDIA's Physical-AV dataset~\citep{nvidia2026physicalaiav}: the all-way stop dataset (Sec.~\ref{sec:aws_exp}) and the general driving dataset (Sec.~\ref{sec:gen_exp}), where the former provides a task-specific setting in which memory is beneficial, while the latter serves to evaluate generalizability. Each dataset is divided into three splits: the SFT training set $\mathcal{D}_{\text{SFT}}$, the RL training set $\mathcal{D}_{\text{RL}}$, and the test set $\mathcal{D}_{\text{test}}$, in a $3:1:1$ ratio. Each dataset consists of triplets $(V, Q, y^{\text{VQA}})$, which are the clip itself $V=\{(\tau^{\text{past}}_1, I_1,L_1,y_1),\dots, (\tau^{\text{past}}_n,I_n,L_n,y_n)\}$, the accompanying question $Q$, and the answer to the question $y^{\text{VQA}}$. See Appendix~\ref{sec:datastats} for details on dataset curation and statistics.

\subsection{Two-Stage Training Recipe}
\label{sec:training}

\begin{figure}[t]
    \centering
    \subfigure[Driving]{
    \includegraphics[height=4cm]{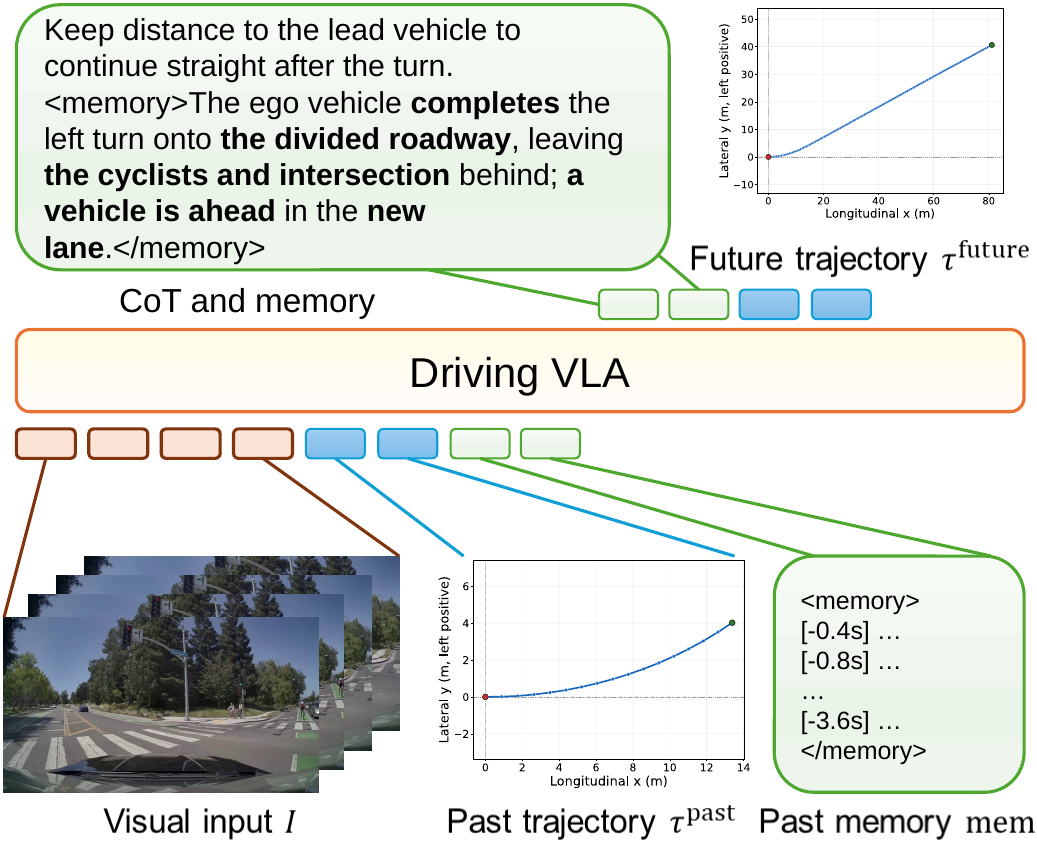}}
    \subfigure[Question-Answering]{
    \includegraphics[height=4cm]{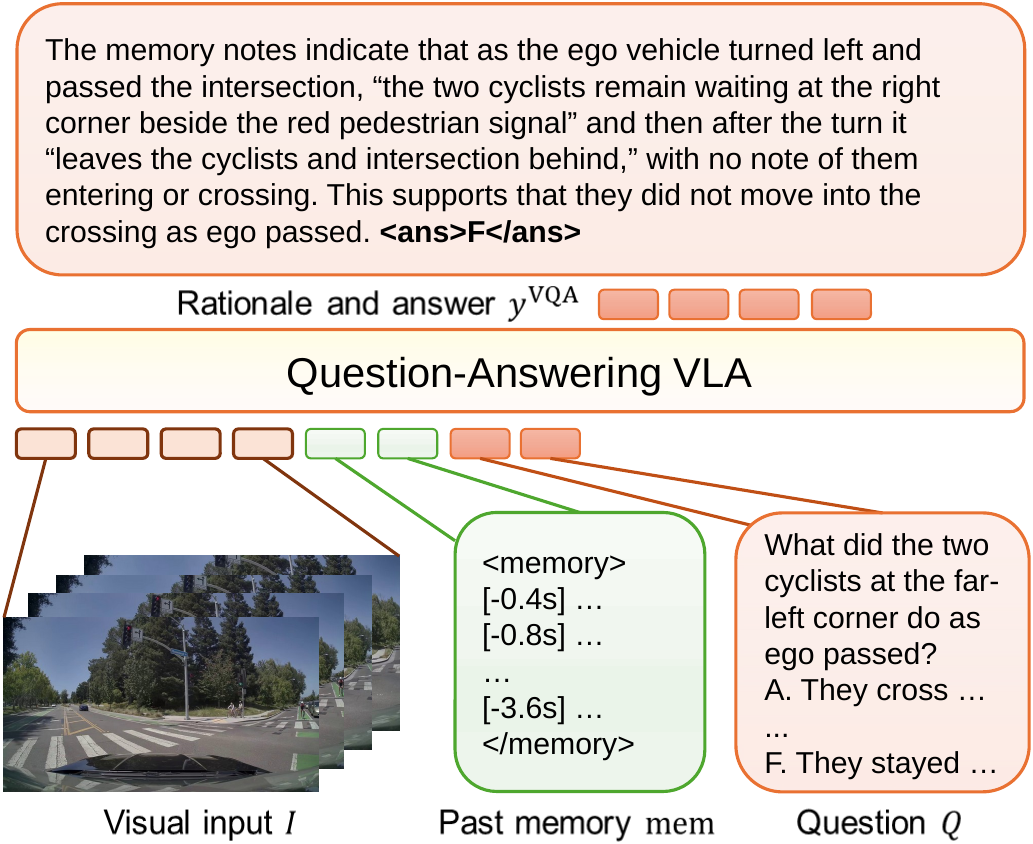}}
    \subfigure[Memory-Writing]{
    \includegraphics[height=4cm]{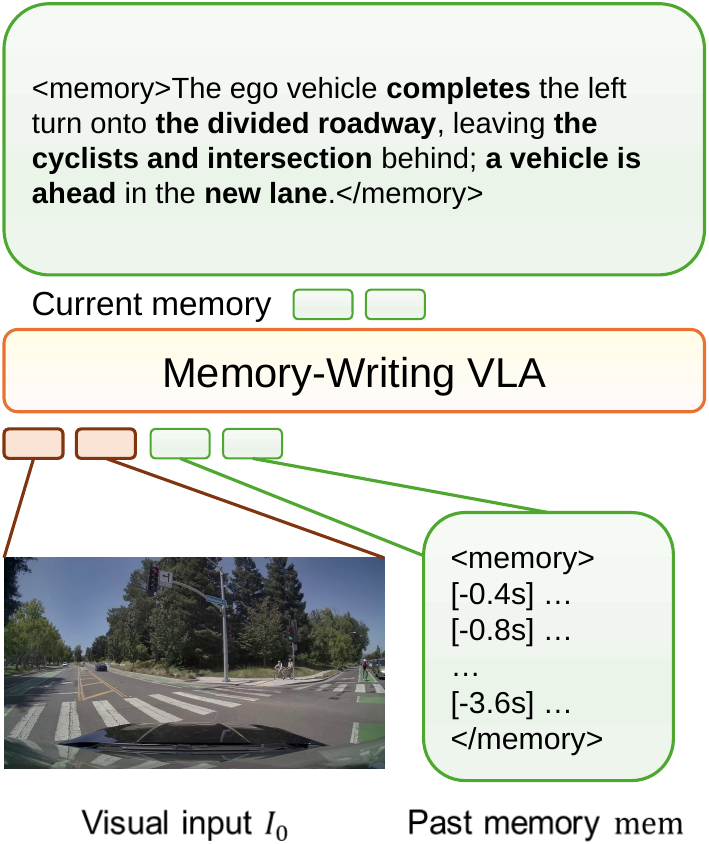}}
    \caption{An illustration of supervised training data of AD-Memo; the three tasks listed in panel (a), (b) and (c) are jointly trained in the SFT stage (all-way stop dataset only has (a) and (b)). Driving and memory-writing data are one entry per step, while question-answering data is one entry per clip.}
    \label{fig:sft-data}
\end{figure}

AD-Memo uses a two-stage training recipe: Supervised Fine-Tuning (SFT) on the SFT training set $\mathcal{D}_{\text{SFT}}$, and semi-closed loop Reinforcement Learning (RL) on the RL training set $\mathcal{D}_{\text{RL}}$.


\textbf{Supervised Fine-Tuning (SFT).} This is the main stage for equipping the VLA model with memory-writing and memory-dependent reasoning abilities, during which the model jointly learns three tasks: driving while writing memory, answering VQA questions based on its memory, and standalone memory writing, as described in Sec.~\ref{sec:lbm}. 
The training loss for this stage is
\begin{equation}
\label{eq:sftloss}
\begin{aligned}
&-\mathbb{E}_{(V,Q,y^{\mathrm{VQA}})\sim\mathcal{D}_{\mathrm{SFT}}}
\biggl[
\sum_{i=1}^{n}(\ell_{1,i}+c_1\ell_{2,i})
+c_2\ell_3
\biggr],
\quad \text{where}
\\
&\ell_{1,i}=\log\pi_\theta(y_i^{\mathrm{drive}}|x_i),\;
\ell_{2,i}=\log\pi_\theta(\mathrm{mem}_i|I_i^0,L_i),\;
\ell_3=\log\pi_\theta(y^{\mathrm{VQA}}|I_{n+1},L_{n+1},Q).
\end{aligned}
\end{equation}
In~\Eqref{eq:sftloss}, $I_i$ contains the past few frames at step $i$, $I_i^0$ is the current single frame ($I_i^0\in I_i$), $L_i=\{\text{mem}_{i-T},\dots,\text{mem}_{i-1}\}$ is the memory from prior steps, $y_i^{\text{drive}}=(\text{CoT}_i, \text{mem}_i,\tau^{\text{future}})$ is the driving task output, $\text{mem}_i$ is the current memory, $Q$ is the question, and $y^{\text{VQA}}$ is the response for the VQA task. $c_1, c_2\geq 0$ are constants. See examples of driving, question-answering and memory-writing data in Fig.~\ref{fig:sft-data}. For both datasets, we train 2 epochs as discussed in Appendix~\ref{sec:longer}.


\textbf{Semi-Closed-Loop RL.} While our SFT stage effectively equips the VLA agent with memory-related abilities, one key shortcoming remains: exposure bias. When driving and answering driving-related questions, the model relies on its own memory from prior steps, which it does not encounter during SFT and which may be out of distribution at inference time. Closed-loop RL using the agent's own memory is a natural way to address this mismatch. However, fully closed-loop RL requires a simulator to reproduce changes in environmental inputs and the reactions of other road users. Simulators like CARLA~\citep{dosovitskiy2017carla} exhibit a substantial sim-to-real appearance gap~\citep{pasios2025carla2real}, while 3D reconstruction-based simulators such as NuRec~\citep{nvidia2026instantnurec} can be costly and still inaccurate.

\begin{wrapfigure}{r}{0.35\linewidth}
    \centering
    \includegraphics[width=\linewidth]{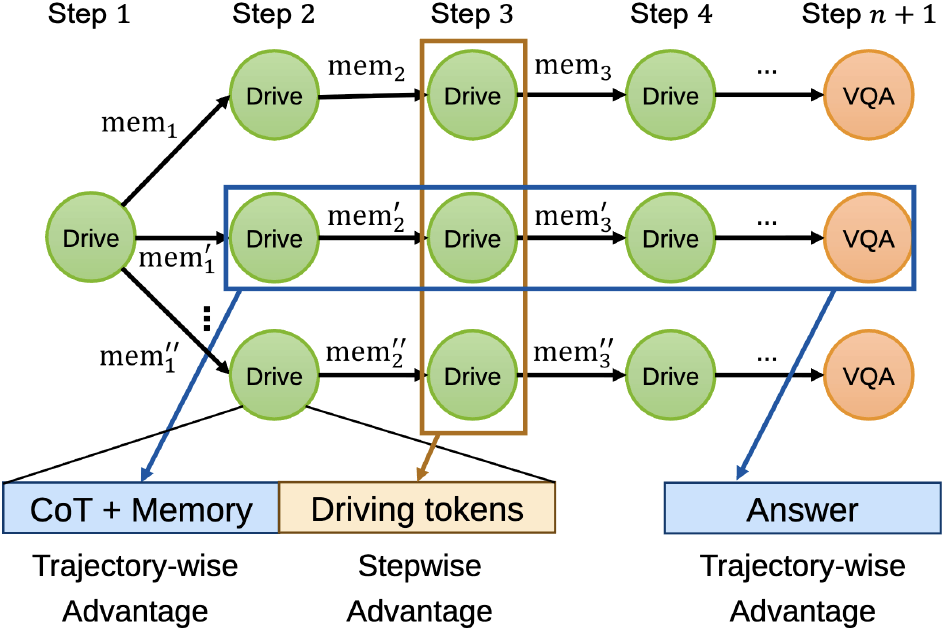}
    \caption{An illustration of semi-closed loop RL, where we keep $K$ parallel memories; the memory and driving tokens require heterogeneous handling, with the former being trajectory-wise and the latter being stepwise.}
    \label{fig:semiclosed}
\end{wrapfigure}


To address this challenge, we propose a novel training paradigm: \textit{semi-closed-loop RL}, which is open-loop in control and closed-loop in memory. More specifically, for a given clip with $|V|=n+1$ keyframes, we start from the first keyframe and sample $K$ independent rollouts, each using memory generated at previous steps of the same rollout; the other inputs, such as visual observations and past trajectories, are replayed from ground truth and are therefore identical at the same step across all $K$ rollouts. We then apply GRPO to these rollouts.


However, the key difficulty of semi-closed-loop RL lies in computing advantages for two heterogeneous components: the driving trajectory $\tau^{\text{future}}$ and memory/VQA answers. The former does not affect future decisions in our open-loop control process, so its advantage is computed across rollouts at the same step to provide denser feedback and reduce variance; memory, in contrast, affects subsequent driving and question answering, so its advantage must account for downstream rewards, while terminal VQA answers are evaluated using the final MCQ reward (see Fig.~\ref{fig:semiclosed} for an illustration). Considering these factors, and inspired by factored policy gradients~\citep{spooner2021factored}, we propose a novel RL algorithm named \textbf{D}eviation-\textbf{A}djusted \textbf{C}ausal \textbf{A}dvantage \textbf{P}olicy \textbf{O}ptimization (\textit{Da Capo}). More specifically, for rollout $i\in\{1,\ldots,K\}$ at step $j\in\{1,\dots,n\}$, let $r^{\text{Driving}}_{i,j}$ denote the driving reward and $r^{\text{VQA}}_i$ the terminal MCQ reward. We define the following returns for driving trajectories, memory, and VQA answers, respectively: $G^{\text{Driving}}_{i,j}=\frac{r^{\text{Driving}}_{i,j}}{n}$, $G^{\text{Memory}}_{i,j}=\frac{1}{n}\sum_{k=j}^n r^{\text{Driving}}_{i,k}+\lambda^{\text{VQA}}r^{\text{VQA}}_i$, and $G^{\text{VQA}}_{i}=\lambda^{\text{VQA}}r^{\text{VQA}}_i$, where $\lambda^{\text{VQA}}>0$ is a constant. The corresponding advantage $A^c_{i,j}$ for component $c\in\{\text{Driving}, \text{Memory}, \text{VQA}\}$ in rollout $i$ at step $j$ is:
\begin{equation}
\label{eq:adv}
A^c_{i,j}=(G^c_{i,j}-\frac{1}{K}\sum_{k=1}^KG^c_{k,j})/\text{std}(G^c_{*,j}),
\end{equation}
with the step index $j$ omitted for VQA (i.e., $\forall i,j$,  $G^{\text{VQA}}_{i,j}=G^{\text{VQA}}_i$). The name ``causal advantage'' reflects the fact that, without standard-deviation normalization, the resulting on-policy gradient estimator has the same expectation as one based on the trajectory-level reward, while improving credit assignment by removing conditionally independent terms from the applied advantage; for example, as the driving tokens in step 1 do not affect the memory in step 10 with our open-loop control, reward from driving at step 1 should not be considered on memory advantage at step 10. Formally:

\begin{claim}
\label{thm:cda-unbiased}

The advantage in~\Eqref{eq:adv}, without division by $\text{std}(G^c_{*,j})$, yields the same expected on-policy score-function gradient as using the group-centered trajectory-level reward $\frac{1}{n}\sum_{s=1}^{n}r^{\text{Driving}}_{i,s}+\lambda^{\text{VQA}}r^{\text{VQA}}_i$, also without standard-deviation normalization.
\end{claim}



See Appendix~\ref{sec:math-proof-cda} for the proof. In our implementation, we set $\lambda^{\text{VQA}}=0.5$, with $r_i^{\text{VQA}}$ equal to $1$ if the parsed answer is correct and $0$ otherwise. The driving reward is $r^{\text{Driving}}=\max(-\frac{0.2}{64}\sum_{i=1}^{64}\|p_i-p^{\text{GT}}_i\|_2, -1)$, where $p_i\in\mathbb{R}^2$, $i=\{1,2,\dots,64$\}, denotes a predicted 2D waypoint, with waypoints sampled every $0.1$ seconds over the $6.4$-second future trajectory, and $p_i^{\text{GT}}$ is the corresponding ground-truth waypoint. We train our model for $2$ epochs on the all-way stop dataset and $1$ epoch on the general driving dataset. See Appendix~\ref{sec:hyperparam} for detailed hyperparameter choices. In contrast to ``Dr.\ GRPO''~\citep{liu2025understanding}, standard-deviation normalization of the advantage is vital in our setting, as it naturally accounts for differences in reward scales, e.g., 0.1m vs.\ 10m ADE difference between rollouts; see Appendix~\ref{sec:math-proof-cda} for detailed rationale and Appendix~\ref{sec:aws_exp} for ablations.

%


\section{Experiments}
\label{sec:exp}

This section is organized as follows. Sec.~\ref{sec:aws_exp} studies whether AD-Memo can drive well on the task-specific all-way stop scenes where memory is particularly important, as the model must remember the arrival order at the crossing when multiple cars are stopping simultaneously; Sec.~\ref{sec:gen_exp} studies whether AD-Memo can handle general and diverse driving scenes. In both sections, we study whether AD-Memo better understands the driving scene by better answering related VQA questions, and whether our semi-closed loop RL algorithm, \textit{Da Capo} can further improve our model. Finally, in Sec.~\ref{sec:port_exp}, we study whether the memory produced by AD-Memo can be portable to other models.

\textbf{Baseline Settings.} For a fair comparison, unless otherwise specified, all models are based on the VLA backbone of a 2B, single camera-focused checkpoint of Alpamayo 2~\citep{nvidia2026alpamayo2super} (i.e., no action experts; see Appendix~\ref{sec:limit} for details). We mainly test four baselines: 1) the \textit{base model} mentioned above that has not been further finetuned; 2) \textit{no memory}, which goes through the same training recipe but without memory; 3) \textit{CoT-as-memory}, which goes through the same training recipe but uses the original past CoT text as memory; 4) To provide a standard for the performance, we also test NVIDIA's recently published \textit{flagship driving model, Alpamayo 2 Super 34B}.

\subsection{All-Way Stop}
\label{sec:aws_exp}

\textbf{Evaluation Settings.} For each of the 2479 clips in the test set, following prior work~\citep{tan2025latent}, we generate $K=6$ parallel trajectories with different memories. To expand our evaluation basis, we calculate and report the average performance in each keyframe before the ground truth stop, which is a total of 50533 keyframes. 
For those keyframes, we follow existing evaluation protocols~\citep{liang2026planning} and report the following metrics: (1) \textit{minADE, Average (Avg.) ADE, and Most Likely (ML.) ADE}~\citep{minade}, where ADE stands for Average Displacement Error. This metric measures the difference between the predicted trajectory and the ground truth human expert trajectory;  (2) \textit{Task-specific metric}, such as go Success Rate (go SR), stop Success Rate (stop SR), roll-through rate (Roll), position error ($\Delta_{\text{pos}}$) and stop duration error ($\Delta_{\text{dur}}$), which measures the performance of prediction stopping and moving within a particular threshold of the ground truth, the chance of not stopping at all when it should stop, and the spatial and temporal difference related to ground truth stop; (3) \textit{MCQ Accuracy (Acc.)}, which measures the scene understanding ability of the VLA as a streaming video agent. See Appendix~\ref{sec:metrics} for details on the metrics.

\begin{table}[ht]
    \scriptsize
    \centering
    \caption{The main result for all-way stop dataset, with the best results bolded. ``ref. mem'' stands for using ``reference memory'' as input; they are generated in the same way as SFT ground truth labels. $\downarrow$ means the lower the better, and vice versa. The result shows that AD-Memo works the best on both driving and question-answering tasks, and even far outperforms the flagship model, Alpamayo2 Super 34B (which has action expert and does not output logits, thus has no most likely ADE).}
    \begin{tabular}{c@{\hspace{1pt}}cccccccc@{\hspace{6pt}}c}
    \toprule
          & minADE\rlap{$\,\downarrow$}
        & Avg. ADE\rlap{$\,\downarrow$}
        & ML. ADE\rlap{$\,\downarrow$}
        & Stop SR\rlap{$\,\uparrow$}
        & Go SR\rlap{$\,\uparrow$}
        & $\Delta_{\text{pos}}$\rlap{$\,\downarrow$}
        & $\Delta_{\text{dur}}$\rlap{$\,\downarrow$}
        & Roll\rlap{$\,\downarrow$}
        & MCQ Acc.\rlap{$\,\uparrow$} \\
    \midrule
        Base Model & 1.386 & 2.351 & 2.393 & 82.06\% & 33.74\% & 1.725 & 1.871 & 11.79\% & 0\% \\
        Alpamayo 2 Super 34B & 0.993 & 2.293 & N/A & 78.75\% & 34.30\% & 2.182 & 1.488 & 15.94\% & 33.19\% \\
        \midrule
        No mem. (SFT only) & 1.102 & 2.185 & 2.318 & 87.04\% & 38.59\% & 1.358 & 1.564 & 6.62\% & 33.77\% \\
        CoT mem. (SFT only) & 1.096 & 2.171 & 2.296 & 86.89\% & 39.37\% & 1.377 & 1.531 & 6.69\% & 45.41\% \\
        AD-Memo (SFT only) & 1.049 & 2.121 & 2.256 & 87.48\% & 40.66\% & 1.257 & 1.469 & 6.05\% & 45.99\%\\
        \midrule
        CoT mem. (SFT+ref. mem) & 1.098 & 2.169 & 2.286 & 86.95\% & 39.45\% & 1.367 & 1.521 & 6.61\% & 45.24\% \\
        AD-Memo (SFT+ref. mem) & 0.963 & 1.963 & 2.086 & 87.70\% & 44.66\% & 1.185 & 1.296 & 5.84\% & \textbf{89.53\%} \\
        \midrule
        No mem. (SFT+\textit{Da Capo}) & \textbf{0.944} & 1.952 & 2.005 & 88.92\% & 43.76\% & 1.125 & 1.248 & \textbf{5.21\%} & 35.32\% \\
        CoT mem. (SFT+\textit{Da Capo}) & 0.968 & 1.906 & 1.942 & 89.20\% & 44.88\% & 1.157 & 1.194 & 5.31\% & 45.87\% \\
        AD-Memo (Ours) & 0.951 & \textbf{1.866} & \textbf{1.915} & \textbf{89.26\%} & \textbf{45.01\%} & \textbf{1.070} & \textbf{1.183} & 5.26\% & 51.66\% \\
        
    \bottomrule
    \end{tabular}
    
    \label{tab:main_aws}
\end{table}

\textbf{Results.} Tab.~\ref{tab:main_aws} illustrates the result, which clearly shows that AD-Memo, both after SFT and RL, outperforms all other baselines (even Alpamayo2 Super 34B) under the same training recipe, which shows the effectiveness of AD-Memo. Notably, with the reference memory, AD-Memo works much better than its SFT counterpart with self-generated memory, which shows the potential of improvement from accurate language memory, and also necessitates RL training. CoT-as-memory works similar or slightly better in driving than no memory and largely helps question answering, but shows much lower potential as the MCQ accuracy of reference and self-generated memory is almost the same for CoT-as-memory. This is because CoT focuses more on ego action and misses most objects on the road (Appendix~\ref{sec:datastats}). Also, all three types of memory after RL significantly improved, which proves the effectiveness of \textit{Da Capo}. 

\textbf{Qualitative example.} Fig.~\ref{fig:qual_img} shows an example of how memory aids driving. In this case, both ego and the left car stay still at the intersection. Without memory, the agent cannot decide which car to go first. In fact, the car on the left arrives prior to ego car, which is recorded in the memory as ``left car arrives at the crossing''. This information does not exist in the original CoT without memory. Thus, as Fig.~\ref{fig:qual_res} shows, our method achieves more precise control that better fits ground truth.

\begin{figure}[t]
    \centering
    \includegraphics[width=\linewidth]{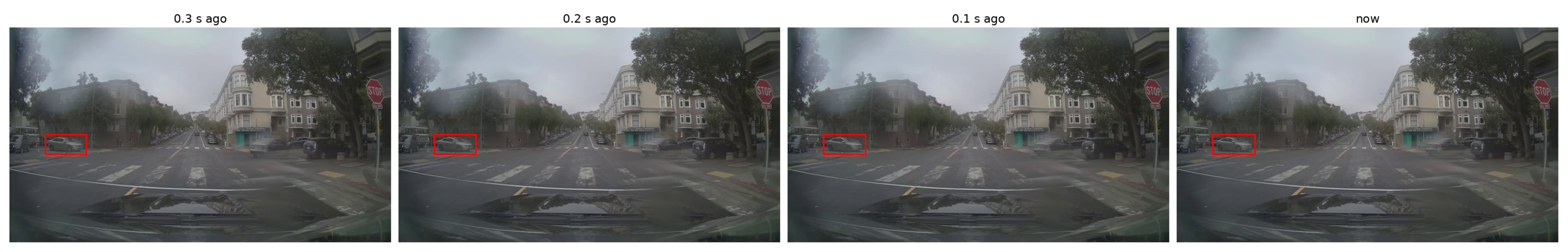}
    \caption{An example of how memory aids driving. The ego and the car on the left (circled in red) are both stopping. Without memory, the agent does not know whether it should move first.}
    \label{fig:qual_img}
\end{figure}

\begin{figure}[t]
    \centering
    \subfigure[No memory]{
    \includegraphics[width=0.3\linewidth]{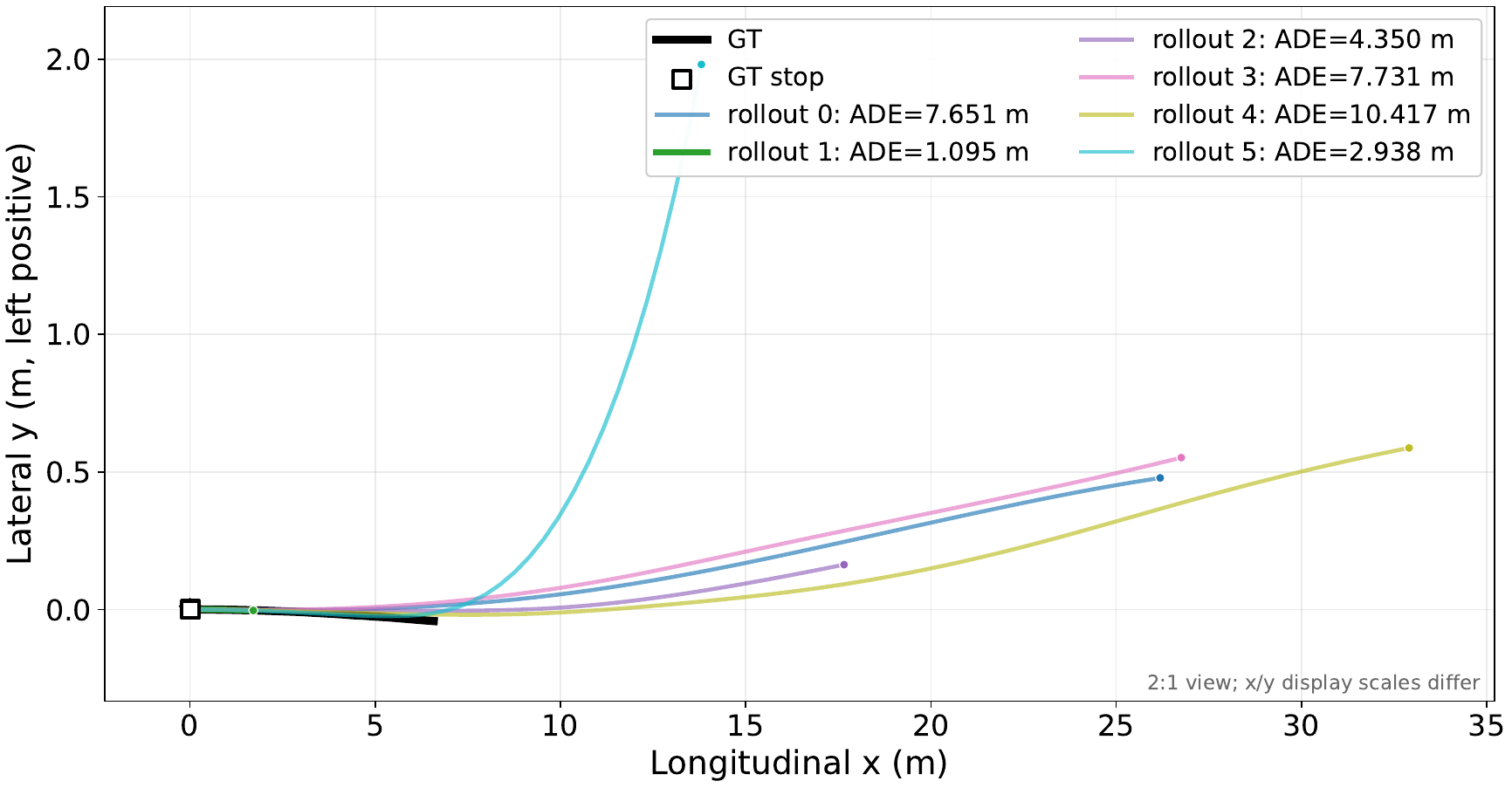}}
    \subfigure[CoT as memory]{
    \includegraphics[width=0.3\linewidth]{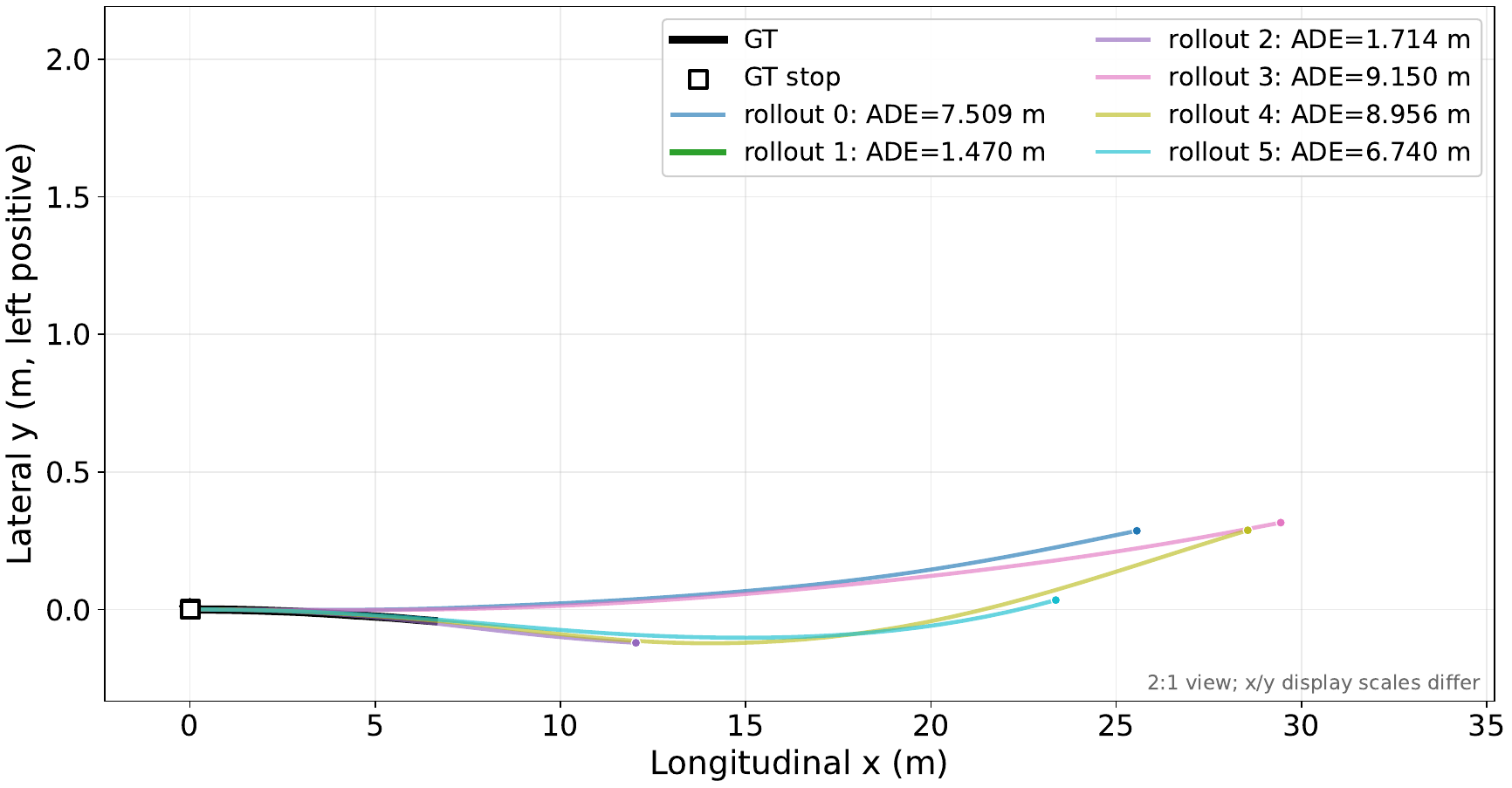}}
    \subfigure[Ours]{
    \includegraphics[width=0.3\linewidth]{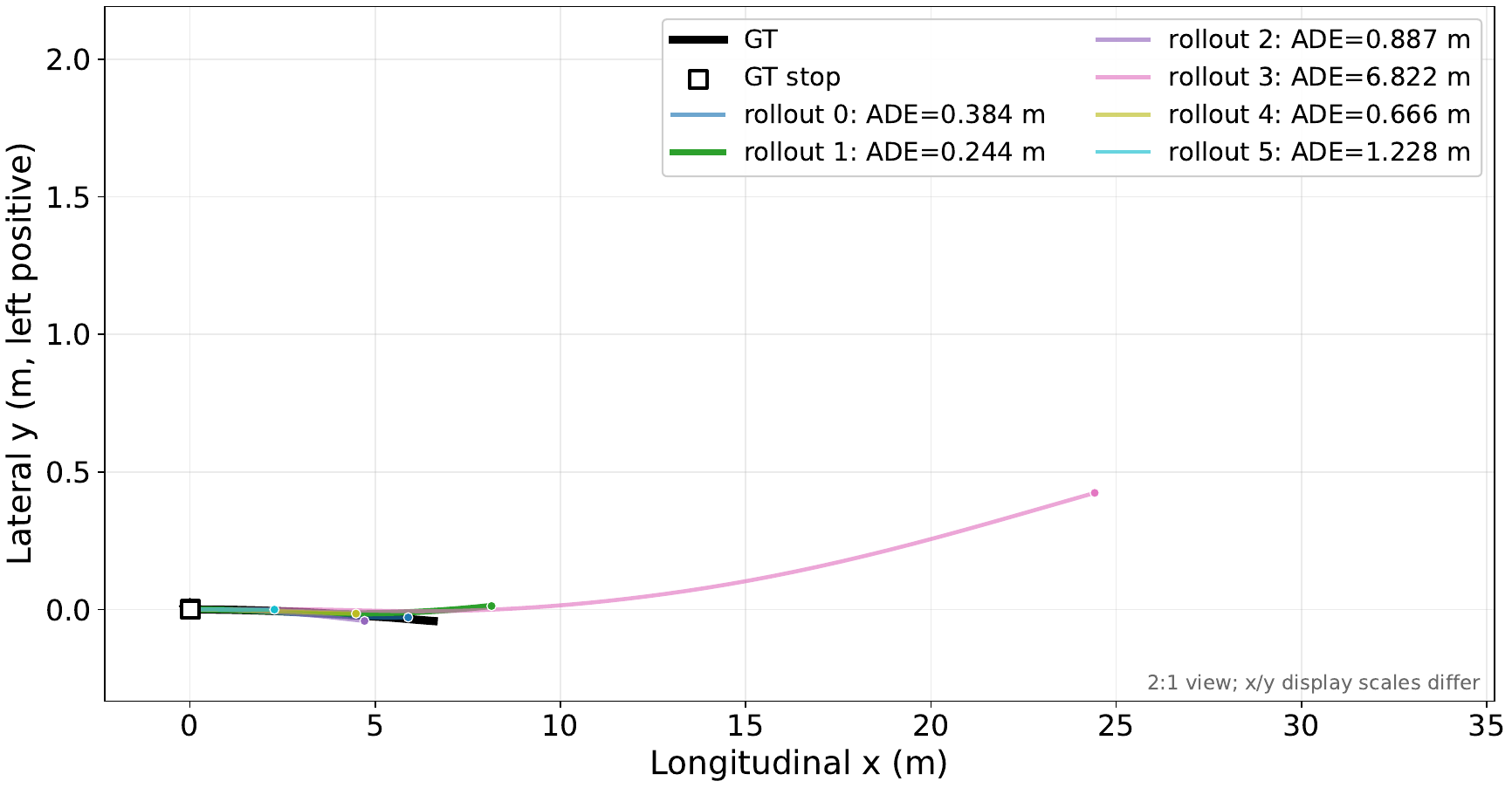}}
    \caption{The trajectory of different methods with the input from Fig.~\ref{fig:qual_img}. The result shows that our method can better recognize that ego should stop, which yields lower ADE results.}
    \label{fig:qual_res}
\end{figure}

\textbf{Ablation on the components of Da Capo.} To test the necessity of introducing standard deviation in our RL algorithm (see Appendix~\ref{sec:math-proof-cda} for detailed rationale), we test our RL algorithm based on the same AD-Memo (SFT-only) checkpoint against the following baselines: \textit{Capo} (no standard deviation), using a constant scaling factor on advantage $(=0.4)$, and trajectory-level GRPO. The results summarized in Tab.~\ref{tab:ablation_dacapo} show that adding standard deviation to our algorithm improves the empirical performance. The result also shows that all variants of \textit{Capo} far outperform GRPO with trajectory-level reward, which proves the benefit of removing causally unrelated reward terms. See RL training curves and ablations on $\lambda^{\text{VQA}}$ in Appendix~\ref{sec:curve} and~\ref{sec:ablation_lambda}.

\begin{table}[t]
    \centering
    \scriptsize
    \caption{The ablation on \textit{Da Capo}, which proves that a better credit assignment significantly improves final performance over trajectory-level GRPO. The standard deviation in advantage also helps by introducing a natural scaling factor that balances rewards of different magnitudes.}
    \begin{tabular}{c@{\hspace{4pt}}ccccccccc}
    \toprule
        & minADE\rlap{$\,\downarrow$} & Avg. ADE\rlap{$\,\downarrow$} & ML. ADE\rlap{$\,\downarrow$} & Stop SR\rlap{$\,\uparrow$} & Go SR\rlap{$\,\uparrow$} & $\Delta_\text{pos}$\rlap{$\,\downarrow$} & $\Delta_{\text{dur}}$\rlap{$\,\downarrow$} & Roll\rlap{$\,\downarrow$} & MCQ Acc.\rlap{$\,\uparrow$} \\
    \midrule
       \textit{Da Capo} (Ours) & 0.951 & \textbf{1.866} & \textbf{1.915} & \textbf{89.26\%} & \textbf{45.01\%}
       & \textbf{1.070} & \textbf{1.183} & 5.26\% & \textbf{51.66\%} \\
       \textit{Capo} (no std) & \textbf{0.942} & 1.952 & 2.047 & 89.22\% & 44.27\%
       & 1.130 & 1.277 & \textbf{5.01\%} & 50.91\% \\
       \textit{Capo} + const. scaling & 0.943 & 1.911 & 1.979 & 89.21\% & 44.74\%
       & 1.122 & 1.227 & 5.03\% & 51.57\%  \\
       GRPO (traj-level reward) & 1.007 & 2.113 & 2.194 & 87.11\% & 40.94\% & 1.312 & 1.392 & 6.28\% & 48.16\% \\
       \midrule
       AD-Memo (SFT only) & 1.049 & 2.121 & 2.256 & 87.48\% & 40.66\% & 1.257 & 1.469 & 6.05\% & 45.99\%\\
    \bottomrule
    \end{tabular}
    
    \label{tab:ablation_dacapo}
\end{table}

\subsection{General Driving Scenes}
\label{sec:gen_exp}

\begin{table}[t]
    \centering
    \scriptsize
     \caption{The main result for the general driving dataset, where the best results are bolded; the result shows that our method performs the best. ADE, FDE and corner distance are all the lower the better; the MCQ accuracy is the higher the better.}
    \begin{tabular}{c@{\hspace{4pt}}cccccccc}
         \toprule
          & minADE\rlap{$\,\downarrow$} & Avg. ADE\rlap{$\,\downarrow$} & ML. ADE\rlap{$\,\downarrow$} & minFDE\rlap{$\,\downarrow$} & Avg. FDE\rlap{$\,\downarrow$} &
          ML. FDE\rlap{$\,\downarrow$} & Corner\rlap{$\,\downarrow$} & MCQ Acc.\rlap{$\,\uparrow$} \\
         \midrule
         Base Model
          & 1.032 & 1.981 & 2.049 & 2.853 & 5.937 & 6.194 & 1.001 & 0 \\
         Alpamayo 2 Super 34B
          & 0.954 & 2.081 & N/A & \textbf{2.556} & 6.141  & N/A & 0.902 & 25.71\% \\
         \midrule
          No mem. (SFT only)
          & 1.007 & 2.033 & 2.075 & 2.716 & 6.108 & 6.279 & 0.968 & 58.02\% \\
         CoT mem. (SFT only) 
          & 1.031 & 2.059 & 2.105 & 2.787 & 6.192 & 6.375 & 0.993 & 57.51\% \\
         AD-Memo (SFT only)
          & \textbf{1.005} & 2.011 & 2.049 & 2.705 & 6.042 & 6.196 & \textbf{0.967} & \textbf{65.63\%} \\
         \midrule
         CoT mem. (SFT+ref. mem)
          & 1.034 & 2.064 & 2.105 & 2.795 & 6.207 & 6.373 & 0.995 & 57.36\% \\
         AD-Memo (SFT+ref. mem)
          & 1.012 & 2.014 & 2.051 & 2.734 & 6.053 & 6.208 & 0.974 & 67.65\% \\
         \midrule
          No mem. (SFT+\textit{Da Capo}) & 1.090 & 1.883 & 1.891 & 3.042 & 5.583 & 5.648 & 1.062 & 57.43\% \\
         
          CoT mem. (SFT+\textit{Da Capo})
            & 1.105 & 1.907 & 1.918 & 3.071 & 5.631 & 5.707 & 1.079 & 56.94\% \\
         AD-Memo (Ours)
          & 1.091 & \textbf{1.864} & \textbf{1.869} & 3.027 & \textbf{5.512} & \textbf{5.576} & 1.066 & 65.51\% \\
         \bottomrule
    \end{tabular}
   
    \label{tab:main_gen}
\end{table}

\textbf{Evaluation Settings.} For general driving scenes, we use the same ADE metrics and MCQ accuracy as those in Sec.~\ref{sec:aws_exp}. We additionally report: (1) \textit{minFDE, average FDE and most likely FDE} (Final Displacement Error), which measures the difference between the car's final position between the predicted and ground truth trajectory; (2) \textit{corner distance} (See Appendix~\ref{sec:metrics} for details), which describes the discrepancy of the car's bounding box between ground truth and prediction. In total, we report the average result of 9112 clips over 145462 keyframes.

\textbf{Results.} Tab.~\ref{tab:main_gen} illustrates the result on the general driving dataset. Our method outperforms all other baselines. Notably, different from the all-way stop dataset (Sec.~\ref{sec:aws_exp}), RL improves average and most likely discrepancy to ground truth but not the minimum of rollouts. This is because with RL, the policy distribution is more concentrated and thus leads to higher ADE with inaccurate intention prediction. For example, the model may concentrate on predictions of turning left at the crossing, while the ground truth is turning right. In contrast, at an all-way stop crossing, the intention of passing is clear. See Appendix~\ref{sec:qualitative} for an example of how our memory benefits general driving.

\subsection{Memory Portability}
\label{sec:port_exp}

\begin{wrapfigure}[18]{r}{0.34\linewidth}
    \centering
    \includegraphics[width=\linewidth]{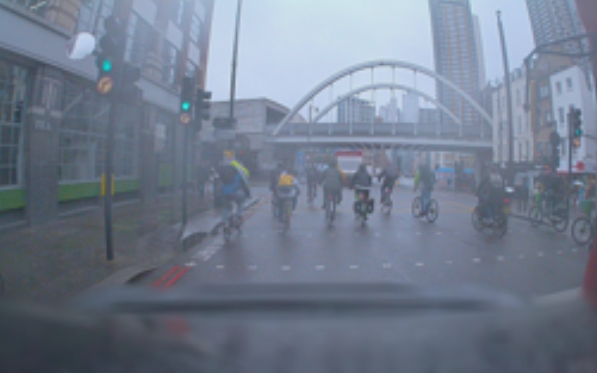}
    \caption{An image paired with the question ``What is the color of traffic lights?'' Repeated statements in the unmerged memory that ``the ego vehicle is now stopped'' mislead GPT-5.6 Luna into incorrectly inferring that the traffic light was red.}
    \label{fig:merge}
\end{wrapfigure}

One crucial advantage of AD-Memo is its \textit{portability}: it can improve the model's own performance and serve as a plug-and-play enhancement for other models. To evaluate this portability, we use our SFT-trained checkpoint to process each clip sequentially, generating memory from memories generated at previous frames and the current frame. We then ask GPT-5.6 Luna to answer questions from LingoQA~\citep{marcu2024lingoqa} and WaymoQA~\citep{yu2025waymoqa} (see Appendix~\ref{sec:qas}), using only the final frame and the generated memory as context. This setup allows the answering model to access information from earlier frames through memory generated by our model. The results in Tab.~\ref{tab:port} show that this memory improves another model's VQA performance in a zero-shot setting.

\begin{table}[t]
    \centering
    \small
    \caption{The accuracy of GPT-5.6 Luna on LingoQA and WaymoQA, which shows that our memory can enhance the scene understanding of other models in a plug-and-play manner.}
    \begin{tabular}{ccccc}
    \toprule
         & Our memory + last frame & Unmerged memory + last frame & Only last frame & Full video \\
        \midrule
       WaymoQA & 71.65\% & 68.3\% & 66.63\% & 75\% \\
       LingoQA & 67\% & 60.4\% & 62.6\% & 70\% \\
    \bottomrule
    \end{tabular}
    
    \label{tab:port}
\end{table}

\textbf{What matters for language memories?} When we curate ``ground truth / reference memory'' as SFT training labels, we ask the model to review the memories generated for each frame, and \textit{merge} them to keep track of the same object (See Appendix~\ref{sec:datastats-gen} for details). To verify the importance of this merging step, we train our checkpoint on the memory-writing task with merged memory and unmerged memory respectively, and test the portability of the memory generated by such checkpoint. The result is illustrated in Tab.~\ref{tab:port}, which clearly shows that unmerged memory performs significantly worse.

A closer inspection in the wording of the two versions of memory suggests that the performance gap can be attributed to three reasons. First, the lack of consistent references (e.g., the lead car, the pedestrian on the left) in standalone description of individual frames caused the memory to repeatedly identify the same item, increasing more than twice in length (16.57 words vs.\ 51.62 words on average) than the merged memory, making it not only undesirable for memory, but also distracting; second, standalone memory contains 97-98\% fewer temporal words such as ``continue'' and ``remain'', reducing the temporal information of the objects; finally, as shown in Fig.~\ref{fig:merge}, when the standalone memory repeatedly states ``ego vehicle is stopped'', the model can be misled and determine that the color of traffic lights were mostly red. Thus, \textit{temporal information in memory is a key to success.}

\section{Related Works}


\textbf{VLA driving agents.} Empowered by the recent success of LLMs~\citep{singh2025openai,openai2026tenadvances} and Vision-Language Models (VLMs)~\citep{li2025survey}, VLA models~\citep{brohan2023rt,o2024open} are an emerging solution to the long-standing challenge of fully autonomous driving~\citep{wang2025alpamayo,li2026unidrivevla,zhou2026qwen}. Compared to the conventional ``perception-decision-action'' modular pipeline~\citep{urmson2009autonomous,8443742}, Vision-Action (VA) models that map sensor input to action~\citep{hu2023planning, liao2025diffusiondrive}, and driving VLMs that only predict high-level meta-actions~\citep{huang2024drivlme,drivemlm}, VLA driving agents utilize the reasoning abilities of their VLM backbones while producing end-to-end trajectories for control. This idea has attracted considerable attention recently~\citep{tan2025latent,fu2025minddrive,peng2026counterfactual,li2026drivevla}. Following this idea, we make full use of VLM backbone knowledge with language memory and build AD-Memo on Alpamayo~\citep{wang2025alpamayo}.


\textbf{Driving agents with memory.} Memory-dependent tasks pose a long-standing and important challenge in autonomous driving~\citep{liang2026planning}. Prior solutions mainly fall into two groups: \textit{latent vector}~\citep{leng2025occupancy,song2025don} and \textit{in-context learning}~\citep{wen2024dilu,mei2024continuously,zhang2026pram}. The former maintains a bank of dense representation vectors extracted from past inputs. While some methods keep KV-memory bank~\citep{shao2023reasonnet} or perform test-time training~\citep{kim2026think}, the prevailing approach is to store compact tokens produced by Q-Former~\citep{li2023blip, fu2025orion, zhang2025adadrive,huang2026mindvla}, ConvLSTM~\citep{shi2015convolutional, 10736002}, or other encoders~\citep{salazar2024tlcfuse}; such memory is usually accessed via cross-attention. The latter often stores structured records of prior driving experiences (e.g., scene descriptions and responses), which are added to the agent's input as prompts in a Retrieval-Augmented Generation (RAG)~\citep{lewis2020retrieval} manner~\citep{cui2024board,yao2025lilodriver}. While other memory methods also exist, such as directly storing past frames for reasoning agents~\citep{zheng2026driveagent,baimbetova2026pedestrian} and maintaining knowledge graphs~\citep{li2026vln}, AD-Memo is the first VLA driving agent with in-episode language memory.




\textbf{VLA agents with in-episode language memory.} Several works in the robotics community explore in-episode language memory in VLA agents. Notes-to-Self~\citep{haresh2026self} records grounding information, plans, and the agent's actions, and updates memory when the model determines that a subtask has been completed; MEM~\citep{torne2026mem}
uses language memory to record the agent's plans and actions for a high-level VLM, which in turn instructs a low-level VLA; Goal2Skill~\citep{liu2026goal2skill} and $\tau_0$-VLA~\citep{cai2026tau_0} are similar but incorporate additional mechanisms for verbal self-reflection and verification. AD-Memo differs from these works in two ways: 1) None of these works studies autonomous driving, where language memory remains surprisingly underexplored. This domain presents new challenges regarding memory content (Sec.~\ref{sec:lbm}); 2) None of the above works uses RL for training, whereas our work proposes a novel semi-closed-loop RL algorithm, representing a methodological advance.


\section{Conclusion}
\label{sec:conclusion}

In this paper, we propose AD-Memo, the first VLA driving agent with in-episode language-based memory that is not only effective but also brief, explainable, and portable. The agent outputs memory alongside CoT when driving, and the memory is then appended to the future input. To train this model, we curate multiple datasets, and propose \textit{Da Capo}, a novel semi-closed-loop RL algorithm that replays ground-truth driving states while using the agent's self-generated memory. We empirically demonstrate that AD-Memo improves driving performance, performs better on VQA tasks for scene understanding, and produces memory portable to other models. We believe that AD-Memo offers a timely and effective solution for memory-dependent tasks in autonomous driving.

\subsection*{Reproducibility statement}

We describe the detail of how we curate our dataset in Appendix~\ref{sec:datastats}. The experiment metric, external benchmarks and hyperparameters are described in Appendix~\ref{sec:metrics}, Appendix~\ref{sec:qas}, and Appendix~\ref{sec:hyperparam} respectively. We also report our computational resource consumption in Appendix~\ref{sec:resource}. 





\bibliography{iclr2027_conference}

@article{kusano2025comparison,
  title={Comparison of Waymo Rider-Only crash rates by crash type to human benchmarks at 56.7 million miles},
  author={Kusano, Kristofer D and Scanlon, John M and Chen, Yin-Hsiu and McMurry, Timothy L and Gode, Tilia and Victor, Trent},
  journal={Traffic Injury Prevention},
  year={2025}
}

@InProceedings{brohan2023rt,
  title = 	 {RT-2: Vision-Language-Action Models Transfer Web Knowledge to Robotic Control},
  author =       {Zitkovich, Brianna and Yu, Tianhe and Xu, Sichun and Xu, Peng and Xiao, Ted and Xia, Fei and Wu, Jialin and Wohlhart, Paul and Welker, Stefan and Wahid, Ayzaan and Vuong, Quan and Vanhoucke, Vincent and Tran, Huong and Soricut, Radu and Singh, Anikait and Singh, Jaspiar and Sermanet, Pierre and Sanketi, Pannag R. and Salazar, Grecia and Ryoo, Michael S. and Reymann, Krista and Rao, Kanishka and Pertsch, Karl and Mordatch, Igor and Michalewski, Henryk and Lu, Yao and Levine, Sergey and Lee, Lisa and Lee, Tsang-Wei Edward and Leal, Isabel and Kuang, Yuheng and Kalashnikov, Dmitry and Julian, Ryan and Joshi, Nikhil J. and Irpan, Alex and Ichter, Brian and Hsu, Jasmine and Herzog, Alexander and Hausman, Karol and Gopalakrishnan, Keerthana and Fu, Chuyuan and Florence, Pete and Finn, Chelsea and Dubey, Kumar Avinava and Driess, Danny and Ding, Tianli and Choromanski, Krzysztof Marcin and Chen, Xi and Chebotar, Yevgen and Carbajal, Justice and Brown, Noah and Brohan, Anthony and Arenas, Montserrat Gonzalez and Han, Kehang},
  booktitle = {CoRL},
  year = 	 {2023},
}

@inproceedings{o2024open,
  title={Open x-embodiment: Robotic learning datasets and rt-x models: Open x-embodiment collaboration},
  author={O’Neill, Abby and Rehman, Abdul and Maddukuri, Abhiram and Gupta, Abhishek and Padalkar, Abhishek and Lee, Abraham and Pooley, Acorn and Gupta, Agrim and Mandlekar, Ajay and Jain, Ajinkya and others},
  booktitle={ICRA},
  year={2024}
}

@article{bai2025qwen3,
  title={Qwen3-vl technical report},
  author={Bai, Shuai and Cai, Yuxuan and Chen, Ruizhe and Chen, Keqin and Chen, Xionghui and Cheng, Zesen and Deng, Lianghao and Ding, Wei and Gao, Chang and Ge, Chunjiang and others},
  journal={arXiv preprint arXiv:2511.21631},
  year={2025}
}

@misc{waymo2026scale,
  author       = {Tekedra Mawakana and Dmitri Dolgov},
  title        = {Accelerating Our Global Growth: Waymo Raises \$16 Billion Investment Round},
  year         = {2026},
  url          = {https://waymo.com/blog/2026/02/waymo-raises-usd16-billion-investment-round/},
}

@article{wang2025alpamayo,
  title={Alpamayo-r1: Bridging reasoning and action prediction for generalizable autonomous driving in the long tail},
  author={Wang, Yan and Luo, Wenjie and Bai, Junjie and Cao, Yulong and Che, Tong and Chen, Ke and Chen, Yuxiao and Diamond, Jenna and Ding, Yifan and Ding, Wenhao and others},
  journal={arXiv preprint arXiv:2511.00088},
  year={2025}
}

@article{zhou2026qwen,
  title={Qwen-Drive-1.0: An Initial Step towards a Vision-Language Foundation Model for Autonomous Driving},
  author={Zhou, Xin and Zhao, Zongchuang and Yang, Zhibo and Li, Mingsheng and Zhong, Humen and Bai, Shuai and Chu, Du and Chen, Ruizhe and Li, Zhaohai and Tang, Jun and others},
  journal={arXiv preprint arXiv:2609.00111},
  year={2026}
}

@article{li2026unidrivevla,
  title={Unidrivevla: Unifying understanding, perception, and action planning for autonomous driving},
  author={Li, Yongkang and Zhou, Lijun and Yan, Sixu and Liao, Bencheng and Yan, Tianyi and Xiong, Kaixin and Chen, Long and Xie, Hongwei and Wang, Bing and Chen, Guang and others},
  journal={arXiv preprint arXiv:2604.02190},
  year={2026}
}

@inproceedings{dao2022flashattention,
  title={Flashattention: Fast and memory-efficient exact attention with io-awareness},
  author={Dao, Tri and Fu, Dan and Ermon, Stefano and Rudra, Atri and R{\'e}, Christopher},
  booktitle={NeurIPS},
  year={2022}
}

@article{liang2026planning,
  title={Planning-aligned Token Compression for Long-Context Autonomous Driving},
  author={Liang, Zhixuan and Chen, Yuxiao and You, Yurong and Karkus, Peter and Ding, Wenhao and Li, Boyi and Popov, Alexander and Wang, Yan and Igl, Maximilian and Li, Yiming and others},
  journal={IEEE Robotics and Automation Letters},
  year={2026}
}

@article{kumar2025occluded,
  title={Occluded nuScenes: A Multi-Sensor Dataset for Evaluating Perception Robustness in Automated Driving},
  author={Kumar, Sanjay and Brophy, Tim and Mohandas, Reenu and Grua, Eoin Martino and Sistu, Ganesh and Donzella, Valentina and Eising, Ciaran},
  journal={arXiv preprint arXiv:2510.18552},
  year={2025}
}

@inproceedings{shi2026memoryvla,
  title={Memoryvla: Perceptual-cognitive memory in vision-language-action models for robotic manipulation},
  author={Shi, Hao and Xie, Bin and Liu, Yingfei and Sun, Lin and Liu, Fengrong and Wang, Tiancai and Zhou, Erjin and Fan, Haoqiang and Zhang, Xiangyu and Huang, Gao},
  booktitle={ICLR},
  year={2026}
}

@inproceedings{zeng2026streamforest,
  title={Streamforest: Efficient online video understanding with persistent event memory},
  author={Zeng, Xiangyu and Qiu, Kefan and Zhang, Qingyu and Li, Xinhao and Wang, Jing and Li, Jiaxin and Yan, Ziang and Tian, Kun and Tian, Meng and Zhao, Xinhai and others},
  booktitle={NeurIPS},
  year={2025}
}

@inproceedings{zeng2025vision,
  title={Are vision llms road-ready? a comprehensive benchmark for safety-critical driving video understanding},
  author={Zeng, Tong and Wu, Longfeng and Shi, Liang and Zhou, Dawei and Guo, Feng},
  booktitle={SIGKDD},
  year={2025}
}

@inproceedings{wei2022chain,
  title={Chain-of-thought prompting elicits reasoning in large language models},
  author={Wei, Jason and Wang, Xuezhi and Schuurmans, Dale and Bosma, Maarten and Ichter, Brian and Xia, Fei and Chi, Ed and Le, Quoc V and Zhou, Denny},
  booktitle={NeurIPS},
  year={2022}
}

@inproceedings{Bengio2015ScheduledSampling,
  author    = {Samy Bengio and Oriol Vinyals and Navdeep Jaitly and Noam Shazeer},
  title     = {Scheduled Sampling for Sequence Prediction with Recurrent Neural Networks},
  booktitle = {NIPS},
  year      = {2015}
}

@article{liu2022understanding,
  title={Understanding time variations of dnn inference in autonomous driving},
  author={Liu, Liangkai and Wang, Yanzhi and Shi, Weisong},
  journal={arXiv preprint arXiv:2209.05487},
  year={2022}
}

@inproceedings{li2023blip,
  title={Blip-2: Bootstrapping language-image pre-training with frozen image encoders and large language models},
  author={Li, Junnan and Li, Dongxu and Savarese, Silvio and Hoi, Steven},
  booktitle={ICML},
  year={2023}
}

@inproceedings{shi2015convolutional,
  title={Convolutional LSTM network: A machine learning approach for precipitation nowcasting},
  author={Shi, Xingjian and Chen, Zhourong and Wang, Hao and Yeung, Dit-Yan and Wong, Wai-Kin and Woo, Wang-chun},
  booktitle={NIPS},
  year={2015}
}

@article{huang2026mindvla,
  title={Mindvla-u1: Vla beats va with unified streaming architecture for autonomous driving},
  author={Huang, Yuzhou and Zhu, Benjin and Lu, Hengtong and Huang, Victor Shea-Jay and Zhang, Haiming and Chen, Wei and Dai, Jifeng and Xie, Yan and Li, Hongsheng},
  journal={arXiv preprint arXiv:2605.12624},
  year={2026}
}

@inproceedings{shao2023reasonnet,
  title={Reasonnet: End-to-end driving with temporal and global reasoning},
  author={Shao, Hao and Wang, Letian and Chen, Ruobing and Waslander, Steven L and Li, Hongsheng and Liu, Yu},
  booktitle={CVPR},
  year={2023}
}

@inproceedings{salazar2024tlcfuse,
  title={TLCFuse: Temporal Multi-Modality Fusion Towards Occlusion-Aware Semantic Segmentation},
  author={Salazar-Gomez, Gustavo and Liu, Wenqian and Diaz-Zapata, Manuel and Sierra-Gonzalez, David and Laugier, Christian},
  booktitle={IEEE Intelligent Vehicles Symposium},
  year={2024}
}

@inproceedings{vaswani2017attention,
  title={Attention is all you need},
  author={Vaswani, Ashish and Shazeer, Noam and Parmar, Niki and Uszkoreit, Jakob and Jones, Llion and Gomez, Aidan N and Kaiser, {\L}ukasz and Polosukhin, Illia},
  booktitle={NIPS},
  year={2017}
}

@article{10736002,
  author={Gerasyov, Matvey and Savchenko, Andrey V. and Makarov, Ilya},
  journal={IEEE Access}, 
  title={Enhancing Autonomous Driving With Spatial Memory and Attention in Reinforcement Learning}, 
  year={2024}}

@inproceedings{leng2025occupancy,
  title={Occupancy learning with spatiotemporal memory},
  author={Leng, Ziyang and Yang, Jiawei and Yi, Wenlong and Zhou, Bolei},
  booktitle={ICCV},
  year={2025}
}

@inproceedings{song2025don,
  title={Don’t shake the wheel: Momentum-aware planning in end-to-end autonomous driving},
  author={Song, Ziying and Jia, Caiyan and Liu, Lin and Pan, Hongyu and Zhang, Yongchang and Wang, Junming and Zhang, Xingyu and Xu, Shaoqing and Yang, Lei and Luo, Yadan},
  booktitle={CVPR},
  year={2025}
}

@article{kim2026think,
  title={Think as Needed: Geometry-Driven Adaptive Perception for Autonomous Driving},
  author={Kim, Donghyun and Park, Jaehyoung},
  journal={arXiv preprint arXiv:2605.10117},
  year={2026}
}

@inproceedings{fu2025orion,
  title={Orion: A holistic end-to-end autonomous driving framework by vision-language instructed action generation},
  author={Fu, Haoyu and Zhang, Diankun and Zhao, Zongchuang and Cui, Jianfeng and Liang, Dingkang and Zhang, Chong and Zhang, Dingyuan and Xie, Hongwei and Wang, Bing and Bai, Xiang},
  booktitle={ICCV},
  year={2025}
}

@inproceedings{zhang2025adadrive,
  title={Adadrive: Self-adaptive slow-fast system for language-grounded autonomous driving},
  author={Zhang, Ruifei and Xie, Junlin and Zhang, Wei and Chen, Weikai and Tan, Xiao and Wan, Xiang and Li, Guanbin},
  booktitle={ICCV},
  year={2025}
}

@inproceedings{zheng2026driveagent,
  title={Driveagent-r1: Advancing vlm-based autonomous driving with active perception and hybrid thinking},
  author={Zheng, Weicheng and Mao, Xiaofei and Ye, Nanfei and Li, Pengxiang and Zhan, Kun and Lang, Xianpeng and Zhao, Hang},
  booktitle={ICLR},
  year={2026}
}

@article{li2026vln,
  title={VLN-AVP: Zero-Shot Vision-Language Navigation with Hybrid Long-Short-Term Memory for Autonomous Valet Parking},
  author={Li, Yijian and Mu, Xiangru and Li, Changze and Shi, Hantian and Cai, Jiyuan and Cai, Jia and Liu, Xiaoxue and Sun, Yajing and Yang, Ming and Qin, Tong},
  journal={arXiv preprint arXiv:2607.17767},
  year={2026}
}

@inproceedings{wen2024dilu,
  title={Dilu: A knowledge-driven approach to autonomous driving with large language models},
  author={Wen, Licheng and Fu, Daocheng and Li, Xin and Cai, Xinyu and Ma, Tao and Cai, Pinlong and Dou, Min and Shi, Botian and He, Liang and Qiao, Yu},
  booktitle={ICLR},
  year={2024}
}

@article{zhang2026pram,
  title={PRAM-R: A Perception-Reasoning-Action-Memory Framework with LLM-Guided Modality Routing for Adaptive Autonomous Driving},
  author={Zhang, Yi and Zhang, Xian and Zhao, Saisi and Song, Yinglei and Wu, Chengdong and Petrovic, Nenad and Knoll, Alois},
  journal={arXiv preprint arXiv:2603.04222},
  year={2026}
}

@inproceedings{mei2024continuously,
  title={Continuously learning, adapting, and improving: A dual-process approach to autonomous driving},
  author={Mei, Jianbiao and Ma, Yukai and Yang, Xuemeng and Wen, Licheng and Cai, Xinyu and Li, Xin and Fu, Daocheng and Zhang, Bo and Cai, Pinlong and Dou, Min and others},
  booktitle={NeurIPS},
  year={2024}
}

@inproceedings{baimbetova2026pedestrian,
  title={Pedestrian-Aware LLM-Driven Behavioral Planning for Autonomous Vehicles},
  author={Baimbetova, Aidana and Yonekura, Haruki and Rizk, Hamada and Yamaguchi, Hirozumi},
  booktitle={ITSC},
  year={2026}
}

@article{yao2025lilodriver,
  title={Lilodriver: A lifelong learning framework for closed-loop motion planning in long-tail autonomous driving scenarios},
  author={Yao, Huaiyuan and Li, Pengfei and Jin, Bu and Zheng, Yupeng and Liu, An and Mu, Lisen and Su, Qing and Zhang, Qian and Chen, Yilun and Li, Peng},
  journal={arXiv preprint arXiv:2505.17209},
  year={2025}
}

@inproceedings{cui2024board,
  author={Cui, Can and Yang, Zichong and Zhou, Yupeng and Peng, Juntong and Park, Sung-Yeon and Zhang, Cong and Ma, Yunsheng and Cao, Xu and Ye, Wenqian and Feng, Yiheng and Panchal, Jitesh and Li, Lingxi and Chen, Yaobin and Wang, Ziran},
  booktitle={IROS}, 
  title={On-Board Vision-Language Models (VLMs) for Personalized Motion Control of Autonomous Vehicles}, 
  year={2025},
}

@inproceedings{lewis2020retrieval,
  title={Retrieval-augmented generation for knowledge-intensive nlp tasks},
  author={Lewis, Patrick and Perez, Ethan and Piktus, Aleksandra and Petroni, Fabio and Karpukhin, Vladimir and Goyal, Naman and K{\"u}ttler, Heinrich and Lewis, Mike and Yih, Wen-tau and Rockt{\"a}schel, Tim and others},
  booktitle={NeurIPS},
  year={2020}
}

@inproceedings{yu2025dapo,
  title={Dapo: An open-source llm reinforcement learning system at scale},
  author={Yu, Qiying and Zhang, Zheng and Zhu, Ruofei and Yuan, Yufeng and Zuo, Xiaochen and Yue, Yu and Dai, Weinan and Fan, Tiantian and Liu, Gaohong and Liu, Juncai and Liu, Lingjun and others},
  booktitle={NeurIPS},
  year={2025}
}

@article{shao2024deepseekmath,
  title={Deepseekmath: Pushing the limits of mathematical reasoning in open language models},
  author={Shao, Zhihong and Wang, Peiyi and Zhu, Qihao and Xu, Runxin and Song, Junxiao and Bi, Xiao and Zhang, Haowei and Zhang, Mingchuan and Li, Y.K. and Wu, Y. and others},
  journal={arXiv preprint arXiv:2402.03300},
  year={2024}
}

@inproceedings{li2025survey,
  title={A survey of state of the art large vision language models: Alignment, benchmark, evaluations and challenges},
  author={Li, Zongxia and Wu, Xiyang and Du, Hongyang and Liu, Fuxiao and Nghiem, Huy and Shi, Guangyao},
  booktitle={CVPRW},
  year={2025}
}

@misc{openai2026tenadvances,
  author       = {{OpenAI}},
  title        = {Ten Advances in Mathematics and Theoretical Computer Science},
  year         = {2026},
  howpublished = {\url{https://cdn.openai.com/pdf/ten-proofs-oai.pdf}}
}

@article{singh2025openai,
  title={Openai gpt-5 system card},
  author={Singh, Aaditya and Fry, Adam and Perelman, Adam and Tart, Adam and Ganesh, Adi and El-Kishky, Ahmed and McLaughlin, Aidan and Low, Aiden and Ostrow, AJ and Ananthram, Akhila and others},
  journal={arXiv preprint arXiv:2601.03267},
  year={2025}
}

@inproceedings{8443742,
  author={Kato, Shinpei and Tokunaga, Shota and Maruyama, Yuya and Maeda, Seiya and Hirabayashi, Manato and Kitsukawa, Yuki and Monrroy, Abraham and Ando, Tomohito and Fujii, Yusuke and Azumi, Takuya},
  booktitle={ICCPS}, 
  title={Autoware on Board: Enabling Autonomous Vehicles with Embedded Systems}, 
  year={2018},
}

@article{urmson2009autonomous,
  title     = {Autonomous Driving in Traffic: Boss and the Urban Challenge},
  author    = {Urmson, Chris and Baker, Chris and Dolan, John and Rybski, Paul and Salesky, Bryan and Whittaker, William and Ferguson, Dave and Darms, Michael},
  journal   = {AI Magazine},
  year      = {2009}
}

@article{drivemlm,
title={DriveMLM: aligning multi-modal large language models with behavioral planning states for autonomous driving},
author={Cui, Erfei and Wang, Wenhai and Li, Zhiqi and Xie, Jiangwei and Zou, Haoming and Deng, Hanming and Luo, Gen and Lu, Lewei and Zhu, Xizhou and Dai, Jifeng},
journal={Visual Intelligence},
year={2025}
}

@inproceedings{huang2024drivlme,
  title={Drivlme: Enhancing llm-based autonomous driving agents with embodied and social experiences},
  author={Huang, Yidong and Sansom, Jacob and Ma, Ziqiao and Gervits, Felix and Chai, Joyce},
  booktitle={IROS},
  year={2024}
}

@inproceedings{hu2023planning,
  title={Planning-oriented autonomous driving},
  author={Hu, Yihan and Yang, Jiazhi and Chen, Li and Li, Keyu and Sima, Chonghao and Zhu, Xizhou and Chai, Siqi and Du, Senyao and Lin, Tianwei and Wang, Wenhai and others},
  booktitle={CVPR},
  year={2023}
}

@inproceedings{liao2025diffusiondrive,
  title={Diffusiondrive: Truncated diffusion model for end-to-end autonomous driving},
  author={Liao, Bencheng and Chen, Shaoyu and Yin, Haoran and Jiang, Bo and Wang, Cheng and Yan, Sixu and Zhang, Xinbang and Li, Xiangyu and Zhang, Ying and Zhang, Qian and others},
  booktitle={CVPR},
  year={2025}
}

@inproceedings{peng2026counterfactual,
  title={Counterfactual vla: Self-reflective vision-language-action model with adaptive reasoning},
  author={Peng, Zhenghao and Ding, Wenhao and You, Yurong and Chen, Yuxiao and Luo, Wenjie and Tian, Thomas and Cao, Yulong and Sharma, Apoorva and Xu, Danfei and Ivanovic, Boris and others},
  booktitle={CVPR},
  year={2026}
}

@inproceedings{tan2025latent,
  title={Latent chain-of-thought world modeling for end-to-end autonomous driving},
  author={Tan, Shuhan and Chitta, Kashyap and Chen, Yuxiao and Tian, Ran and You, Yurong and Wang, Yan and Luo, Wenjie and Cao, Yulong and Krähenbühl, Philipp and Pavone, Marco and others},
  booktitle={CVPR},
  year={2026}
}

@inproceedings{li2026drivevla,
  title={Drivevla-w0: World models amplify data scaling law in autonomous driving},
  author={Li, Yingyan and Shang, Shuyao and Liu, Weisong and Zhan, Bing and Wang, Haochen and Wang, Yuqi and Chen, Yuntao and Wang, Xiaoman and An, Yasong and Tang, Chufeng and others},
  booktitle={ICLR},
  year={2026}
}

@inproceedings{fu2025minddrive,
  title={Minddrive: A vision-language-action model for autonomous driving via online reinforcement learning},
  author={Fu, Haoyu and Zhang, Diankun and Zhao, Zongchuang and Cui, Jianfeng and Xie, Hongwei and Wang, Bing and Chen, Guang and Ye, Hangjun and Liang, Dingkang and Bai, Xiang},
  booktitle={ECCV},
  year={2026}
}

@inproceedings{haresh2026self,
  title={Notes-to-Self: Scratchpad Augmented VLAs for Memory Dependent Manipulation Tasks},
  author={Haresh, Sanjay and Dijkman, Daniel and Bhattacharyya, Apratim and Memisevic, Roland},
  booktitle={ICRA},
  year={2026}
}

@article{torne2026mem,
  title={Mem: Multi-scale embodied memory for vision language action models},
  author={Torne, Marcel and Pertsch, Karl and Walke, Homer and Vedder, Kyle and Nair, Suraj and Ichter, Brian and Ren, Allen Z and Wang, Haohuan and Tang, Jiaming and Stachowicz, Kyle and others},
  journal={arXiv preprint arXiv:2603.03596},
  year={2026}
}

@article{liu2026goal2skill,
  title={Goal2skill: Long-horizon manipulation with adaptive planning and reflection},
  author={Liu, Zhen and Ning, Xinyu and Hu, Zhe and Xie, Xinxin and Li, Weize and Tang, Zhipeng and Wang, Chongyu and Yang, Zejun and Wang, Hanlin and Liu, Yitong and others},
  journal={arXiv preprint arXiv:2604.13942},
  year={2026}
}

@article{cai2026tau_0,
  title={$\tau_0$-VLA: a Hierarchical Robot Foundation Model with World-Model-Guided Test-Time Computation},
  author={Cai, Xiaowei and Cai, Yunuo and Chen, Bingao and Chen, Jingxiao and Chen, Zhi and Feng, Siyuan and Hou, Tengyu and Huang, Jingshun and Jiang, Han and Ju, Runkun and others},
  journal={arXiv preprint arXiv:2608.16885},
  year={2026}
}

@inproceedings{todorov2012mujoco,
  title={MuJoCo: A physics engine for model-based control},
  author={Todorov, Emanuel and Erez, Tom and Tassa, Yuval},
  booktitle={IROS},
  year={2012},
}

@inproceedings{yan2024reinforcement,
  title={Reinforcement learning gradients as vitamin for online finetuning decision transformers},
  author={Yan, Kai and Schwing, Alexander G and Wang, Yu-Xiong},
  booktitle={NeurIPS},
  year={2024}
}

@article{shi2026clear,
  title={CLEAR: Closed-Loop Reinforcement Learning at Scale for End-to-End Autonomous Driving},
  author={Shi, Yunxiao and Cai, Hong and Ghavamzadeh, Mohammad and Porikli, Fatih},
  journal={arXiv preprint arXiv:2607.02841},
  year={2026}
}

@inproceedings{dosovitskiy2017carla,
  title={CARLA: An open urban driving simulator},
  author={Dosovitskiy, Alexey and Ros, German and Codevilla, Felipe and Lopez, Antonio and Koltun, Vladlen},
  booktitle={CoRL},
  year={2017}
}

@article{nvidia2026instantnurec,
  title={Instant NuRec: Feed-Forward 3D Gaussian Reconstruction for Driving Scene Simulation},
  author={Huang, Jiahui and Ren, Jiawei and Tyszkiewicz, Michal and Haefner, Bjoern and Shelley, Michael and Kang, Xin and Kim, Seung Wook and Xu, Ning and Wu, Qi and Esturo, Janick Martinez and others},
  journal={arXiv preprint arXiv:2607.14203},
  year={2026}
}

@inproceedings{choi2026scale,
  title={SCALE: Self-uncertainty Conditioned Adaptive Looking and Execution for Vision-Language-Action Models},
  author={Choi, Hyeonbeom and Ahn, Daechul and Lee, Youhan and Kang, Taewook and Cho, Seongwon and Choi, Jonghyun},
  booktitle={ICML},
  year={2026}
}

@inproceedings{wang2024gensim,
  title={Gensim: Generating robotic simulation tasks via large language models},
  author={Wang, Lirui and Ling, Yiyang and Yuan, Zhecheng and Shridhar, Mohit and Bao, Chen and Qin, Yuzhe and Wang, Bailin and Xu, Huazhe and Wang, Xiaolong},
  booktitle={ICLR},
  year={2024}
}

@inproceedings{xiong2026recogdrive,
  title={Recogdrive: A reinforced cognitive framework for end-to-end autonomous driving},
  author={Li, Yongkang and Xiong, Kaixin and Guo, Xiangyu and Li, Fang and Yan, Sixu and Xu, Gangwei and Zhou, Lijun and Chen, Long and Sun, Haiyang and Wang, Bing and Ma, Kun and others},
  booktitle={ICLR},
  year={2026}
}

@inproceedings{li2026simplevla,
  title={Simplevla-rl: Scaling vla training via reinforcement learning},
  author={Li, Haozhan and Zuo, Yuxin and Yu, Jiale and Zhang, Yuhao and Yang, Zhaohui and Zhang, Kaiyan and Zhu, Xuekai and Zhang, Yuchen and Chen, Tianxing and Cui, Ganqu and others},
  booktitle={ICLR},
  year={2026}
}

@inproceedings{zhou2026autovla,
  title={Autovla: A vision-language-action model for end-to-end autonomous driving with adaptive reasoning and reinforcement fine-tuning},
  author={Zhou, Zewei and Cai, Tianhui and Zhao, Seth and Zhang, Yun and Huang, Zhiyu and Zhou, Bolei and Ma, Jiaqi},
  booktitle={NeurIPS},
  year={2025}
}

@inproceedings{sheng2026explorevla,
  title={Explorevla: Dense world modeling and exploration for end-to-end autonomous driving},
  author={Sheng, Zihao and Ye, Xin and Luo, Jingru and Chen, Sikai and Ren, Liu},
  booktitle={ECCV},
  year={2026}
}

@inproceedings{marcu2024lingoqa,
  title={Lingoqa: Visual question answering for autonomous driving},
  author={Marcu, Ana-Maria and Chen, Long and H{\"u}nermann, Jan and Karnsund, Alice and Hanotte, Benoit and Chidananda, Prajwal and Nair, Saurabh and Badrinarayanan, Vijay and Kendall, Alex and Shotton, Jamie and others},
  booktitle={ECCV},
  year={2024}
}

@inproceedings{spooner2021factored,
  title={Factored policy gradients: Leveraging structure for efficient learning in MOMDPs},
  author={Spooner, Thomas and Vadori, Nelson and Ganesh, Sumitra},
  booktitle={NeurIPS},
  year={2021}
}

@inproceedings{schulman2015gradient,
  title={Gradient estimation using stochastic computation graphs},
  author={Schulman, John and Heess, Nicolas and Weber, Theophane and Abbeel, Pieter},
  booktitle={NIPS},
  year={2015}
}

@inproceedings{xu2026wod,
  title={Wod-e2e: Waymo open dataset for end-to-end driving in challenging long-tail scenarios},
  author={Xu, Runsheng and Lin, Hubert and Jeon, Wonseok and Feng, Hao and Zou, Yuliang and Sun, Liting and Gorman, John and Tolstaya, Kate and Tang, Sarah and White, Brandyn and others},
  booktitle={CVPR},
  year={2026}
}

@inproceedings{he2016deep,
  title={Deep residual learning for image recognition},
  author={He, Kaiming and Zhang, Xiangyu and Ren, Shaoqing and Sun, Jian},
  booktitle={CVPR},
  year={2016}
}

@inproceedings{minade,
  title={You'll never walk alone: Modeling social behavior for multi-target tracking},
  author={Pellegrini, Stefano and Ess, Andreas and Schindler, Konrad and Van Gool, Luc},
  booktitle={ICCV},
  year={2009},
}

@article{pasios2025carla2real,
  title={Carla2real: A tool for reducing the sim2real appearance gap in carla simulator},
  author={Pasios, Stefanos and Nikolaidis, Nikos},
  journal={IEEE Transactions on Intelligent Transportation Systems},
  year={2025}
}

@article{zhou2025tumtraffic,
  title={Tumtraffic-videoqa: A benchmark for unified spatio-temporal video understanding in traffic scenes},
  author={Zhou, Xingcheng and Larintzakis, Konstantinos and Guo, Hao and Zimmer, Walter and Liu, Mingyu and Cao, Hu and Zhang, Jiajie and Lakshminarasimhan, Venkatnarayanan and Strand, Leah and Knoll, Alois C},
  journal={arXiv preprint arXiv:2502.02449},
  year={2025}
}

@inproceedings{liu2025understanding,
  title={Understanding r1-zero-like training: A critical perspective},
  author={Liu, Zichen and Chen, Changyu and Li, Wenjun and Qi, Penghui and Pang, Tianyu and Du, Chao and Lee, Wee Sun and Lin, Min},
  booktitle={COLM},
  year={2025}
}

@inproceedings{qian2024nuscenes,
  title={Nuscenes-qa: A multi-modal visual question answering benchmark for autonomous driving scenario},
  author={Qian, Tianwen and Chen, Jingjing and Zhuo, Linhai and Jiao, Yang and Jiang, Yu-Gang},
  booktitle={AAAI},
  year={2024}
}

@inproceedings{xie2025drivebench,
        author    = {Xie, Shaoyuan and Kong, Lingdong and Dong, Yuhao and Sima, Chonghao and Zhang, Wenwei and Chen, Qi Alfred and Liu, Ziwei and Pan, Liang},
        title     = {Are VLMs Ready for Autonomous Driving? An Empirical Study from the Reliability, Data and Metric Perspectives},
        booktitle = {ICCV},
        year      = {2025},
    }

@inproceedings{huang2026thinkact,
  title={Thinkact: Vision-language-action reasoning via reinforced visual latent planning},
  author={Huang, Chi-Pin and Wu, Yueh-Hua and Chen, Min-Hung and Wang, Frank and Yang, Fu-En},
  booktitle={NeurIPS},
  year={2025}
}

@article{huang2026mobilevla,
  title={MobileVLA-R1 2.0: RL-Enhanced Reasoning for Mobile Robot Control},
  author={Huang, Ting and Huang, Yue and Zhang, Zeyu and Yan, Shuicheng and Tang, Hao},
  journal={arXiv preprint arXiv:2609.06251},
  year={2026}
}

@misc{nvidia2026physicalaiav,
  title        = {PhysicalAI-Autonomous-Vehicles Dataset},
  author       = {{NVIDIA Corporation}},
  year         = {2026},
  howpublished = {\url{https://huggingface.co/datasets/nvidia/PhysicalAI-Autonomous-Vehicles}},
  note         = {NVIDIA Physical AI Autonomous Vehicles Dataset, version 26.03}
}

@misc{nvidia2026alpamayo2super,
  author       = {{NVIDIA}},
  title        = {{Alpamayo 2 Super}},
  year         = {2026},
  howpublished = {Hugging Face},
  url = {https://huggingface.co/nvidia/Alpamayo2-Super},
  note         = {Model card}
}

@article{yu2025waymoqa,
  title={Waymoqa: A multi-view visual question answering dataset for safety-critical reasoning in autonomous driving},
  author={Yu, Seungjun and Lee, Seonho and Kim, Namho and Shin, Jaeyo and Park, Junsung and Ryu, Wonjeong and Jung, Raehyuk and Shim, Hyunjung},
  journal={arXiv preprint arXiv:2511.20022},
  year={2025}
}
\bibliographystyle{iclr2027_conference}

\appendix
\newpage
\section*{Appendix: Vision-Language-Action Autonomous Driving Agent with Language-based Memory}
\renewcommand{\thetheorem}{\thesection.\arabic{theorem}}
\setcounter{claim}{0}
The appendix is organized as follows: in Sec.~\ref{sec:math-proof-cda}, we provide discussion and proofs for our novel semi-closed loop algorithm, \textit{Da Capo}; in Sec.~\ref{sec:details_appendix}, we discuss the details on our evaluation metrics (Sec.~\ref{sec:metrics}), benchmark adopted (Sec.~\ref{sec:qas}), and hyperparameters (Sec.~\ref{sec:hyperparam}); in Sec.~\ref{sec:datastats}, we introduce our dataset curation pipeline, statistics and examples of memory and VQA questions in our dataset; in Sec.~\ref{sec:add_exp}, we provide a broad range of qualitative and quantitative ablations; in Sec.~\ref{sec:resource}, we describe the computational resource for this work; finally, in Sec.~\ref{sec:limit}, we discuss the limitation and potential future work of our AD-Memo.

\section{Mathematical Proofs}
\label{sec:math-proof-cda}

To see the equivalence between the expected gradient of \textit{Da Capo} without standard deviation and the gradient of GRPO with trajectory-level reward, we start from an intuitive example. Suppose we have a $3$-step trajectory with memory $\text{mem}_1,\text{mem}_2, \text{mem}_3$, predicted trajectories $\tau_1, \tau_2,
\tau_3$, driving rewards $r^{\text{Driving}}_1,
r^{\text{Driving}}_2, r^{\text{Driving}}_3$, and a subsequent VQA
reward $r^{\text{VQA}}$ with $\lambda^{\text{VQA}}=1$. With a trajectory-level return, without
baseline subtraction or advantage normalization, the contribution
of $\text{mem}_2$ to the policy gradient is
\begin{equation}
\mathbb{E}_{\pi_\theta}\left[
\left(
\frac{1}{3}
\left(r^{\text{Driving}}_1+r^{\text{Driving}}_2+r^{\text{Driving}}_3\right)
+r^{\text{VQA}}
\right)
\nabla_\theta\log\pi_\theta(\text{mem}_2\mid c_2)
\right],
\end{equation}
where $c_2$ consists of the second-step input, including
$\text{mem}_1$, and the second-step CoT output.


Let $H$ denote the common rollout prefix containing the first-step
input, CoT output, and $\text{mem}_1$, before generating $\tau_1$
and the second-step continuation.
The latter uses $\text{mem}_1$ and the fixed second-step environmental
input, but not $\tau_1$. With independent sampling randomness for
the two branches, $\tau_1$ and $(c_2,\text{mem}_2)$ are therefore
\textit{conditionally independent} given $H$. Since
$r_1^{\text{Driving}}$ is determined by $\tau_1$ and the fixed
reward references,
$r_1^{\text{Driving}}\perp(c_2,\text{mem}_2)\mid H$.

By the score-function identity and the law of total expectation,
\begin{equation}
\begin{aligned}
&\mathbb{E}_{\pi_\theta}\left[
\nabla_\theta\log\pi_\theta(\text{mem}_2\mid c_2)
\mid H
\right]\\
&\qquad =
\mathbb{E}_{c_2\mid H}\left[
\sum_m
\pi_\theta(m\mid c_2)
\nabla_\theta\log\pi_\theta(m\mid c_2)
\right]\\
&\qquad =
\mathbb{E}_{c_2\mid H}\left[
\nabla_\theta\sum_m\pi_\theta(m\mid c_2)
\right]
=0,
\end{aligned}
\end{equation}
where $c_2$ is held fixed in the inner derivative.
Consequently, conditional independence gives
\begin{equation}
\begin{aligned}
&\mathbb{E}_{\pi_\theta}\left[
\frac{1}{3}r^{\text{Driving}}_1
\nabla_\theta\log\pi_\theta(\text{mem}_2\mid c_2)
\right]\\
=&
\frac{1}{3}\mathbb{E}_{\pi_\theta}\left[
\mathbb{E}_{\pi_\theta}
\left[r^{\text{Driving}}_1\mid H\right]
\mathbb{E}_{\pi_\theta}\left[
\nabla_\theta\log\pi_\theta(\text{mem}_2\mid c_2)
\mid H
\right]
\right]=0.
\end{aligned}
\end{equation}
Thus, removing $\frac{1}{3}r^{\text{Driving}}_1$ from the return
weighting the score-function term for $\text{mem}_2$ preserves
the expected policy gradient, consistent with
\citet{schulman2015gradient}.

More generally, fix the replayed inputs and terminal question $X$, and
sample $K$ independent on-policy rollouts with $n+1$ keyframes ($n$ driving frames + 1 VQA frame).
All expectations below are conditional on $X$; the conclusion also
holds after averaging over $X$. Assume differentiable policy
probabilities with fixed support and sufficient integrability to
interchange differentiation, summation, and expectation.

For rollout $i$ at step $j$, let
$S^{\text{Memory}}_{i,j}$, $S^{\text{Driving}}_{i,j}$, and
$S^{\text{VQA}}_i$ denote the scores of the CoT/memory block,
predicted trajectory, and terminal answer, respectively.
Each score is the gradient of the corresponding conditional
log-probability, summed over the block's tokens.

Define the trajectory-level return
\begin{equation}
R_i =
\frac{1}{n}\sum_{k=1}^{n}r^{\text{Driving}}_{i,k}
+\lambda^{\text{VQA}}r^{\text{VQA}}_i,
\end{equation}
and recall the causal returns from Sec.~\ref{sec:training}:
\begin{equation}
\begin{aligned}
G^{\text{Driving}}_{i,j}
&= \frac{1}{n}r^{\text{Driving}}_{i,j},\\
G^{\text{Memory}}_{i,j}
&= \frac{1}{n}\sum_{k=j}^{n}r^{\text{Driving}}_{i,k}
+\lambda^{\text{VQA}}r^{\text{VQA}}_i,\\
G^{\text{VQA}}_i
&= \lambda^{\text{VQA}}r^{\text{VQA}}_i.
\end{aligned}
\end{equation}
Without standard-deviation normalization, the corresponding
group-centered advantages are
\begin{equation}
A^c_{i,j}
= G^c_{i,j}-\frac{1}{K}\sum_{\ell=1}^{K}G^c_{\ell,j},
\qquad
A^{\text{traj}}_i
= R_i-\frac{1}{K}\sum_{\ell=1}^{K}R_\ell,
\end{equation}
where $c\in\{\text{Driving},\text{Memory},\text{VQA}\}$,
with the step index $j$ omitted for VQA.

\begin{claim}[Expected-gradient equivalence]
For every output block,
\begin{equation}
\mathbb{E}\left[S^c_{i,j}A^c_{i,j}\right]
=
\mathbb{E}\left[S^c_{i,j}A^{\text{traj}}_i\right].
\end{equation}
Consequently, without standard deviation, our method and trajectory-level group centering yield
the same expected on-policy score-function gradient (but different credit assignment).
\end{claim}

\begin{proof}
For any generated block $Y$ with conditioning input and prefix $H$,
its score $S=\nabla_\theta\log\pi_\theta(Y\mid H)$ satisfies
\begin{equation}
\mathbb{E}[S\mid H]
=
\sum_y \pi_\theta(y\mid H)
\nabla_\theta\log\pi_\theta(y\mid H)
=
\nabla_\theta\sum_y\pi_\theta(y\mid H)
=0.
\end{equation}

Let $\Delta^c_{i,j}=R_i-G^c_{i,j}$ denote the discarded reward.
For driving, it contains the other steps' driving rewards and
the VQA reward; these are conditionally independent of the
current predicted trajectory because predicted trajectories are
not reused as later driving or VQA inputs.
For Memory and VQA, it contains only driving rewards determined
before the corresponding block is generated.
Thus, in each case, $\Delta^c_{i,j}$ is conditionally independent
of the generated block given its conditioning context $H$, and
\begin{equation}
\mathbb{E}\left[S^c_{i,j}\Delta^c_{i,j}\right]
=
\mathbb{E}\left[
\mathbb{E}[\Delta^c_{i,j}\mid H]\,
\mathbb{E}[S^c_{i,j}\mid H]
\right]
=0.
\end{equation}

For $\ell\neq i$, independence between the rollouts and
$\mathbb{E}[S^c_{i,j}]=0$ also give
$\mathbb{E}[S^c_{i,j}\Delta^c_{\ell,j}]=0$.
Therefore,
\begin{equation}
\mathbb{E}\left[
S^c_{i,j}\left(A^{\text{traj}}_i-A^c_{i,j}\right)
\right]=\mathbb{E}\left[
S^c_{i,j}\left(
\Delta^c_{i,j}
-\frac{1}{K}\sum_{\ell=1}^{K}\Delta^c_{\ell,j}
\right)
\right]
=0.
\end{equation}
Summing over output blocks and averaging over rollouts
completes the proof.
\end{proof}

\begin{remark}
Adding standard deviation to the advantage (i.e. the Deviation-Adjustment (DA) in \textit{Da Capo}) generally breaks the expected-gradient equivalence in two ways. First, as standard deviation is computed individually for every step, the introduction of such denominator adds a stepwise scaling to the total reward. Second, estimating
the denominator from the same rollout group can invalidate the zero-mean cancellation of conditionally independent reward terms. To see this, decompose the trajectory-level return associated with a driving output as $r_1+r_2$, where $r_2$ is conditionally independent of the current trajectory tokens given their conditioning context. Although $r_2$ can be removed without changing the expected
unnormalized gradient, its contribution after division by
$\operatorname{std}(r_1+r_2)$ does not necessarily vanish, because the denominator
also depends on the sampled output through $r_1$.

Despite this loss of exact equivalence, we find standard-deviation normalization empirically beneficial in our autonomous-driving setting, unlike the findings of ``Dr.~GRPO''~\citep{liu2025understanding} for mathematical reasoning. This choice is motivated by the
heterogeneous reward scales across scenes and task components. For example, driving-reward differences within a rollout group can range from approximately $0.02$ for a $0.1$\,m ADE difference to $1$ for an ADE difference exceeding $5$\,m. Normalization within each step and output component puts these heterogeneous advantage signals on a comparable scale, allowing small but meaningful driving improvements to contribute alongside larger reward differences. Our ablations show that this adaptive normalization outperforms RL with no standard deviation or a fixed normalization constant; see Sec.~\ref{sec:aws_exp} for ablations.
\end{remark}

\section{Experiment Details}
\label{sec:details_appendix}
\subsection{Evaluation Metrics}
\label{sec:metrics}
We list the details on the evaluation metrics used in Sec.~\ref{sec:exp} here:

\begin{itemize}

\item \textbf{minADE (lower is better).} Minimum Average Displacement Error (ADE) is one of the most commonly used metric in autonomous driving. More specifically, given $K$ predicted trajectory $\tau_1=(p^1_1,p^1_2,\dots,p^1_{M})$, $\dots$, $\tau_K=(p^K_1,p^K_2,\dots,p^K_{M})$ and ground truth trajectory $\tau=(p^{\text{GT}}_1,p^{\text{GT}}_2,\dots,p^{\text{GT}}_M)$, minADE is calculated as

\begin{equation}
\min_{i\in\{1,2,\dots,K\}} \frac{1}{M}\sum_{j=1}^M\|p_j^i-p^{\text{GT}}_j\|_2,
\end{equation}

where $p^i_j, p^{\text{GT}}_j\in \mathbb{R}^2$ are waypoints on the 2D coordinate, each represents the location for every 0.1 second. Following prior works, unless otherwise specified, we have $K=6$ and $M=64$ (i.e. prediction of 6.4s) in our evaluation. minADE describes the behavior of the best prediction, which compensates for potential mode error (e.g. due to human intention unavailable to the agent, the car turns right when the prediction is turning left).

\item \textbf{Average ADE (lower is better).} Similar to minADE, average ADE can be calculated as 

\begin{equation}
\frac{1}{MK}\sum_{i=1}^K\sum_{j=1}^M\|p_j^i-p^{\text{GT}}_j\|_2,    
\end{equation}

which describes the average driving performance of the VLA.

\item \textbf{Most likely ADE (lower is better).} Most likely ADE describes the performance of the agent if the car is under its control. With $K$ rollouts, we calculate the log probability of each rollout, and report the ADE $\frac{1}{M}\sum_{j=1}^M\|p_j^i-p^{\text{GT}}_j\|_2$ for trajectory $i$ with the highest log probability.

\item \textbf{minFDE (lower is better).} Minimum Final Displacement Error (FDE) describes the discrepancy between the final position of the predicted and ground truth trajectory. More specifically, it is calculated as

\begin{equation}
\min_{i\in\{1,2,\dots,K\}} \|p_{M}^i-p^{\text{GT}}_{M}\|_2,
\end{equation}

i.e., only take the last waypoint into account.

\item \textbf{Average FDE (lower is better).} Similar to minFDE and average ADE, average FDE is calculated as 

\begin{equation}
\sum_{j=1}^{K}\frac{1}{K}\|p_{M}^j-p^{\text{GT}}_{M}\|_2,
\end{equation}

\item \textbf{Most likely FDE (lower is better).} Similar to most likely ADE, most likely FDE is the FDE for the trajectory with the highest log probability.

\item \textbf{Corner Distance (lower is better).} Corner distance jointly measures the translation and orientation errors of the predicted ego-vehicle trajectory. Let $c^{i}_{j,l}$ and $c^{\mathrm{GT}}_{j,l}$ denote the $l$-th corner of the predicted and ground-truth ego-vehicle bounding boxes, respectively, at timestep $j$, where $l\in\{1,\dots,8\}$ for a 3D bounding box. Corner distance is calculated as
\begin{equation}
\frac{1}{8M}
\sum_{j=1}^{M}
\sum_{\ell=1}^{8}
\min_{i\in\{1,\ldots,K\}}
\left\|
c_{j,\ell}^{i}-c_{j,\ell}^{\mathrm{GT}}
\right\|_2 .
\end{equation}
The corners are obtained from the predicted vehicle position and orientation using a fixed-size ego-vehicle bounding box.

\item \textbf{Stop Success rate (SR, higher is better).} We first detect a stop as a contiguous interval during which the vehicle speed is below a threshold $v_{\mathrm{stop}}$. Only examples containing a valid ground-truth stop are included in the evaluation. A prediction is considered a successful stop if it contains a valid predicted stop and its stop-start time $t^{\mathrm{pred}}_{\mathrm{stop}}$ satisfies
\begin{equation}
t^{\mathrm{GT}}_{\mathrm{stop}}-2\epsilon_{\mathrm{stop}}
\leq t^{\mathrm{pred}}_{\mathrm{stop}}
\leq t^{\mathrm{GT}}_{\mathrm{stop}}+\epsilon_{\mathrm{stop}}.
\end{equation}
Stop SR is the fraction of eligible predictions satisfying this condition. Unless otherwise specified, following prior work~\citep{liang2026planning}, we use $v_{\mathrm{stop}}=0.5\,\mathrm{m/s}$, a minimum stop duration of $0.2\,\mathrm{s}$, and $\epsilon_{\mathrm{stop}}=0.7\,\mathrm{s}$. Notably, only when a keyframe is before the ground truth stop time will we consider it as eligible for metric computation as reported in Sec.~\ref{sec:aws_exp}.

\item \textbf{Go Success rate (SR, higher is better).} For a ground-truth trajectory that departs after stopping, a prediction is considered successful if it contains a valid stop and its stop-end time is within the specified tolerance of the ground-truth stop-end time:
\begin{equation}
t^{\mathrm{GT}}_{\mathrm{go}}-\epsilon_{\mathrm{early}}
\leq t^{\mathrm{pred}}_{\mathrm{go}}
\leq t^{\mathrm{GT}}_{\mathrm{go}}+\epsilon_{\mathrm{late}}.
\end{equation}
If the ground-truth vehicle remains stopped until the end of the prediction horizon, the prediction is successful only if it also remains stopped. Go SR is the fraction of eligible predictions satisfying the corresponding condition. We use $\epsilon_{\mathrm{early}}=\epsilon_{\mathrm{late}}=0.7\,\mathrm{s}$ following prior work~\citep{liang2026planning}.

\item \textbf{Roll-through Rate (Roll, lower is better).}
A prediction is classified as a roll-through if the ground-truth trajectory contains a valid stop but the predicted trajectory contains no valid stop interval. Roll-through rate is the fraction of eligible predictions classified as roll-throughs.

\item \textbf{Position Error ($\Delta_{\mathrm{pos}}$, lower is better).}
Position error measures the spatial discrepancy between the predicted and ground-truth stopping locations. Let $p^{\mathrm{GT}}_{\mathrm{stop}}$ denote the ground-truth stopping position and $p^{\mathrm{pred}}_{\mathrm{stop}}$ the position at which the predicted stop begins. It is calculated as
\begin{equation}
\Delta_{\mathrm{pos}}
=
\left\|
p^{\mathrm{pred}}_{\mathrm{stop}}
-
p^{\mathrm{GT}}_{\mathrm{stop}}
\right\|_2.
\end{equation}
If no valid predicted stop is detected, the final predicted waypoint is used as $p^{\mathrm{pred}}_{\mathrm{stop}}$. We report the mean position error in meters over eligible examples.

\item \textbf{Stop Duration Error ($\Delta_{\mathrm{dur}}$, lower is better).}
Let $d^{\mathrm{pred}}_{\mathrm{stop}}$ and $d^{\mathrm{GT}}_{\mathrm{stop}}$ denote the durations of the predicted and ground-truth stop intervals, respectively. Stop duration error is calculated as
\begin{equation}
\Delta_{\mathrm{dur}}
=
\left|
d^{\mathrm{pred}}_{\mathrm{stop}}
-
d^{\mathrm{GT}}_{\mathrm{stop}}
\right|.
\end{equation}
The predicted stop duration is set to zero when no valid predicted stop is detected. We report the mean duration error in seconds over eligible examples.

\item \textbf{MCQ Accuracy (higher is better).} MCQ accuracy is the accuracy of the model answering the question at the end of the clip, and is an indicator for the general scene understanding.

\end{itemize}

\subsection{LingoQA and WaymoQA}
\label{sec:qas}

\textbf{LingoQA.} LingoQA~\citep{marcu2024lingoqa} is a video question-answering benchmark for autonomous driving that covers both fine-grained scene perception and higher-level driving reasoning. Each sample contains a 4-second driving clip sampled at 1Hz (i.e. a total of 5 frames), together with free-form questions and answers concerning nine competencies, including action recognition, action justification, object to pay attention to, object identification, localization, description, counting, anticipation, and counterfactual reasoning. We test our memory on its evaluation set, which comprises 500 questions from 100 held-out scenarios, with two independently written reference answers per question, yielding 1,000 reference answers. Predictions are evaluated primarily using Lingo-Judge, a learned truthfulness classifier that determines whether a generated answer is semantically consistent with either reference answer; the resulting correctness scores are averaged over the evaluation set.

\textbf{WaymoQA.} WaymoQA~\citep{yu2025waymoqa} is a visual question-answering benchmark built from long-tail scenarios in the Waymo End-to-End Driving Dataset~\citep{xu2026wod} and designed to assess safety-critical driving reasoning. It includes both key-frame image questions and temporal video questions spanning perception, behavior prediction, planning, uncertainty, counterfactual reasoning, spatial and temporal relationships, traffic signals, and two-stage reasoning about immediate hazards and risks induced by an initial evasive action. In our experiment, we keep the video-VQA in the test set, which translates to 896 questions, and write memory for the video at 2.5Hz (same frequency as that in our dataset). The predictions are verified by rule-based scripts that extracts answers.

\subsection{Hyperparameters}
\label{sec:hyperparam}
Tab.~\ref{tab:hyperparam} lists the hyperparameter we use for our experiments. Notably, to align the training dynamics, for our no memory and CoT-as-memory baseline on the general driving dataset, we also train them on the same memory writing data as our method.

\begin{table}[t]
    \centering
    \small
    \renewcommand{\arraystretch}{1.08}
    \caption{List of hyperparameters for our algorithm.}
    \begin{tabular}{p{0.36\linewidth}p{0.56\linewidth}}
         \toprule
         \multicolumn{2}{c}{Supervised Fine-Tuning} \\
         \midrule
         global batch size & 128 \\
         learning rates & $1\times10^{-5}$ (vision encoder);
                              $1\times10^{-4}$ (language model and LM head) \\
         optimizer & fused AdamW \\
         Adam betas / epsilon & $(0.9, 0.95)$ / $1\times10^{-8}$ \\
         weight decay & 0.1 \\
         lr scheduler & cosine decay; 100-step linear warm-up;
                         minimum lr factor 0.01 \\
         gradient clipping & 1.0 \\
         precision & bfloat16 parameters; float32 reduction \\
         maximum sequence length & 8192 tokens \\
         input frames per sample & 4 (0.1s per frame) \\
         weight for memory-writing in SFT $c_1$ & 1 for general driving, 0 for all-way stop \\
         VQA weight in SFT $c_2$ & 10 \\
         \midrule
         \multicolumn{2}{c}{Semi-Closed-Loop RL} \\
         \midrule
         number of steps &
             420 (all-way stop, 2 epochs);
             760 (general driving, 1 epoch) \\
         global rollout batch & 144 trajectories per training step
                                (12 clips $\times$ 12 trajectories) \\
         rollouts per clip / group size & 12 \\
         optimizer & fused AdamW \\
         learning rate & $3\times10^{-6}$ \\
         lr scheduler  & constant  \\
         Adam betas / epsilon & $(0.9, 0.999)$ / $1\times10^{-6}$ \\
         weight decay & 0.01 \\
         sampling temperature / top-$p$ & 0.6 / 0.98 \\
         gradient accumulation interval &
             60 samples per GPU, with the final partial interval flushed \\
         optimizer steps per epoch & vary with trajectory length; typically 3-4 (144 trajectories $\times$ 20-25 steps / 16 GPUs / 60 samples)\\
         epochs of update per batch & 1 \\
         gradient clipping & 1.0 \\
         PPO clipping range & $\epsilon_{\mathrm{low}}=0.20$,
                              $\epsilon_{\mathrm{high}}=0.28$ \\
         KL coefficient & 0.001 \\
         importance sampling weight clipping & 2.0 \\
         MCQ reward weight & 0.5 \\
         memory reward &
             $\frac{1}{n}\sum_{t=i}^{n}r_t^{\mathrm{drive}}
              +0.5\,r^{\mathrm{MCQ}}$ for step $i$\\
         driving reward & $-0.2\,\mathrm{ADE}$ for
                           $\mathrm{ADE}<5\,\mathrm{m}$; $-1$ otherwise \\
         maximum response / context length & 512 / 4096 tokens \\
         training precision & bfloat16 parameters; float32 reduction \\
         \midrule
         \multicolumn{2}{c}{Evaluation} \\
         \midrule
         sampling temperature & 0.6 \\
         sampling top-$p$ & 0.98 \\
         parallel rollouts $K$ & 6 \\
         trajectory horizon & 6.4\,s \\
         trajectory discretization & 64 future waypoints at 10\,Hz \\
         memory horizon $T$ & 50 \\
         \bottomrule
    \end{tabular}
    \label{tab:hyperparam}
\end{table}

\section{Dataset Curation and Statistics}
\label{sec:datastats}

\subsection{All-Way Stop Dataset}
\label{sec:datastats-aws}

\subsubsection{Dataset Curation}
\label{sec:awsdata}

We curate our all-way stop dataset by first choosing the clips with all-way stop sign with the following standards: 1) ego has stopped in the clip, where ``stop'' is defined as speed $\leq 0.3$ m/s for $\geq 0.2$ seconds; 2) besides ego car, there must be at least one more car that has stopped, and the period of stopping must be no more than 1 second apart from ego's period of stopping (such that memory is needed); 3) At most one car has stopped before the clip starts, such that the stopping order can be recovered from the video. With such video, we then identify the moments of ``stop'' and ``go'' event of each car in the scene with trajectory information for ego and bounding box detection for other cars. In total, we curate 7436 clips for SFT training, 2479 clips for RL training and 2479 clips for evaluation with no joint clip between the splits.

\textbf{Memory curation.} As we care about ``stop'' and ``go'' in the all-way stop scenario, our memory is closely related to the stop/go status of each car and generated in a rule-based manner. We define six sequential states for the car: ``moving towards the crossing'', ``almost arrives at the crossing'', ``arrived at the crossing'', ``stopping'', ``start moving'', ``passing the crossing'', and assign them to our key frames according to the time of stop/go events. Memory at each frame is a concatenation of state description for each car (e.g. front car, left car, right car and ego); see Sec.~\ref{sec:aws-examples} for examples.

\textbf{VQA curation.} In this dataset, we focus on the following question: ``What is the stopping order between the cars?''. If the scene only involves two cars, we ask whether one of the car comes earlier; if three or more cars are involved, we ask that in which permutation did the car arrive. The reason for focusing on this question is because we find the other questions, such as going order, are severely biased (e.g. towards ego car moving last; see Sec.~\ref{sec:aws_stats} for details). 

\subsubsection{Statistics}
\label{sec:aws_stats}

Our all-way stop dataset contains 12394 clips, randomly divided into SFT training, RL training, and test sets in proportions of 60\%, 20\%, and 20\%, respectively, as described in the main paper. Of these, 9877 clips concern the arrival order of two cars, while 2517 concern the arrival order of three or more cars. As co-waiting and situations in which vehicles' stopping periods overlap or are separated by at most 1 second are important but relatively rare in driving, we curated this dataset by filtering over one million clips. Each memory entry contains 14.827 tokens on average.

\begin{figure}[t]
    \centering
    \subfigure[Number of keyframes]{
    \includegraphics[width=0.48\linewidth]{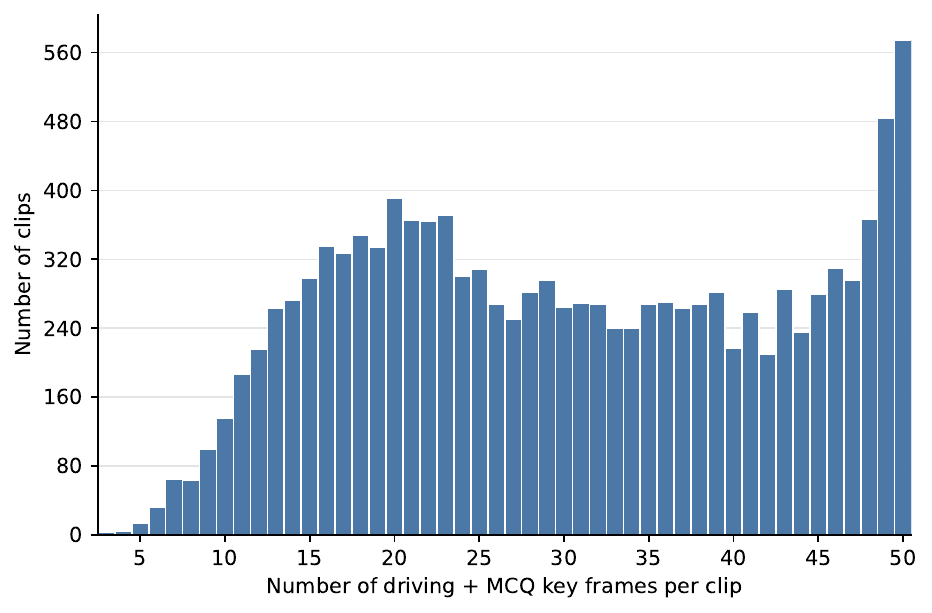}}
    \subfigure[Co-Wait Time]{
    \includegraphics[width=0.48\linewidth]{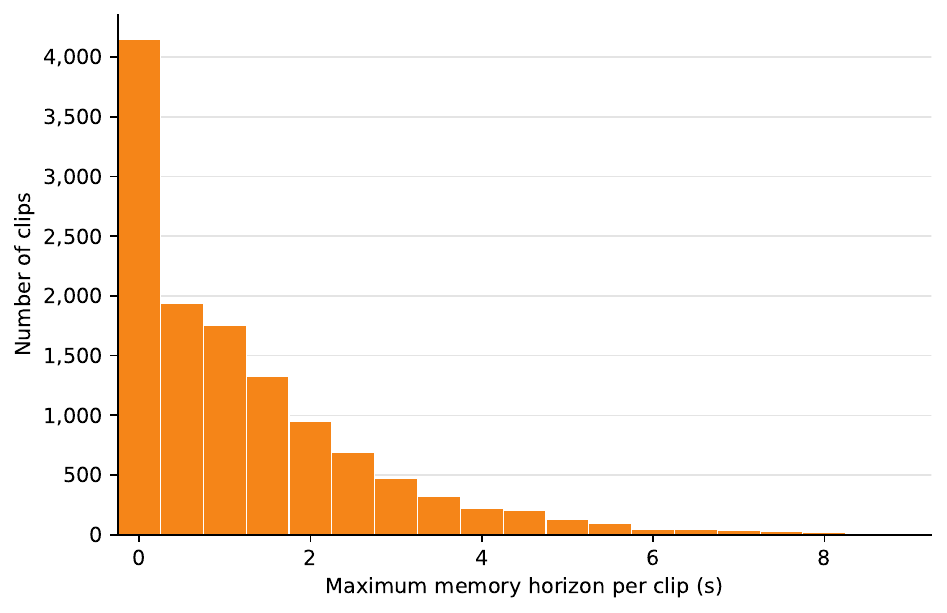}}
    \caption{An illustration of the two key statistics of our dataset. Panel (a) is the number of keyframes, which is the length of our clip; panel (b) is the co-wait time between ego and other car, which is a surrogate for the memory horizon required.}
    \label{fig:aws-stats}
\end{figure}


Fig.~\ref{fig:aws-stats} shows the distribution of the number of keyframes per clip in panel (a) and the maximum co-waiting duration between the ego vehicle and any other vehicle in each clip in panel (b). Panel (a) reflects the distribution of clip lengths, as consecutive keyframes are spaced $0.4$ seconds apart. Panel (b) serves as a proxy for the required memory horizon: when two vehicles are waiting simultaneously, the current frame alone may not reveal which arrived first. Notably, only 2.72\% of clips have a co-waiting duration of at least $5$ seconds, and none exceeds $10$ seconds. \textbf{Using this proxy, a memory horizon of $5$ seconds covers the estimated memory requirement for over 97\% of clips, while $10$ seconds covers all clips.} These observations guide our choice of memory horizon.


\textbf{Bias in departure order.} We examine the distributions of stop and go event orderings between the ego vehicle and other vehicles in our dataset. Across all pairs of stop events involving the ego vehicle and another vehicle, the ego vehicle stops earlier in 31.7\% of cases, later in 42.6\%, and simultaneously in 25.6\%. For pairs of go events, the ego vehicle starts moving earlier in 18.8\% of cases, later in 66.5\%, and simultaneously in 14.8\%. Given this strong bias toward the ego vehicle starting later, we exclude questions about departure order from our VQA task.

\subsubsection{Examples}
\label{sec:aws-examples}
Below we show an example of question, corresponding memory and answer in an all-way stop dataset, with input visualized in Fig.~\ref{fig:example_aws}. The rationale of the answer are generated by GPT-5.6 Luna by taking the memory, the question and the answer option as input. 

\begin{figure}[t]
    \centering
    \includegraphics[width=\linewidth]{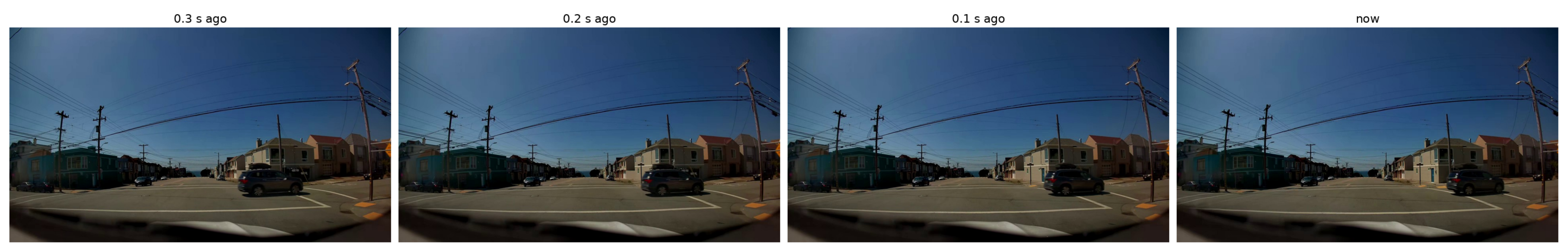}
    \caption{An example of the visual input for all-way stop dataset at the VQA frame.}
    \label{fig:example_aws}
\end{figure}

\textit{Example Question:}

{
\setlength{\fboxsep}{0.3cm}
\ovalbox{\small
\begin{minipage}{0.95\linewidth}
Determine the chronological arrival order of the following vehicles at the
crossing. Here, ``$>$'' means that the event on the left occurred more than
0.5 seconds earlier, while ``$=$'' means that the two events occurred within
0.5 seconds of each other; ``$\geq$'' means ``no later than'' and allows
either relation. If multiple answers apply, choose any correct one.

1. car ahead arrives at the crossing \qquad
2. ego arrives at the crossing \qquad
3. left car arrives at the crossing

\begin{center}
\textbf{A.} $3 \geq 2 \geq 1$ \qquad
\textbf{B.} $2 \geq 1 \geq 3$ \qquad
\textbf{C.} $2 \geq 3 \geq 1$ \\[0.6em]
\textbf{D.} $1 \geq 3 \geq 2$ \qquad
\textbf{E.} $1 \geq 2 \geq 3$ \qquad
\textbf{F.} $3 \geq 1 \geq 2$
\end{center}
\end{minipage}
}
}

\textit{Example Memory:}

{
\setlength{\fboxsep}{0.3cm}
\ovalbox{\small
\begin{minipage}{0.95\linewidth}
\texttt{[-0.4s]} The car ahead seems to start moving. The left car is passing
the crossing. The ego vehicle seems to start moving.

\texttt{[-0.8s]} The car ahead arrives at the crossing. The left car is passing
the crossing. The ego vehicle is stopping.

\centerline{$\ldots$}

\texttt{[-4.8s]} The car ahead is moving towards the crossing. The left car
stays still. The ego vehicle arrives at the crossing.

\texttt{[-5.2s]} The car ahead is moving towards the crossing. The left car
arrives at the crossing. The ego vehicle almost arrives at the crossing.

\centerline{$\ldots$}

\texttt{[-6.8s]} The car ahead, the left car, and the ego vehicle are moving
towards the crossing.

\end{minipage}
}
}

\textit{Example CoT as memory:}

{
\setlength{\fboxsep}{0.3cm}
\ovalbox{\small
\begin{minipage}{0.95\linewidth}

\texttt{[-0.4s]} the crosswalk due to the pedestrian crossing ahead.

\texttt{[-0.8s]} the intersection because the crossing vehicle has cleared.

\centerline{$\ldots$}

\texttt{[-4.8s]} to yield to the vehicle crossing the intersection.

\texttt{[-5.6s]} to stop at the stop sign controlling the intersection.

\centerline{$\ldots$}

\texttt{[-6.8s]} to stop due to the stop sign at the intersection.

\end{minipage}
}
}

\textit{Example Answer:}

{
\setlength{\fboxsep}{0.3cm}
\ovalbox{\small
\begin{minipage}{0.95\linewidth}
\textbf{Ours.} According to the history, the car ahead arrives at the crossing
(1) 0.8 seconds ago, the ego vehicle arrives at the crossing (2) 4.8 seconds
ago, and the left car arrives at the crossing (3) 5.2 seconds ago. Since the
last two events are within 0.5 seconds of each other, the order is
$3 = 2 > 1$. Thus, the answer is \textbf{A}.

\medskip
\textbf{CoT as memory.} The left car is already passing through and then
clearing the intersection, so it arrived first. The ego vehicle remains behind
the stop line during the final visual inputs, while the car ahead does not
enter the crossing before the ego vehicle. Thus, the answer is \textbf{A}.
\end{minipage}
}
}

\subsection{General Driving Dataset}
\label{sec:datastats-gen}

\subsubsection{Dataset Curation}
\label{sec:generaldata}

We curate a dataset that tests the performance of our algorithm in more diverse and challenging driving scenes. To identify such scenes, we use Qwen3-VL-30B-A3B-Instruct~\citep{bai2025qwen3} and collect a set of 45446 clips (27246 SFT training set, 9088 RL training set and 9112 test set) with challenging scenes; see Sec.~\ref{sec:prompt} for the prompt. 

\textbf{Decision graph.} With such clips available, the key challenge of language-based memory is to balance between brevity and inclusion of the objects that indeed affect driving; describing each of the 20 pedestrians on the sidewalk who never enters the road will only make the language memory less efficient than image input. Thus, we design a novel agentic framework in which the agent first identify and track items of interests in the frame with tools according to the video context. The agent then builds a \textit{decision graph}, which describes the objects of interest on the road and their relation with the ground truth action of ego car on this frame. By doing so, we ensure that all items described in the memory are relevant to the driving decision. Such graph is built with the following procedure:

\begin{itemize}


\item \textit{Object Grounding:}
Given the video clips, we first apply specialized object detection and map-element perception models to localize driving-relevant entities, including vehicles, Vulnerable Road Users (VRUs), crosswalks, lane areas, road signs, and traffic lights. Each detected entity is represented by a grounded visual prompt, such as an indexed bounding box, mask, or polygon overlaid on the corresponding frame. These prompts establish an explicit correspondence between image regions and unique object or map-element identifiers.

\item \textit{ Semantic Association:} With the grounded object, a vision-language model is then queried with the visually prompted frames to associate higher-level semantics with the grounded entities, such as their type, state, motion, and relation to the ego vehicle or road layout. The resulting grounded entities and semantic attributes are converted into structured evidence records. Such evidences form a comprehensive set of candidate nodes in the decision graph, which will be filtered in the next step.

\item \textit{Evidence Gathering:} Given a specific key frame from the video and the related evidences, we call GPT-5.6 Luna to select the ones that affects driving with rationale, which is essentially a filter over the evidences. The model is also asked to review the frame and add any missing evidences in the image.

\item \textit{Retrieve Final Action:} We call GPT-5.6 Luna to generate future action of ego car in natural language, and calibrate with the ground truth future trajectory.

\item \textit{Build Graph:} With both the evidences and final action available, we prompt GPT-5.6 Luna to build a graph, where each evidence and final action is a node and the connection between them is an edge. The model is asked to give rationale and a category (e.g. causal, temporal, intentional, regulatory, etc.) for each edge connected.

\item \textit{Description-Based Linting:} With the decision graph generated, we use a rule-based script to conduct a sanity check on the graph, and regenerate those which have invalid node, empty rationale, or no connection to the final action.
\end{itemize}

\begin{figure}[t]
    \centering
    \includegraphics[width=0.6\linewidth]{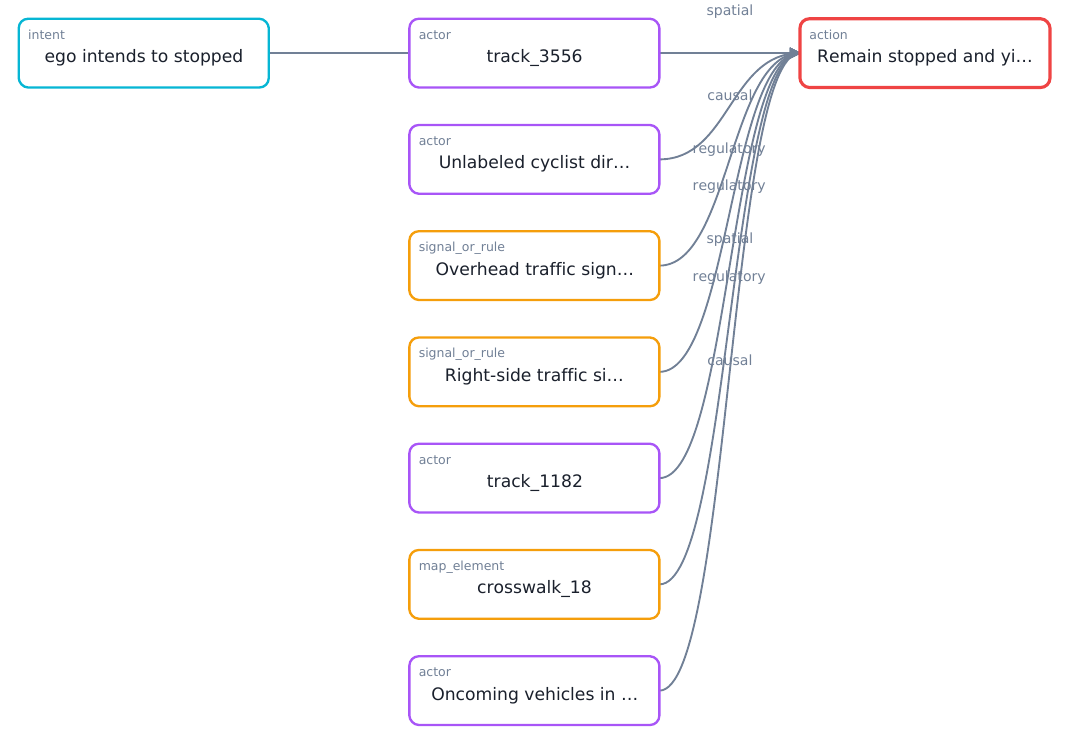}
    \caption{An example of a decision graph.}
    \label{fig:decision-graph}
\end{figure}

Fig.~\ref{fig:decision-graph} shows an example of decision graph. Notably, as the decision graph requires multiple VLM calls and can be expensive in practice, we only generate such graph for every 1.2s (i.e. every three keyframes), and then use GPT-5.6 Luna to fill in memory for the time between decision graphs.

\textbf{Memory curation.} With decision graphs for each key frame, we prompt GPT-5.6 Luna to curate the memory by first describing the nodes in the graph connected to the final action, then asking the model to review and merge memories from the same clip and ``connect'' them by keeping track of the same object throughout the text. We found that merging memory is vital, as repeated description of the same item is not only lengthy, but can also be misleading; for example, the stop caused by a pedestrian on the road, if described repeatedly, can mislead the model on the judgment of the traffic light. See Sec.~\ref{sec:port_exp} for an example. Notably, the memories are generated independently from VQA; this is because if the memory overfits toward a particular VQA question, it will only focus on a small set of object of interest and miss most items in the scene that affects driving.

\textbf{VQA curation.} We curate two types of VQA questions with GPT-5.6 Luna: the easy type and the hard type, where the former directly asks about an object of interest in the video, while the latter is multi-hop, i.e., asks the relative change between two time frames by checking the decision graph difference between the two frames. We use the former in the training set and the latter in the test set. This is because we found that training and testing on the same type of VQA questions will overfit and only illustrates the memorization of the dataset instead of genuine scene understanding; see Sec.~\ref{sec:example_gen} for examples of the dataset, Sec.~\ref{sec:prompt} for prompts on curating the questions, and Sec.~\ref{sec:overfit_vqa} / Sec.~\ref{sec:ablation_dg} for ablations.

\subsubsection{Statistics}

Our general driving dataset has a total of 27246 SFT training clips, 9088 RL training clips and 9112 test clips; Fig.~\ref{fig:gen-stats} shows the distributions of the number of keyframes. On average, the length of each entry of memory is 29.054 tokens.

\begin{figure}[t]
    \centering
    \includegraphics[width=0.6\linewidth]{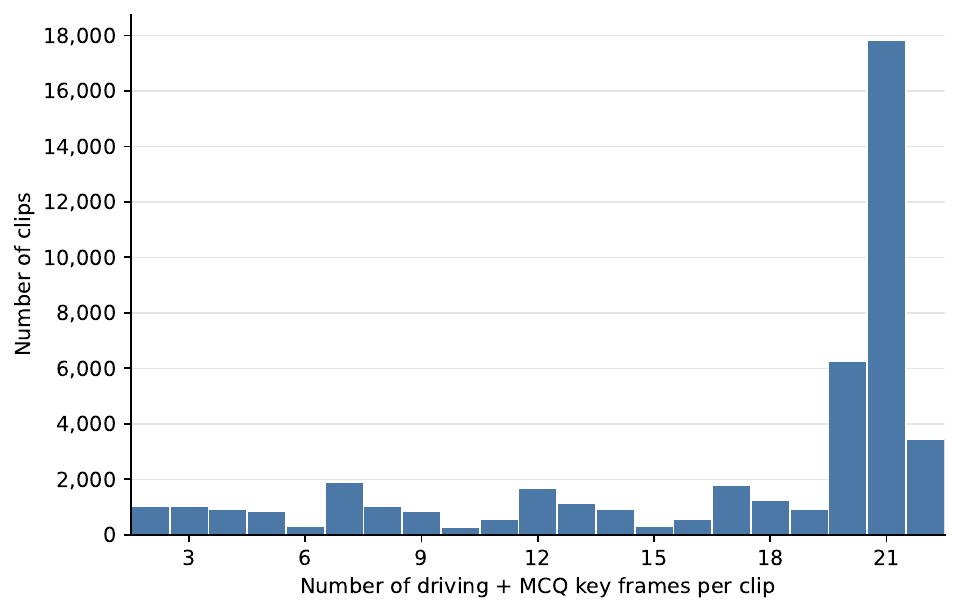}
    \caption{The distributions of the number of keyframes in the general driving dataset. Most clips are around 21 frames, which translates to 8 seconds of video.}
    \label{fig:gen-stats}
\end{figure}

\subsubsection{Examples}
\label{sec:example_gen}
Below we show an example of questions and corresponding memories in our general driving dataset, with the image input illustrated in Fig.~\ref{fig:example_gen}:

\begin{figure}[t]
    \centering
    \includegraphics[width=\linewidth]{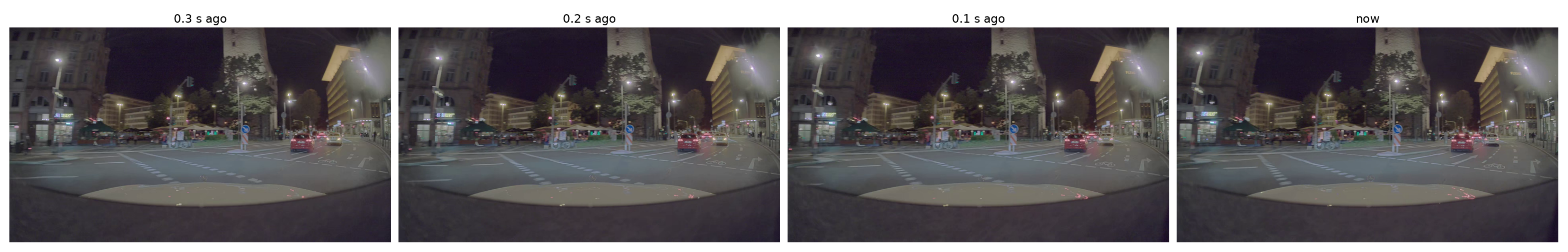}
    \caption{An example of the visual input for the general driving dataset at the VQA frame.}
    \label{fig:example_gen}
\end{figure}

\textit{Example Easy Question:}

{
\setlength{\fboxsep}{0.3cm}
\ovalbox{\small
\begin{minipage}{0.95\linewidth}
What does the ego vehicle do throughout the sequence as it approaches and
traverses the intersection?

\textbf{A.} It stops before the intersection, waiting for cross traffic to
clear.

\textbf{B.} It turns left across opposing traffic after the signal changes
red.

\textbf{C.} It pulls toward the curb and parks beside the tour storefront.

\textbf{D.} It continues straight across the intersection without changing
its lane position.

\textbf{E.} It reverses briefly, then waits behind the stopped red car ahead.

\textbf{F.} It proceeds through the green-lit intersection, following the
roadway's rightward curve.
\end{minipage}
}
}

\textit{Example Hard Question \#1:}

{
\setlength{\fboxsep}{0.3cm}
\ovalbox{\small
\begin{minipage}{0.95\linewidth}
What sequence links the scene at 5.6 seconds ago with the later situation at
2.0 seconds ago?

\textbf{A.} Left-side traffic and a right-side parked vehicle constrained
passage; later, a red vehicle followed ahead right-to-left.

\textbf{B.} Left-side traffic and a right-side parked vehicle constrained
passage; later, a red vehicle crossed ahead left-to-right.

\textbf{C.} Left-side traffic and a right-side parked vehicle allowed passage;
later, a red vehicle crossed ahead right-to-left.

\textbf{D.} Right-side traffic and a left-side parked vehicle constrained
passage; later, a red vehicle crossed ahead right-to-left.

\textbf{E.} Left-side traffic and a right-side parked vehicle constrained
passage; later, a white vehicle crossed ahead right-to-left.

\textbf{F.} Left-side traffic and a right-side parked vehicle constrained
passage; later, a red vehicle crossed ahead right-to-left.
\end{minipage}
}
}

\textit{Example Hard Question \#2:}

{
\setlength{\fboxsep}{0.3cm}
\ovalbox{\small
\begin{minipage}{0.95\linewidth}
Which sequence best describes the ego vehicle's response as it approached and
then traversed the intersection?

\textbf{A.} It began stopped behind traffic, then merged right while continuing
through the intersection.

\textbf{B.} It began stopped near the crosswalk, then turned left after nearby
pedestrians cleared its path.

\textbf{C.} It began stopped amid tight lateral clearances, then proceeded
straight while preserving crosswalk clearance.

\textbf{D.} It began moving beside parked vehicles, then stopped only after
clearing the crosswalk.

\textbf{E.} It began moving through the intersection, then continued straight
after stopping beside roadside vehicles.

\textbf{F.} It began moving behind traffic, then remained stopped as oncoming
vehicles cleared the intersection.
\end{minipage}
}
}

\textit{Example CoT as Memory:}

{
\setlength{\fboxsep}{0.3cm}
\ovalbox{\small
\begin{minipage}{0.95\linewidth}
\texttt{[-0.4s]} the right lane change due to slower traffic ahead in the
current lane.

\centerline{$\ldots$}

\texttt{[-2.0s]} the right turn into the rightmost lane because the signal
permits the turn, the lead car ahead is already turning, and the crosswalk and
bike lane to the right are clear of conflicts.

\centerline{$\ldots$}

\texttt{[-4.8s]} to merge right behind the lead car because the intersection
and crosswalk ahead require a rightward merge.

\texttt{[-5.6s]} to turn right at the intersection to follow the lead vehicle.

\texttt{[-6.4s]} to stop due to the red traffic light.

\centerline{$\ldots$}

\texttt{[-8.4s]} to merge left due to the lead vehicle moving left ahead.
\end{minipage}
}
}

\textit{Example Memory:}

{
\setlength{\fboxsep}{0.3cm}
\ovalbox{\small
\begin{minipage}{0.95\linewidth}
\texttt{[-0.4s]} The ego vehicle continues straight past the intersection
crosswalk behind the red vehicle, with additional traffic ahead and a marked
bicycle lane on the right.

\centerline{$\ldots$}

\texttt{[-2.0s]} The ego vehicle continues through the intersection while the
red crossing vehicle moves farther left and vulnerable road users remain near
the crosswalk.

\centerline{$\ldots$}

\texttt{[-5.6s]} The ego vehicle remains stopped as the oncoming vehicle
continues past on the left and vulnerable road users remain near the crosswalk
and right curb.

\texttt{[-6.0s]} The ego vehicle remains stopped with the vehicle ahead
waiting, a large black vehicle at the right edge, an oncoming vehicle alongside
on the left, and vulnerable road users near both sides of the crosswalk.

\centerline{$\ldots$}

\texttt{[-8.4s]} The ego vehicle is stopped at a nighttime urban intersection
approach with an oncoming vehicle on the left, a vehicle directly ahead, a
vehicle close along the right edge, crosswalks ahead and left, and pedestrians
near both crossings.
\end{minipage}
}
}

\textit{Example Answer:}

{
\setlength{\fboxsep}{0.3cm}
\ovalbox{\small
\begin{minipage}{0.95\linewidth}
\textbf{Ours.}

\textbf{Easy: F.} The ego vehicle proceeds through the green-lit intersection,
following the roadway's rightward curve.

\textbf{Hard \#1: F.} Traffic occupied the left side while a parked vehicle
narrowed the right side, constraining straight passage. A red vehicle later
crossed ahead from right to left.

\textbf{Hard \#2: C.} The ego vehicle was initially stopped while nearby
traffic and close roadside vehicles constrained the approach. It then began
moving straight through the intersection while preserving crosswalk clearance.

\medskip
\textbf{CoT as memory.}

\textbf{Easy: F.} The ego vehicle proceeds through the green-lit intersection,
following the roadway's rightward curve.

\textbf{Hard \#1: F.} The constrained passage was followed by the red vehicle
crossing ahead from right to left.

\textbf{Hard \#2: C.} The ego vehicle changes from being stopped amid tight
lateral clearances to proceeding through the intersection.

\end{minipage}
}
}

\subsection{Prompts for Data Generation}
\label{sec:prompt}

Below we show the core prompt for data generation. We first show the prompt for filtering the challenging clips for the general driving dataset. We only disclose excerpts of the prompt for both brevity and protection of proprietary information:

{
\setlength{\fboxsep}{0.3cm}
    \ovalbox{\small
        \begin{minipage}{0.95\linewidth}
        
Find challenging scenario that requires deep human reasoning to drive through safely. Key criteria include:

    - Complex interactions between multiple road users (e.g., pedestrians, cyclists, vehicles) that require careful negotiation and anticipation.

    - Unusual or unexpected events (e.g., sudden pedestrian crossing, erratic behavior from other drivers) that demand quick and accurate decision-making.
    \\
    
    and exclude:
    \\
    - only concern is low visibility or adverse weather, without complex road user interactions or unexpected events.
    \\
    ...
    \\
\end{minipage}
    }
}

{
\setlength{\fboxsep}{0.3cm}
    \ovalbox{\small
        \begin{minipage}{0.95\linewidth}

         Analyze the video and motion to deduce ego's driving decision at the keyframe according to the pre-defined decision list below. Here's a list of pre-defined LONGITUDINAL and LATERAL driving decisions:
        
        - LONGITUDINAL decisions
        
            -- Stop
            
            Decelerate to, hold at stop/yield lines or other control points (traffic light, stop sign, school bus/railroad rules, blocked path by lead vehicle).
            
            -- Yield
            
                Yield to the [pedestrian/cross-traffic/cyclist/emergency vehicle] / give way for / wait for priority traffic to clear.

    ...

    Priority policy (must follow):

        1. First check if ego is conducting 'interesting' decisions:
        
           - LONGITUDINAL: 'Yield', 'Stop', 'Pass/Overtake', 'Gap-search', 'Adapt Speed'.
        
           - LATERAL: 'Lane Change', 'Nudge', 'Turn', 'Merge', 'Split'.
           
           - If any interesting decision is clearly present, output that decision and do not downgrade to generic 'Keep Distance', 'Keep Lane' or 'Resume Speed'.

         2. Lane change and nudge are top-priority lateral decisions:
       
       ...

 Critical definitions:
     
        1. Lead vehicle means a vehicle ahead in the same lane and moving in the same direction.
        
    ...

    Best practices:

    1. Use "crosswalk" and only mention it when a pedestrian is walking across the road.

    ...
    
    Format:

    ...
       
        \end{minipage}
    }
}

Next, we show the prompt for VQA curation of the general driving dataset:

{
\setlength{\fboxsep}{0.3cm}
    \ovalbox{\small
        \begin{minipage}{0.95\linewidth}
        
       \textbf{\# Easy questions}

       You are an expert autonomous-driving scene analyst.

You are given front-wide-camera frames sampled at 2 frames per second from
exactly the interval beginning at this clip's first retained driving frame and
ending at its final retained driving frame. Frames are ordered oldest to newest.
A prior caption is supplied only to help locate potentially relevant actors; it
may be wrong, so determine behavior from the frames.

Prior caption:
\{prior\_description\}

Create one six-option multiple-choice question about a temporal behavior or ego-object interaction that requires the full interval, not only the final frame. If no other actor is salient, ask about ego motion or traffic-control response. Audio is unavailable.

Requirements:

- Exactly six mutually exclusive options A-F.

- correct\_answer must be A before downstream deterministic shuffling.

- Options must be plausible and scene-compatible, with matched grammar, specificity, certainty, and action intensity.

- Each option must have 8-12 words; longest-shortest difference at most 3 words.

- Do not use memory notes or information outside the supplied frames.

- Return only the requested JSON object.

        \end{minipage}
    }
}

{
\setlength{\fboxsep}{0.3cm}
    \ovalbox{\small
        \begin{minipage}{0.95\linewidth}
        
       \textbf{\# Hard question 1}

      You are an expert autonomous-driving scene analyst and a rigorous temporal
multiple-choice-question author.
You receive two synchronized sources for one driving clip:

1. Front-wide-camera images ordered oldest to newest. Immediately before every
   image is a TEXT label of the form ``Frame time: N.N seconds ago'' or
   ``Frame time: now''. The label is not drawn on the image. Images are sampled at
   approximately 2 frames per second, so adjacent labels are usually about
   0.5 seconds apart. ``Now'' is the final question image.
   
2. Raw structured decision-graph snapshots sampled on a 1.2-second source grid.
   Every snapshot is explicitly marked with its time relative to ``now'' and
   contains compact `evidence', `nodes', and `edges'. Node fields include id,
   type, label, and confidence. Edge fields include source, target, relation,
   rationale, and confidence. Use this structure to follow entities and
   decision relations over time, but verify every question-critical claim
   against the images.

Important graph interpretation rules:

- An `action' node is a hypothetical safe decision, not necessarily observed ego
  behavior. Never present it as something ego actually did unless images show it.
  
- Internal node/track IDs are for linking snapshots only and must never appear in
  a question or option.
  
- A snapshot marked slightly "after now" exists only because source and camera
  timebases can differ within the stated alignment tolerance. Never use a fact
  visible only after now.
  
- Missing snapshots must not be interpolated or invented.

Raw 1.2-second decision-graph context:\{decision\_graph\_context\}

Produce exactly ONE `past\_anchor' six-option MCQ. Analyze the complete clip before writing:

1. Track every salient actor, traffic control, road feature, hazard, and the ego
   vehicle across time. Resolve which references denote the same entity.
   
2. Identify state changes, temporal order, interactions, and causal or regulatory
   links. Images are authoritative when they conflict with graph snapshots.
   
3. Choose one exact supplied historical time and write it naturally as
   ``N.N seconds ago''. Never use notation such as ``t=-2.4s''.
   
4. The answer must require both the anchored historical evidence and evidence
   from a second supplied time. Do not ask for a static fact at the anchor alone.
   
5. Apply the final-image-only counterfactual: hide every historical image and all
   graph snapshots. If a careful observer could answer from "now" alone,
   reject and rewrite the question.

Time and spatial ambiguity rules:

- Do not ask ``when did X happen?'' when X persists across multiple images. A
  state covering an interval cannot be assigned a unique point timestamp.
  
- Prefer questions about a change after the anchor, an interaction spanning the
  anchor and a later time, or the complete anchored state needed to explain a
  later response.
  
- Exactly one option must be completely correct under every reasonable reading.

- Do not mix overlapping spatial axes. If an actor is ahead-left, ``ahead'' and
  ``left'' are both partially true. Ask about one explicitly named axis:
  lateral = left/center/right or longitudinal = ahead/alongside/behind. Otherwise
  use matched composite relations such as ahead-left/ahead-center/ahead-right in
  every option.
  
- Avoid nested categories such as moving versus accelerating, stopped versus
  yielding, vehicle versus truck, or crossing versus ahead-left.
  
- Distinguish actor identity whenever two similar actors could match.

- State in `ambiguity\_check' why A is uniquely complete, why the strongest
  distractor is false, and why the anchor does not represent a persistent-state
  timing ambiguity.

Distractor quality and wording:

- Distractors must be realistic near-misses for this exact scene. Prefer changing one subtle attribute: temporal order, matched spatial relation, speed trend,
  signal state, actor identity, or ego response.
  
- Never use obviously absurd, dangerous, or extreme actions merely to fill an
  option. In particular, no option may use reverse/reversed/reversing,
  swerve/swerved/swerving, crash/collide/collision, U-turn, driving onto a
  sidewalk or curb, entering oncoming traffic, spinning, or rolling over.
  
- Do not invent an actor type or maneuver unsupported by the clip.

- Avoid giveaway modifiers such as ``suddenly'', ``recklessly'', ``violently'',
  ``immediately'', ``without reason'', ``at full speed'', or ``completely ignored''.
  
- Match grammar, semantic axis, specificity, certainty, and action intensity across all six options.

Output requirements:

- Exactly six mutually exclusive options A-F.

- `correct\_answer` must be A before downstream deterministic shuffling.

- Each option must contain 8-16 words; longest-shortest difference at most  5 words.

- Do not ask about a static fact, count, color, or location visible at ``now''.

- Do not ask about audio, hidden intent, or information outside supplied inputs.

- Do not expose internal graph IDs or refer to graphs, memory, notes, or frames in the question or options.

- Return only the requested JSON object.

        \end{minipage}
    }
}

{
\setlength{\fboxsep}{0.3cm}
    \ovalbox{\small
        \begin{minipage}{0.95\linewidth}
        
       \textbf{\# Hard question 2}

... (same as hard question 1) ...

Produce exactly ONE `multi\_hop' six-option MCQ.

Analyze the complete clip before writing:

1. Track every salient actor, traffic control, road feature, hazard, and the ego
   vehicle across time. Resolve which references denote the same entity.
   
2. Identify state changes, temporal order, interactions, and causal or regulatory
   links. Images are authoritative when they conflict with graph snapshots.
   
3. Build an evidence chain from at least two distinct supplied times. At least
   one indispensable fact must be historical.
   
4. The answer must require combining at least two different decision-relevant
   facts, such as an actor state change plus ego response, a traffic-control
   transition plus actor interaction, or temporal order plus the resulting
   maneuver. Comparing two locations or restating one change is not multi-hop.
   
5. Apply the final-image-only counterfactual: hide every historical image and all
   graph snapshots. If a careful observer could answer from ``now'' alone,
   reject and rewrite the question.
   
6. Do not put numerical timestamps in the question or options.

Mandatory ambiguity audit:

... (same as hard question 1) ...

        \end{minipage}
    }
}

Then, we show the prompt for memory merging of the general driving dataset below:

{
\setlength{\fboxsep}{0.3cm}
    \ovalbox{\small
        \begin{minipage}{0.95\linewidth}
        Fill and rewrite the complete 0.4-second memory sequence for this driving clip. Every labeled image
is one required output timestamp. Some images include an independently generated anchor memory;
the other images explicitly say that no independent memory exists. Return exactly one memory for
every image timestamp, in the same order.

Rules:

1. Use every image, the anchor memories, and temporal continuity together. Preserve all factual
   anchor information, but do not invent details unsupported by either an image or an anchor.
   
2. Locate changes at the earliest 0.4-second image where they are visibly supported. For example,
   if a signal changes between two 1.2-second anchors, inspect all intermediate images and mention
   the new color only at the frame where it first appears. Apply the same frame-accurate treatment
   to ego acceleration/deceleration, stopping, starting, turning, lane changes, actor motion, and
   visibility changes.
   
3. The first output establishes the scene. Later outputs are concise temporal deltas. Omit
   unchanged static background such as an already established signal color, sign, crosswalk, lane,
   snowbank, or road surface until it changes. Every frame must still contain useful supervision:
   use concise continuity wording for reliably visible dynamic subjects when no new event occurs.
   
4. Track the same physical subject across frames. Introduce it once with enough visible appearance,
   location, and behavior to identify it uniquely, then use a stable natural reference while
   unambiguous.
   
5. When reliable observed ego behavior is visible, put it first and begin with "The ego vehicle".
   Describe ego motion only qualitatively, such as stopped, starting, accelerating, decelerating,
   moving straight, turning, or changing lanes. NEVER output a numeric ego speed or units such as
   m/s, mph, or km/h in any output frame. If ego behavior is uncertain, omit the ego vehicle
   completely. Never write "The ego motion remains unspecified", "The ego vehicle's observed
   motion is unspecified", "No ego maneuver is evidenced", or any equivalent placeholder.
   
6. An anchor memory equal to "N/A" is a failed generation, not a fact. Recover that frame from its
   image and temporal context; never copy "N/A" into the output.
   
7. Never emit track IDs, lane ordinals, node/evidence IDs, or raw timestamps inside
   memory\_text. Do not report hypothetical safe actions or future plans as actual behavior.
   
8. Return one concise present-tense sentence per timestamp. Do not leave any output empty.

The response must match the requested JSON shape. Copy each input timestamp exactly into the
corresponding output object; put only the completed sentence in memory\_text.
       
        \end{minipage}
    }
}

Finally, we report the prompt for generating answers for VQA:

{
\setlength{\fboxsep}{0.3cm}
    \ovalbox{\small
        \begin{minipage}{0.95\linewidth}
        
       You are verifying a fixed-ground-truth multiple-choice question about an
autonomous-driving scene.

You receive:
1. one current ego front-wide-camera image;
2. timestamped scene memory from earlier frames;
3. a fixed question with six fixed options; and
4. the fixed correct answer label and text.

Write one concise rationale that explains why the fixed correct answer follows
from the current image and the supplied earlier scene memory.

Requirements:
- Do not change or challenge the supplied answer.
- Use only evidence visible in the current image or stated in the supplied
  memory.
- Describe the relevant temporal change or ego-object interaction directly.
- Do not mention "image", "frame", "memory", "notes", "prompt", or the answer
  label.
- Do not introduce precise numbers, identities, directions, or events that are
  absent from the supplied evidence.
- Use one or two sentences and no answer tag.
- If the supplied evidence genuinely contradicts the fixed answer or cannot
  support any concise rationale, set supported to false instead of inventing
  evidence.
- Return only the requested JSON object.

Earlier scene memory:
\{memory\_block\}

Fixed question:
\{question\}

Fixed options:
\{options\}

Fixed correct answer:
\{correct\_label\}. \{correct\_text\}

        \end{minipage}
    }
}

\section{Additional Experiment Results}
\label{sec:add_exp}
\subsection{More Qualitative Results}
\label{sec:qualitative}

\begin{figure}[t]
    \centering
    \includegraphics[width=\linewidth]{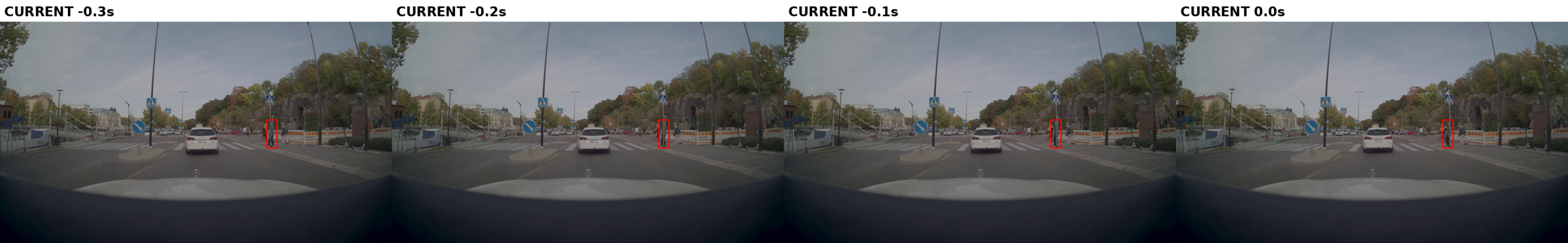}
    \caption{An example where memory benefits general driving; a pedestrian in the red box is occluded behind the crosswalk sign, which is hard to discover; the intention of the pedestrian is even harder to tell.}
    \label{fig:general_example_now}
\end{figure}

\begin{figure}[t]
    \centering
    \includegraphics[width=\linewidth]{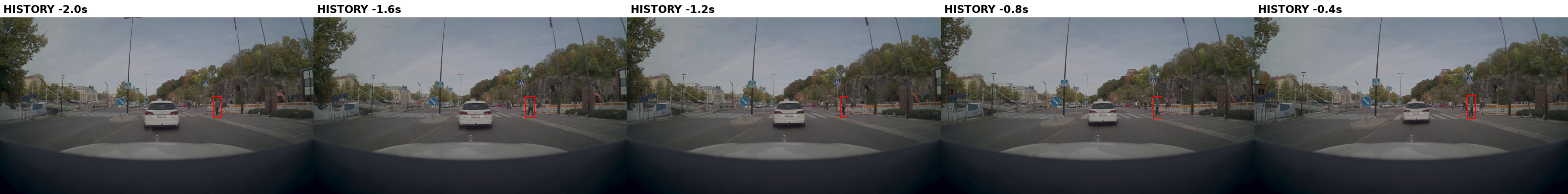}
    \caption{The previous frames of the same clip as that in Fig.~\ref{fig:general_example_now}. With these frames, it is clearly shown that the pedestrian is walking towards the road and the ego car should yield.}
    \label{fig:general_example_past}
\end{figure}

In this section, we illustrate another qualitative example in our general driving dataset (Sec.~\ref{sec:gen_exp}) where language memory helps. Fig.~\ref{fig:general_example_now} shows that a pedestrian is occluded by the crosswalk sign, which is hard to discover in current input; however, once the agent has access to memory as shown in Fig.~\ref{fig:general_example_past}, it is much easier to see that the pedestrian has an intention of walking onto the road. Below shows an example of memory of our agent:

{
\setlength{\fboxsep}{0.3cm}
    \ovalbox{\small
        \begin{minipage}{0.95\linewidth}
[ -0.4 s ] The ego vehicle continues straight toward the crosswalk behind the vehicle ahead while pedestrians remain on the right and center-left portions of the crossing.

[ -0.8 s ] The ego vehicle continues straight behind the vehicle ahead; the pedestrian is farther into the right side of the crosswalk, with another pedestrian near the center-left refuge area.

[ -1.2 s ] The ego vehicle continues straight behind the vehicle ahead as the \textbf{pedestrian moves leftward across the right side of the crosswalk} and the construction barriers remain along the right edge.

[ -1.6 s ] The ego vehicle is moving straight toward a marked crosswalk, following a vehicle directly ahead, with another vehicle ahead-left, \textbf{a pedestrian crossing from the right}, and red-and-white construction barriers narrowing the right edge.

...

        \end{minipage}
    }
}

The memory clearly records the existence and intention of pedestrian. The predicted trajectories are illustrated in Fig.~\ref{fig:general_example_traj}, where our method outperforms no memory and CoT-as-memory.

\begin{figure}[t]
    \centering
   \subfigure[No memory]{
    \includegraphics[width=0.32\linewidth]{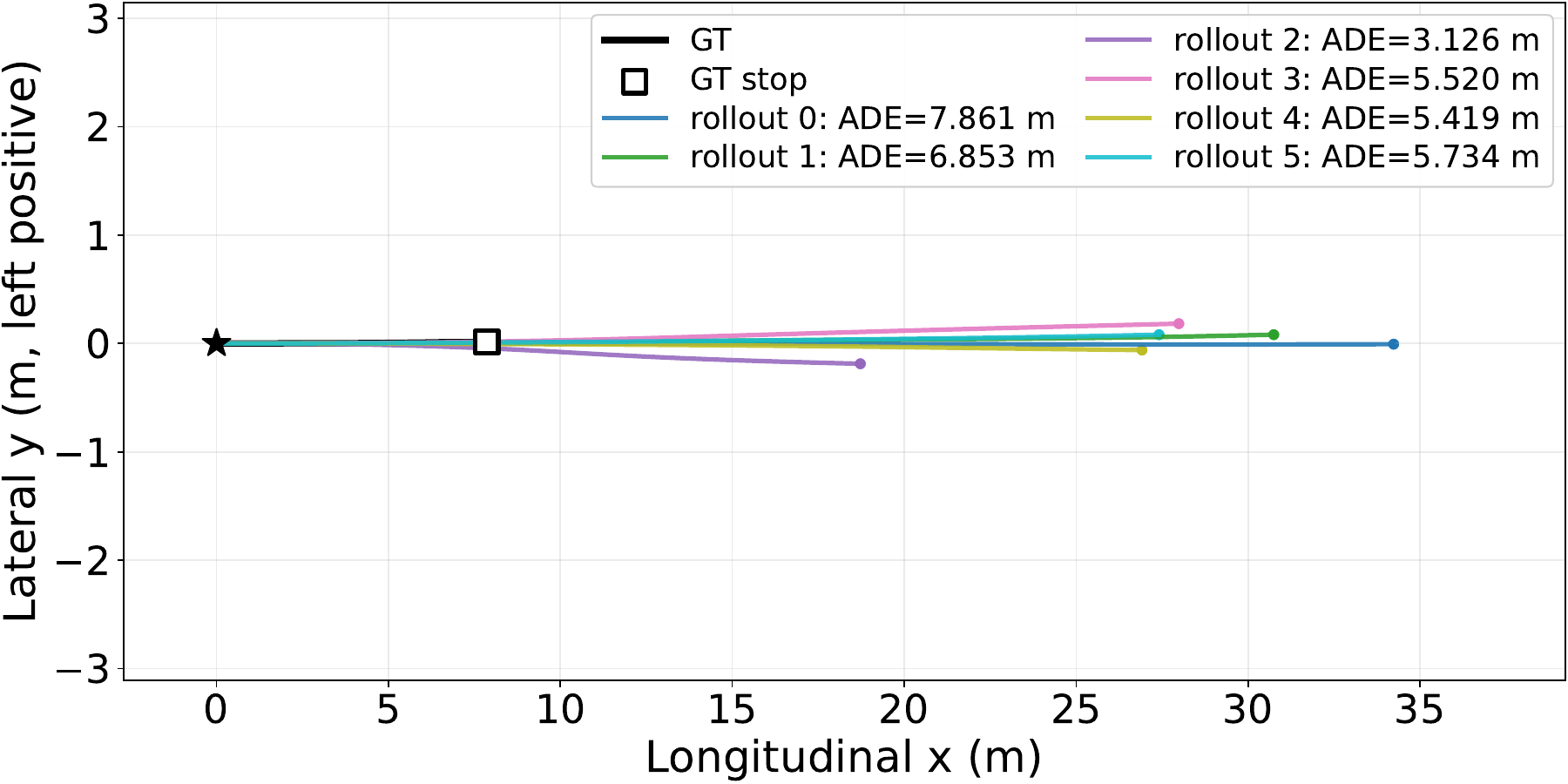}}
    \subfigure[CoT-as-memory]{
    \includegraphics[width=0.32\linewidth]{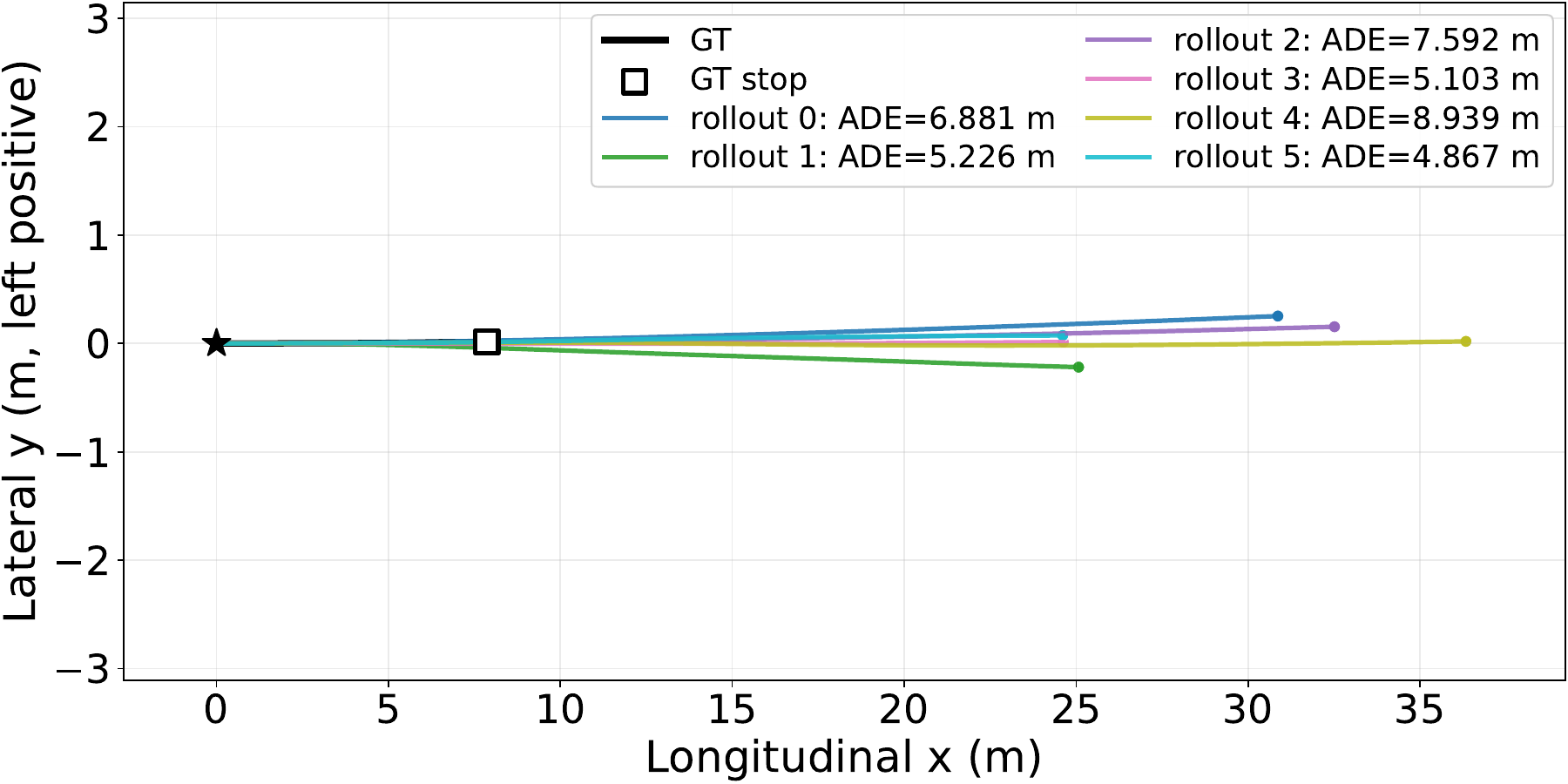}}
    \subfigure[Ours]{
    \includegraphics[width=0.32\linewidth]{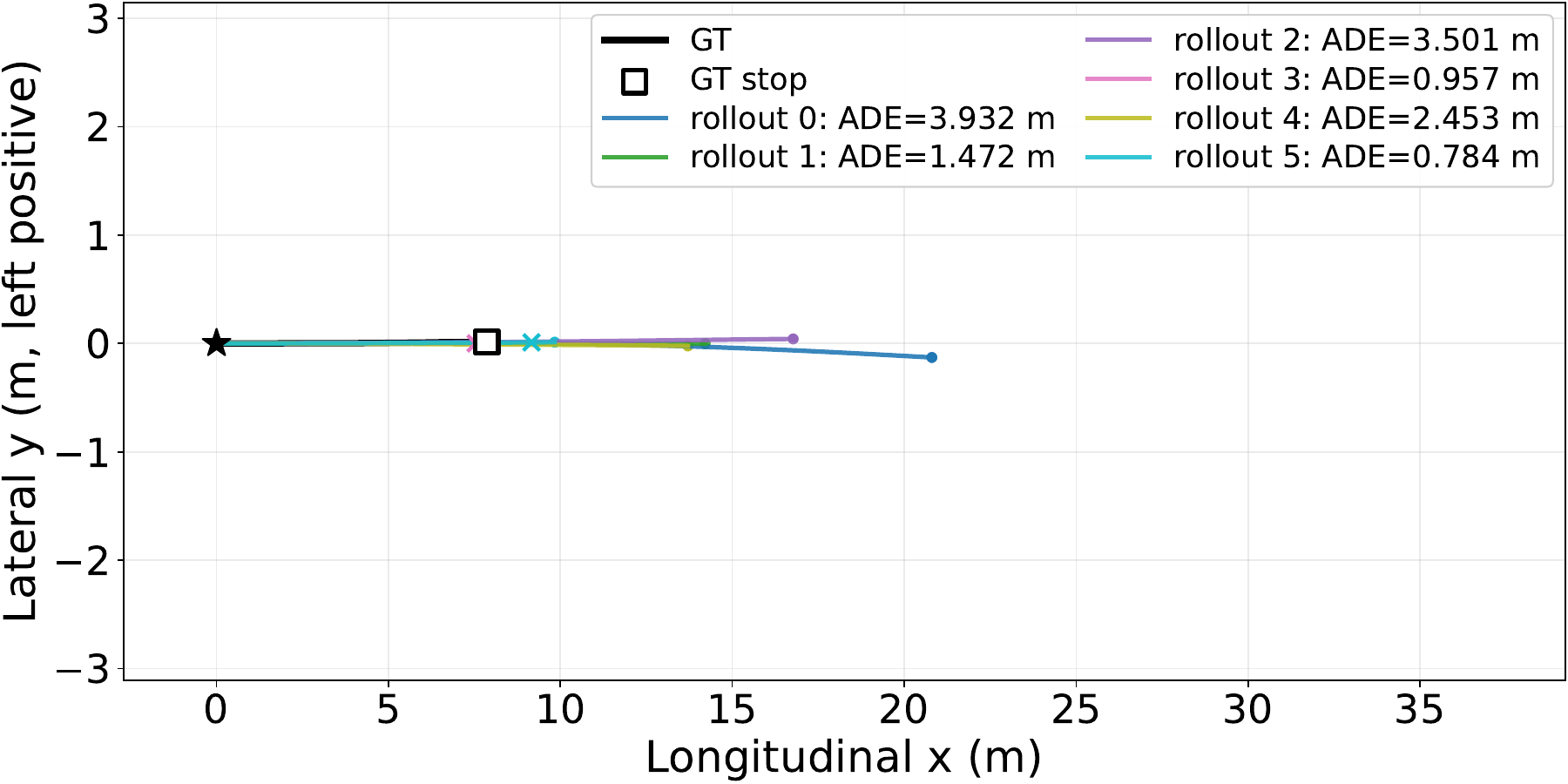}}
    \caption{The trajectories predicted by different methods in the example of Fig.~\ref{fig:general_example_now} and Fig.~\ref{fig:general_example_past}. It is clearly shown that our method outperforms baselines.}
    \label{fig:general_example_traj}
\end{figure}

\subsection{Scaling Up Model Size}

In order to see how our model works with models of larger scale, we further conduct SFT on another 8B checkpoint of Alpamayo 2 on our all-way stop dataset (Sec.~\ref{sec:awsdata}). Tab.~\ref{tab:8b_aws} shows the result, where 8B models generally outperform 2B models; our finetuned model performs the best among 8B models, and also far outperforms Alpamayo 2 Super 34B.

\begin{table}[t]
    \centering
    \scriptsize
    \caption{Results of 8B models on All-Way-Stop v4. The result clearly shows that our 8B model works the best.}
    \begin{tabular}{c@{\hspace{4pt}}ccccccccc}
    \toprule
      & minADE\rlap{$\,\downarrow$} & Avg. ADE\rlap{$\,\downarrow$} & ML. ADE\rlap{$\,\downarrow$} & Stop SR\rlap{$\,\uparrow$} & Go SR\rlap{$\,\uparrow$}
      & $\Delta_{\text{pos}}$\rlap{$\,\downarrow$} & $\Delta_{\text{dur}}$\rlap{$\,\downarrow$}
      & Roll\rlap{$\,\downarrow$} & MCQ Acc.\rlap{$\,\uparrow$}\\
    \midrule
        Alpamayo 2 Super 34B & 0.993 & 2.293 & N/A & 78.75\% & 34.30\% & 2.182 & 1.488 & 15.94\% & 33.19\% \\
        \midrule
        Base model (2B) & 1.386 & 2.351 & 2.393 & 82.06\% & 33.74\% & 1.725 & 1.871 & 11.79\% & 0\% \\
        
        No memory (2B SFT) & 1.102 & 2.185 & 2.318 & 87.04\% & 38.59\% & 1.358 & 1.564 & 6.62\% & 33.77\% \\
        CoT-as-memory (2B SFT) & 1.096 & 2.171 & 2.296 & 86.89\% & 39.37\% & 1.377 & 1.531 & 6.69\% & 45.41\% \\
        Ours (2B SFT) &  \textbf{1.049} & \textbf{2.121} & \textbf{2.256} & \textbf{87.48\%} & \textbf{40.66\%} & \textbf{1.257} & \textbf{1.469} & \textbf{6.05\%} & \textbf{45.99\%} \\
        \midrule
  Base model (8B)
  & 1.315 & 2.254 & 2.332
  & 83.72\% & 35.65\%
  & 1.627 & 1.776 & 10.90\% & 0\% \\
No memory (8B SFT)
  & 1.011 & 2.092 & 2.255
  & 88.11\% & 42.05\%
  & 1.252 & 1.415 & 6.04\% & 39.31\% \\
CoT-as-memory (8B SFT)
  & 1.000 & 2.025 & 2.130
  & 88.18\% & 42.48\%
  & 1.230 & 1.374 & 5.92\% & 44.92\% \\
Ours (8B SFT)
  & \textbf{0.980} & \textbf{1.986} & \textbf{2.071}
  & \textbf{88.99\%}& \textbf{43.24\%}
  & \textbf{1.130} & \textbf{1.352} & \textbf{5.40\%} & \textbf{52.46\%} \\
    \bottomrule
    \end{tabular}
    
    \label{tab:8b_aws}
\end{table}

\subsection{Overfitting Check on VQA of the General Driving Dataset}
\label{sec:overfit_vqa}
The biggest challenge of curating VQA questions on the general driving dataset is that models will easily overfit to questions from the same distribution by memorizing patterns in the dataset without understanding, sometimes even without assigning any attention to the visual input. Many prior autonomous driving benchmarks are haunted by such issue. For example, models on TUMTraffic-VideoQA~\citep{zhou2025tumtraffic} can achieve 71\% accuracy with no visual input, while only achieving $81\%$ accuracy with 11 frames; NuScenes-QA~\citep{qian2024nuscenes} reports an accuracy of 53.4\% by only looking at the question alone; on DriveBench~\citep{xie2025drivebench}, models with no image input can even achieve 95\% GPT-score performance compared to normal input. 


Therefore, to check whether our model overfits on the curated questions after training, and also as a sanity check to the correctness of our curated questions, we test the performance of GPT-5.6 Luna on our ``easy'' and ``hard'' split of MCQ questions, as well as the performance of our checkpoint when trained on the respective split of these questions. The result is illustrated in Tab.~\ref{tab:overfit-vqa}, which gives us the following two insights:

\begin{itemize}
    \item VQAs training on the same distribution will very likely overfit on the test set, even surpassing GPT-5.6 Luna performance with full video. We also tried to use the ``hard'' split as the training set, but the accuracy on test data generated from the same distribution will still reach over 96\%.
    
    \item Our designed memory indeed achieves the best result on the harder test set with different distributions, and can do even better with ground truth memory (as opposed to CoT-as-memory where ground truth memory does not help).
 
\end{itemize}

\begin{table}[t]
    \centering
    \small

    \caption{The MCQ accuracy on our ``easy'' VQA questions (i.e. the one in the training set) and ``hard'' VQA questions (i.e. the one in the test set) for the general driving dataset. GPT-5.6 Luna with full video serves as a sanity check and verifier for the correctness of our generated VQA questions. Avg. test set is the average of hard \#1 and hard \#2; the example for each type is listed in Sec.~\ref{sec:datastats-gen}. It is clearly shown that test on the same distribution of VQA questions as training causes overfitting. Thus, we deliberately use different questions for the training and test set.}
    
    \begin{tabular}{ccccc}
    \toprule
         & Easy & Hard (\#1, past-anchor) & Hard (\#2, multi-hop) & Avg. test set \\
    \midrule
    Random & 16.67\% & 16.67\% & 16.67\% & 16.67\% \\
      Alpamayo 2 Super 34B & 48.45\% & 25.77\% & 25.65\% & 25.71\% \\
       No memory & \textbf{98.14\%} & 54.93\% & 61.10\% & 58.02\% \\
       CoT-as-Memory & 98.02\% & 51.46\% & 63.56\% & 57.51\% \\
       Ours  &  97.92\% & \textbf{60.92\%} & \textbf{70.34\%} & \textbf{65.63\%}\\
       \midrule
       CoT-as-Memory (GT mem) & 97.99\% & 51.25\% & 63.48\% & 57.36\% \\
       Ours (GT mem)  & 98.22\% & 63.00\% & 72.29\% & 67.65\%\\
    \midrule
    GPT-5.6 Luna + full video & 94.34\% & 89.14\% & 91.57\% & 90.35\% \\
    \bottomrule
    \end{tabular}
    
    \label{tab:overfit-vqa}
\end{table}

\subsection{Which Scene Does AD-Memo Benefit the Most?}

In this section, we will investigate our test set and to locate the scenes where AD-Memo benefits the most. Following prior work~\citep{tan2025latent}, we ask GPT-5.6 Luna to classify the clips in our test set into 15 categories, including turning maneuver, speed control, merging, lane change, cut-in, nudge maneuver, nudge static obstacle maneuver, Vulnerable Road Users (VRU), intersection navigation, lane keeping curve, lead vehicle following, traffic control compliance, lane keeping, stop for vehicle, and others. Fig.~\ref{fig:piechart} shows the ratio of each type of data, and Tab.~\ref{tab:perscene} shows the performance difference between our method and the baseline without memory. The result shows that our method improves more on the performance of average and the most likely rollout, and works better on the more common challenging scenes, especially traffic control compliance, lead vehicle following, and stop for vehicles.

\begin{figure}[t]
    \centering
    \includegraphics[width=0.7\linewidth]{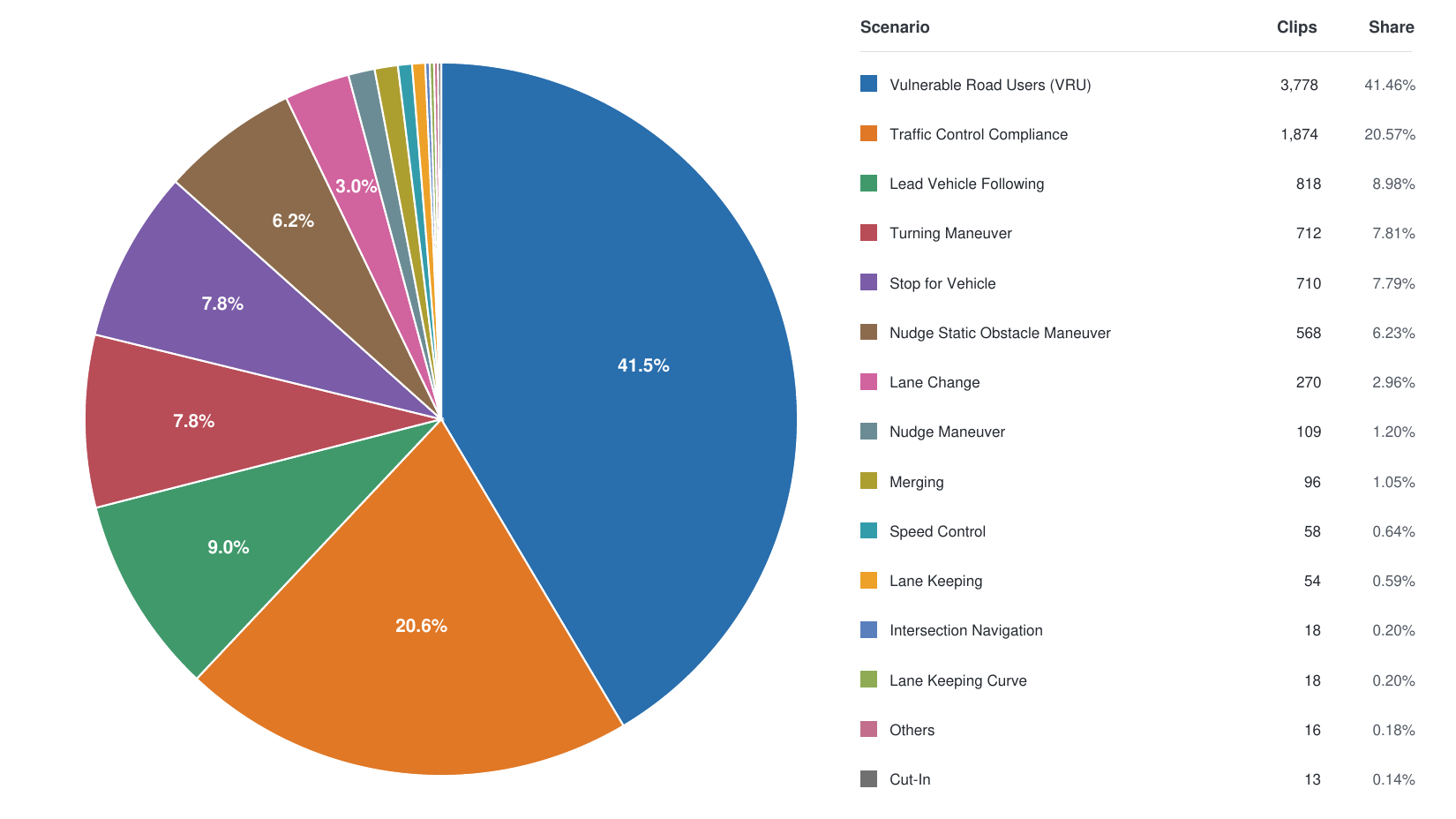}
    \caption{The distribution of different scenes in the test set of our general driving dataset.}
    \label{fig:piechart}
\end{figure}

\begin{table}[t]
    \centering
    \scriptsize

\caption{Per-scenario open-loop trajectory error gains of our model
    over the no-memory model (after RL); as it is the \textit{gain}, the higher the better, i.e., \textcolor{green!50!black}{positive
    values indicate improvement}, and \textcolor{red}{vice versa.} The categories are sorted in descending order of their shares. All values are in meters. It is shown that AD-Memo outperforms the no-memory model on the more common scenes, especially traffic control compliance, lead vehicle following and stop for vehicles.}
    
 \begin{tabular}{cccc}
    \toprule
         & ADE gain & FDE gain & Shares\\
         & (min/Avg./ML) & (min/Avg./ML) & (\%)\\
    \midrule
    Vulnerable Road Users
    & \textcolor{red}{-0.011}/+0.010/\textcolor{green!50!black}{+0.018}
    & \textcolor{red}{-0.011}/\textcolor{green!50!black}{+0.054}/\textcolor{green!50!black}{+0.072} & 41.46\% \\

    Traffic Control Compliance
    & \textcolor{green!50!black}{+0.013}/\textcolor{green!50!black}{+0.023}/\textcolor{green!50!black}{+0.030} 
    & \textcolor{green!50!black}{+0.089}/\textcolor{green!50!black}{+0.101}/\textcolor{green!50!black}{+0.105} & 20.57\% \\

    Lead Vehicle Following
    & +0.003/\textcolor{green!50!black}{+0.017}/\textcolor{green!50!black}{+0.012}
    & \textcolor{green!50!black}{+0.021}/\textcolor{green!50!black}{+0.058}/\textcolor{green!50!black}{+0.031} & 8.98\% \\

    Turning Maneuver
    & -0.002/\textcolor{green!50!black}{+0.024}/\textcolor{green!50!black}{+0.031}
    & \textcolor{red}{-0.019}/\textcolor{green!50!black}{+0.061}/\textcolor{green!50!black}{+0.092} & 7.81\% \\

    Stop for Vehicles
    & +0.008/\textcolor{green!50!black}{+0.011}/\textcolor{green!50!black}{+0.017}
    & \textcolor{green!50!black}{+0.057}/\textcolor{green!50!black}{+0.072}/\textcolor{green!50!black}{+0.067} & 7.79\% \\

    Nudge Static Obstacle Maneuver
    & -0.006/\textcolor{green!50!black}{+0.030}/-0.001
    & \textcolor{red}{-0.040}/\textcolor{green!50!black}{+0.063}/\textcolor{red}{-0.035} & 6.23\% \\

    Lane Changes
    & \textcolor{green!50!black}{+0.020}/\textcolor{green!50!black}{+0.060}/\textcolor{green!50!black}{+0.060}
    & \textcolor{green!50!black}{+0.021}/\textcolor{green!50!black}{+0.143}/\textcolor{green!50!black}{+0.124} & 2.96\% \\

    Nudge Maneuver
    & +0.002/\textcolor{green!50!black}{+0.042}/\textcolor{green!50!black}{+0.067}
    & \textcolor{green!50!black}{+0.051}/\textcolor{green!50!black}{+0.081}/\textcolor{green!50!black}{+0.154} & 1.20\% \\

    Merging
    & -0.010/\textcolor{green!50!black}{+0.031}/\textcolor{red}{-0.017}
    & \textcolor{red}{-0.013}/\textcolor{green!50!black}{+0.092}/\textcolor{red}{-0.068} & 1.05\% \\

    Speed Control
    & \textcolor{red}{-0.043}/\textcolor{green!50!black}{+0.033}/\textcolor{green!50!black}{+0.030}
    & \textcolor{red}{-0.089}/\textcolor{green!50!black}{+0.159}/\textcolor{green!50!black}{+0.136} &  0.64\% \\

    Lane Keeping
    & \textcolor{red}{-0.036}/\textcolor{red}{-0.031}/\textcolor{red}{-0.012}
    & \textcolor{red}{-0.155}/\textcolor{red}{-0.132}/\textcolor{red}{-0.049} & 0.59\%\\

    Intersection Navigation
    & \textcolor{green!50!black}{+0.022}/\textcolor{green!50!black}{+0.051}/\textcolor{red}{-0.128}
    & \textcolor{red}{-0.238}/\textcolor{red}{-0.035}/\textcolor{red}{-0.503} & 0.20\%\\

    Lane Keeping Curve
    & \textcolor{red}{-0.089}/\textcolor{red}{-0.061}/\textcolor{green!50!black}{+0.024}
    & \textcolor{red}{-0.455}/\textcolor{red}{-0.258}/\textcolor{red}{-0.045} & 0.20\%\\

    Others
    & \textcolor{red}{-0.270}/\textcolor{red}{-0.337}/\textcolor{red}{-0.332}
    & \textcolor{red}{-1.041}/\textcolor{red}{-1.220}/\textcolor{red}{-1.398} & 0.18\%\\

    Cut-in
    & \textcolor{red}{-0.128}/\textcolor{red}{-0.068}/+0.003
    & \textcolor{red}{-0.396}/\textcolor{red}{-0.211}/\textcolor{red}{-0.190} & 0.14\%\\
    \bottomrule
    \end{tabular}

    \label{tab:perscene}
\end{table}

\subsection{The Necessity of Decision Graph}
\label{sec:ablation_dg}

As we mentioned in Sec.~\ref{sec:generaldata}, the purpose of building decision graphs is to identify the objects on the road that affects driving. In this section, we illustrate the necessity of decision graph by providing an example in Fig.~\ref{fig:dg_ablation}, where we compare the memory directly generated by simply feeding the whole video into GPT-5.6 Luna and prompting the model to write memory. The generated memories are as follows:

{
\setlength{\fboxsep}{0.3cm}
\ovalbox{\small
\begin{minipage}{0.95\linewidth}
\texttt{Direct generation:} The pedestrian is at or very near the right curb/sidewalk and is nearly clear
   of the crosswalk; ego continues straight.

\texttt{Decision graph:} The ego vehicle remains stopped while the pedestrian with the stroller is at
   the right side of the rainbow crosswalk, with a yield sign on the right and
   large vehicles, the pickup, and box truck obstructing the left approach.

\end{minipage}
}
}

The result clearly shows that without decision graph, the generated memory will lose track of many vital items in the image, such as the stroller, crosswalk, yield sign and obstruction on the left.

\begin{figure}[t]
    \centering
    \subfigure[Visual input]{
    \includegraphics[width=0.32\linewidth]{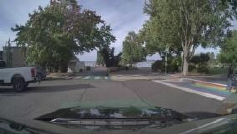}}
    \subfigure[Object detection]{
    \includegraphics[width=0.32\linewidth]{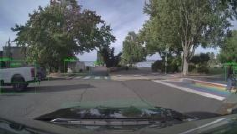}}
    \subfigure[Decision graph]{
    \includegraphics[width=0.32\linewidth]{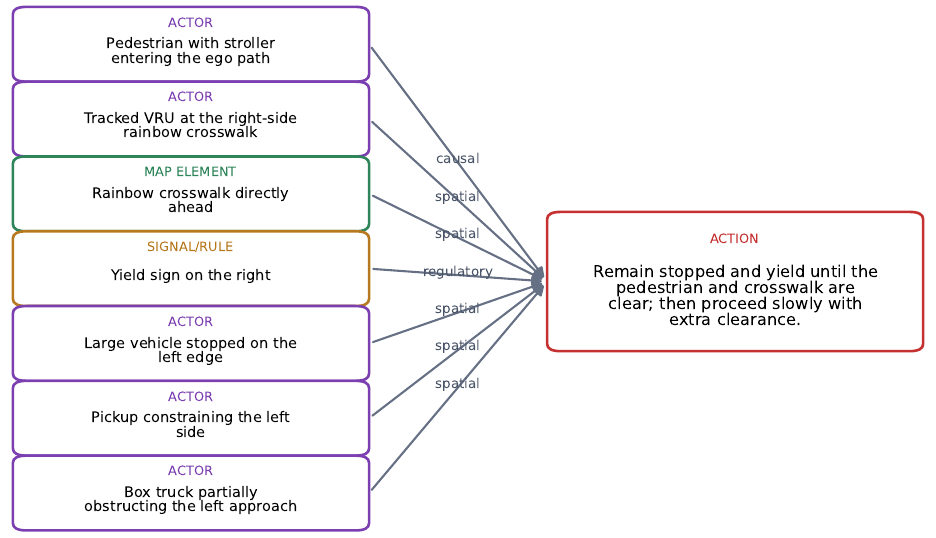}}
    \caption{An example of how our pipeline based on decision graph identify all objects related to driving. All objects that affect driving decision are grounded in panel (b) from the input frame shown in panel (a); these objects end up as nodes in the corresponding decision graph in panel (c).}
    \label{fig:dg_ablation}
\end{figure}

\subsection{Decode Success Rate}

As the output of our model contains CoT, memory and trajectory tokens, it is possible that the output of the model does not follow the corresponding format, thus leading to failures of decoding output. In Tab.~\ref{tab:decode}, we report the decode rate of our model and baselines, which shows that our model maintains a high decode success rate. Notably, during evaluation, if we fail to decode memory from one step, we will set the corresponding slot of memory to null moving forward.

\begin{table}[t]
    \centering
    \small
     \caption{The decode success rate for our model and baselines, which are all reasonably close to 100\%.}
    \begin{tabular}{lc}
    \toprule
    Model & Decode success rate \\
    \midrule
    \multicolumn{2}{c}{All-way stop dataset} \\
    \midrule
    Base model                & 100.00\% \\
    Alpamayo 2 Super 34B      & 100.00\% \\
    No memory (SFT)           & 99.44\% \\
    CoT-as-Memory (SFT)       & 99.41\% \\
    Ours (SFT)                & 99.51\% \\
    No memory (RL)            & 99.85\% \\
    CoT-as-Memory (RL)        & 99.87\% \\
    Ours (RL)                 & 99.86\% \\
    \midrule
    \multicolumn{2}{c}{General driving dataset} \\
    \midrule
    Base model                & 99.87\% \\
    Alpamayo 2 Super 34B      & 100.00\%        \\
    No memory (SFT)           & 99.78\% \\
    CoT-as-Memory (SFT)       & 99.69\% \\
    Ours (SFT)                & 99.62\% \\
    No memory (RL)            & 99.91\%  \\
    CoT-as-Memory (RL)        & 99.85\%   \\
    Ours (RL)                 & 99.83\%   \\
    \bottomrule
    \end{tabular}
   
    \label{tab:decode}
\end{table}

\subsection{Semi-Closed Loop RL Training Curve}
\label{sec:curve}
Fig.~\ref{fig:rl_aws} illustrates the training dynamics of our semi-closed loop RL on the all-way stop dataset, while Fig.~\ref{fig:rl_gen} illustrates the training dynamics of our semi-closed loop RL on the general driving dataset.

\begin{figure}[t]
    \centering
    \subfigure[Mean reward]{
    \includegraphics[width=0.4\linewidth]{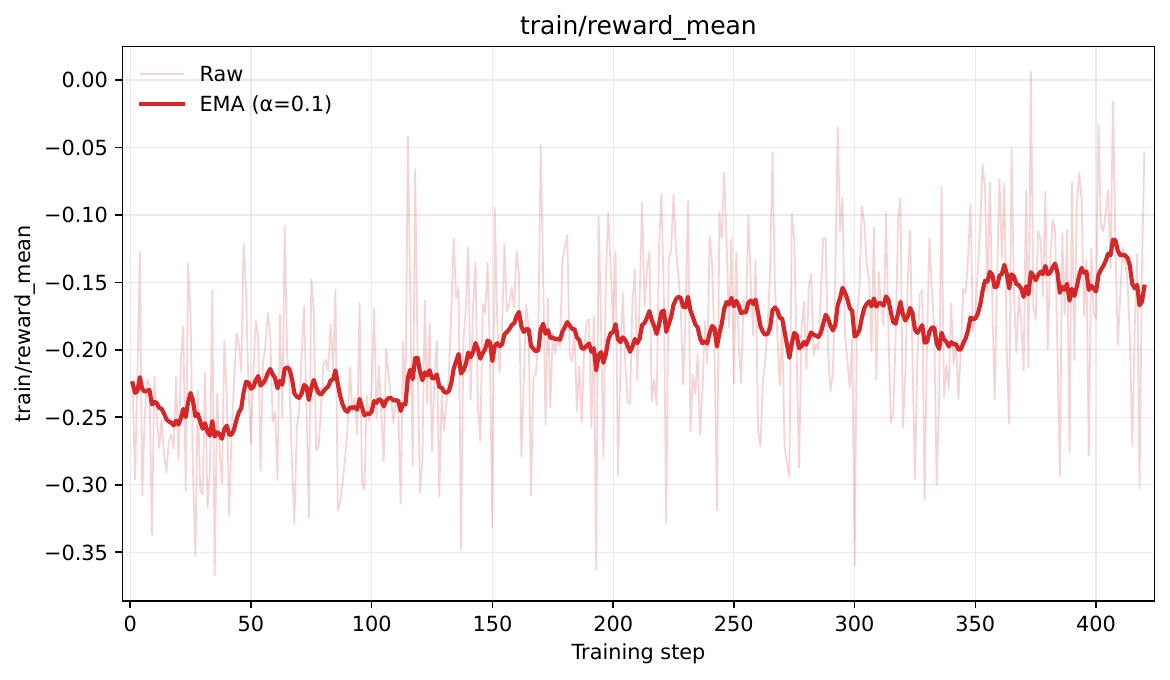}}
    \subfigure[L2 distance]{
    \includegraphics[width=0.4\linewidth]{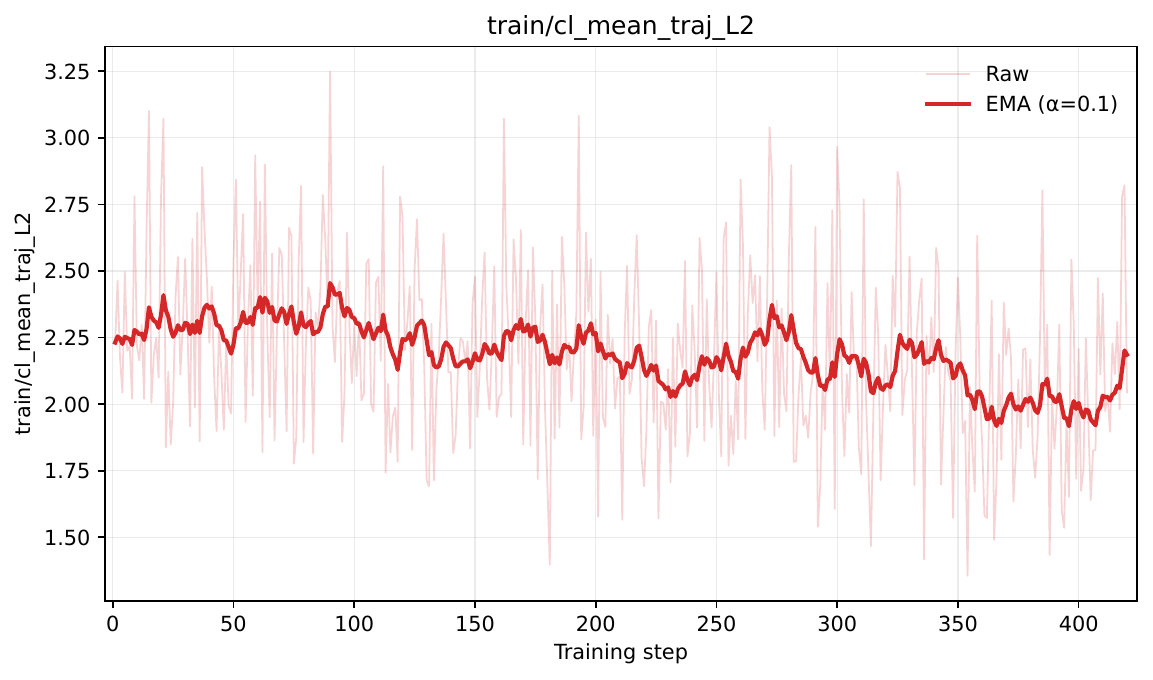}}
    \subfigure[VQA accuracy]{
    \includegraphics[width=0.4\linewidth]{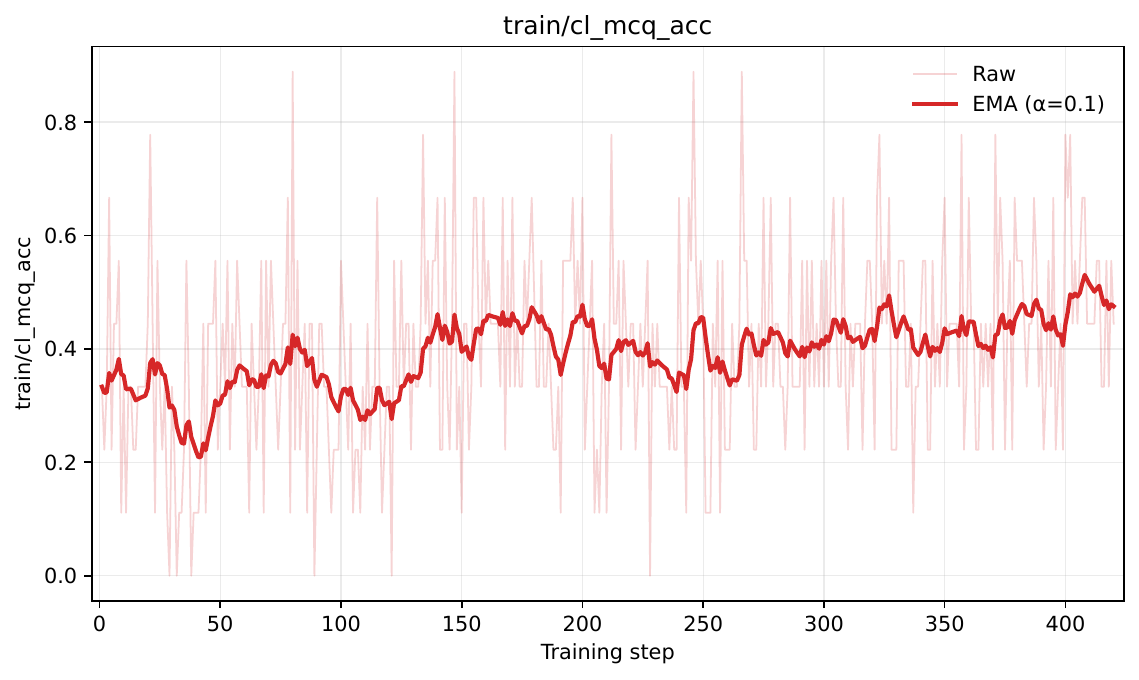}}
     \subfigure[Decode rate]{
    \includegraphics[width=0.4\linewidth]{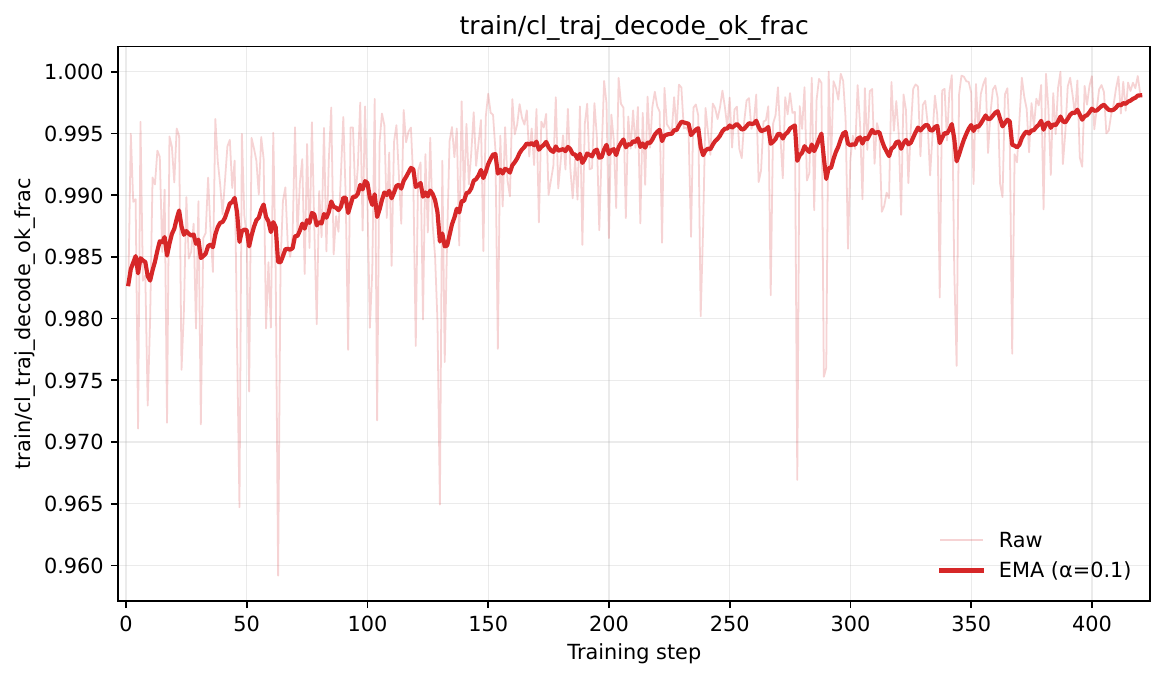}}
    \caption{The training dynamics of our semi-closed loop RL on the all-way stop dataset. We report the original curve and Exponential Moving Average (EMA).}
    \label{fig:rl_aws}
\end{figure}

\begin{figure}[t]
    \centering
    \subfigure[Mean reward]{
    \includegraphics[width=0.4\linewidth]{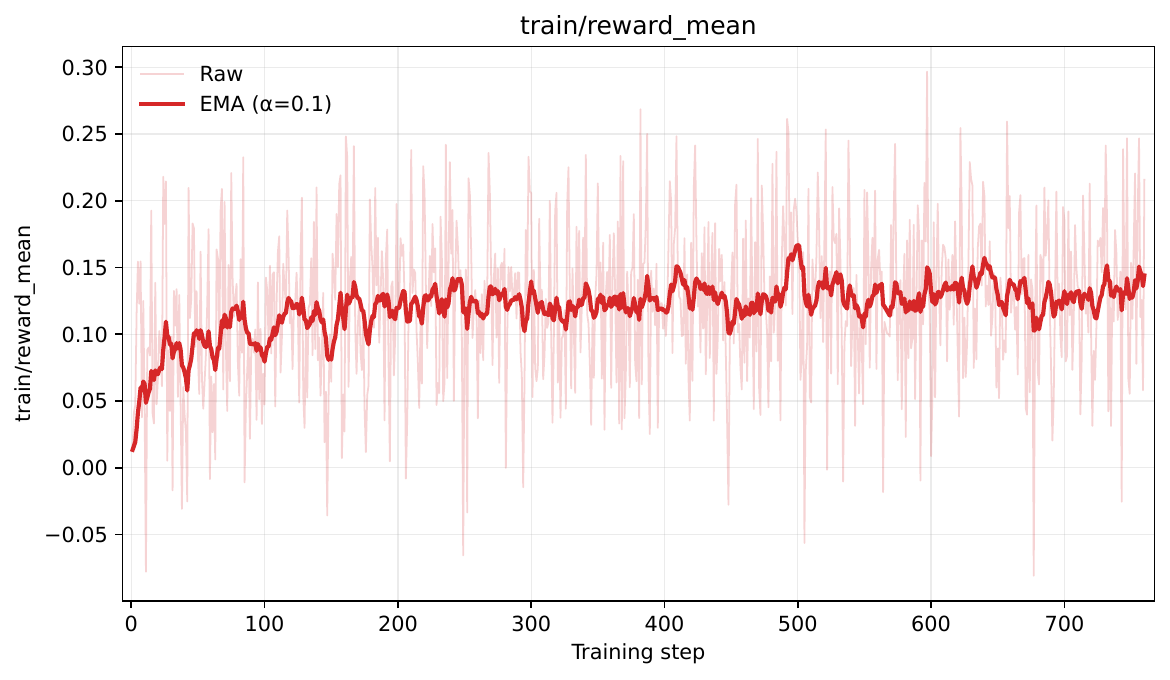}}
    \subfigure[L2 distance]{
    \includegraphics[width=0.4\linewidth]{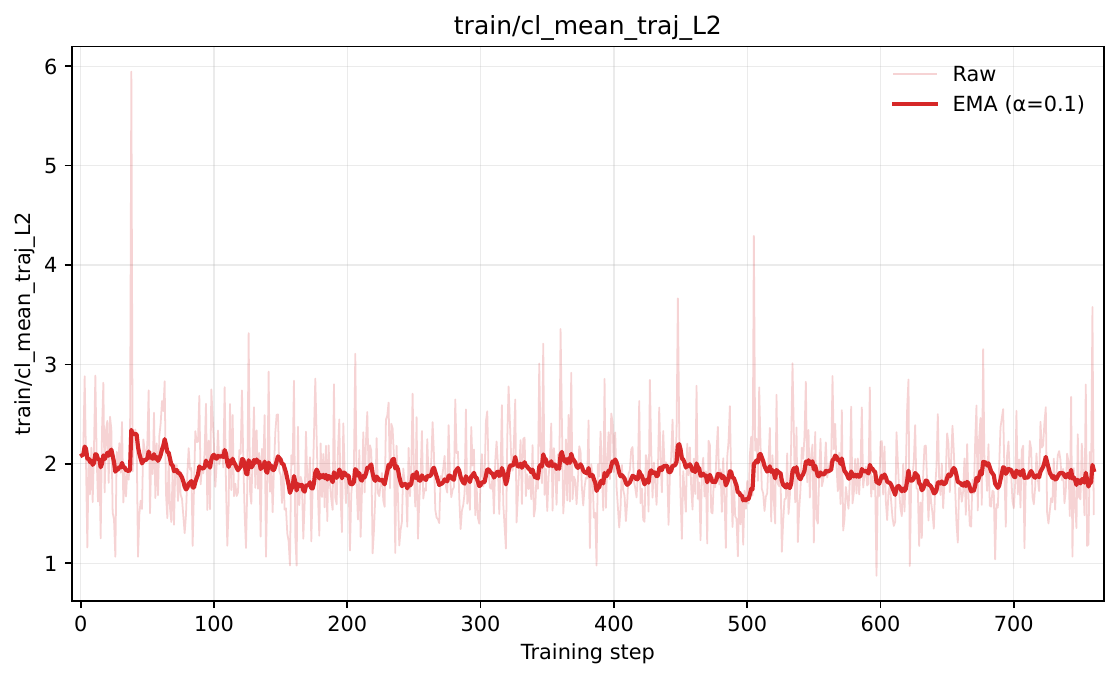}}
    \subfigure[VQA accuracy]{
    \includegraphics[width=0.4\linewidth]{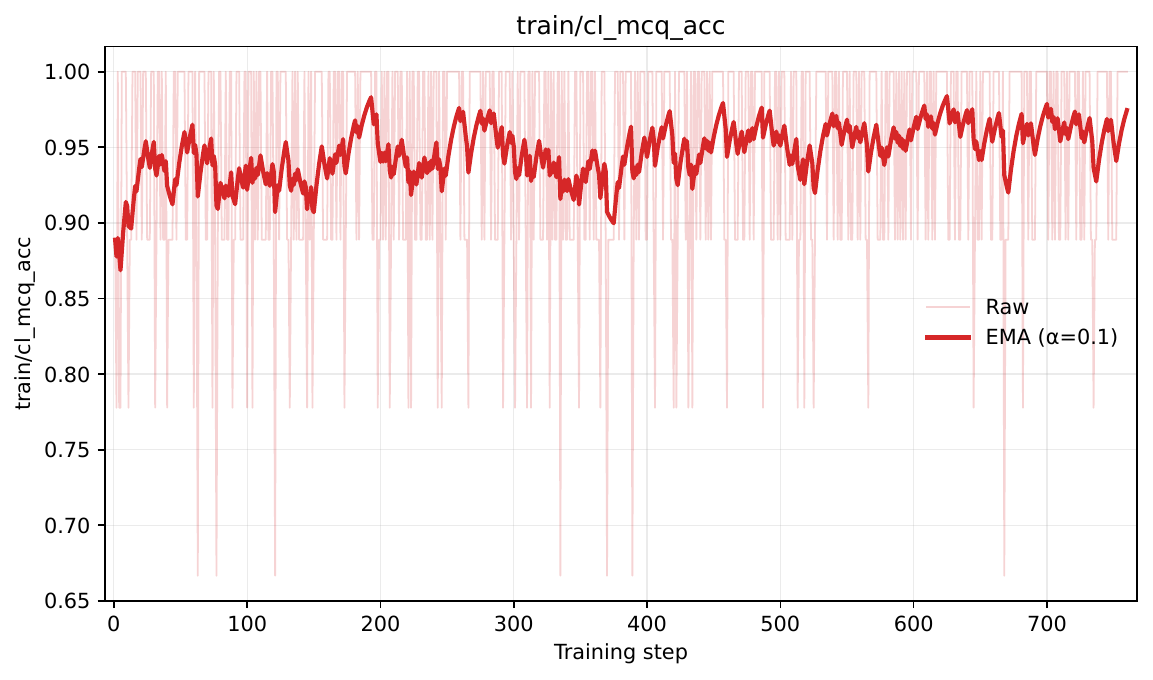}}
     \subfigure[Decode rate]{
    \includegraphics[width=0.4\linewidth]{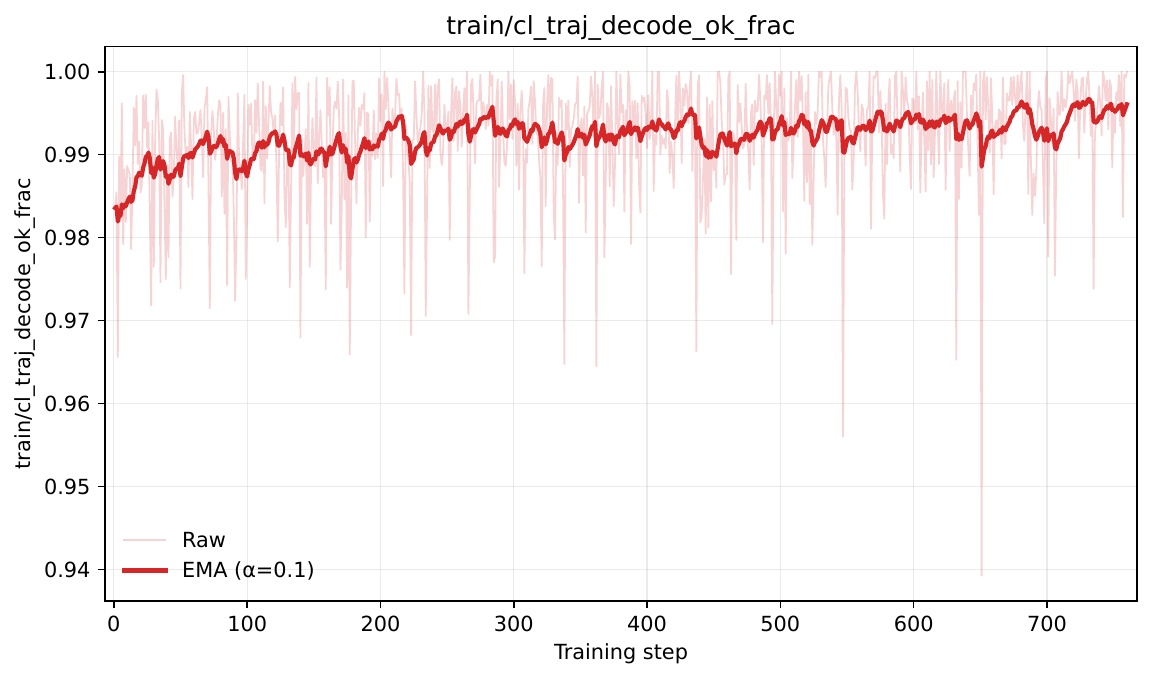}}
    \caption{The training dynamics of our semi-closed loop RL on the general driving dataset. We report the original curve and Exponential Moving Average (EMA). The VQA accuracy is close to 1 as the problems are relatively easy to overfit, and thus we test different problems in the test set; see Sec.~\ref{sec:overfit_vqa} for details.}
    \label{fig:rl_gen}
\end{figure}

\subsection{Ablation on Memory Horizon \texorpdfstring{$T$}{T}}
Fig.~\ref{fig:mem-horizon-aws} shows the result for using different memory horizon with AD-Memo and CoT-as-memory on all-way stop dataset during inference, and Fig.~\ref{fig:mem-horizon-gen} shows the result on the general driving dataset. Several conclusions can be drawn from the results:

\begin{itemize}
    \item Our memory is consistently more effective than CoT-as-memory throughout different memory horizons on both datasets.
    \item The MCQ accuracy is highly correlated with the memory horizon. On the all-way stop dataset, a memory of 8 seconds (20 frames) is sufficient; on the general driving dataset, a memory of 4 seconds (10 frames) is generally fine. The performance of AD-Memo scales largely with memory horizon, but CoT-as-memory does not, which shows the effectiveness of our memory.
    \item minADE, minFDE and corner distance are the metrics that benefits the most from longer memory horizon. We hypothesize that it because a more diverse set of memory induces more diverse driving behavior, leading to a better coverage (which also explains why avg. and most likely metrics are not scaling with respect to memory horizon).
    \item The trend of task-specific metrics on the all-way stop task, such as success stop rate, success go rate and roll-through, does not necessarily correlate with ADE, which shows the necessity of adding these metrics.
\end{itemize}

\begin{figure}[t]
    \centering
    \includegraphics[width=\linewidth]{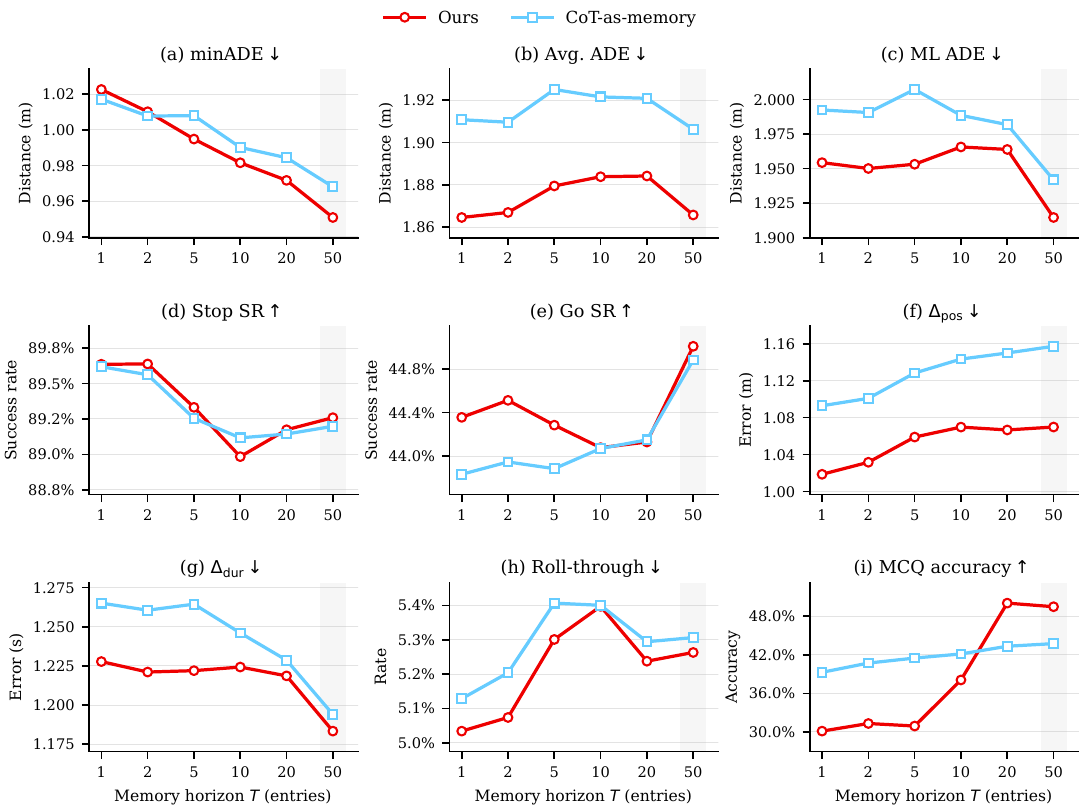}
    \caption{The change of performance with respect to memory horizon during inference time on the all-way stop dataset.}
    \label{fig:mem-horizon-aws}
\end{figure}

\begin{figure}[t]
    \centering
    \includegraphics[width=\linewidth]{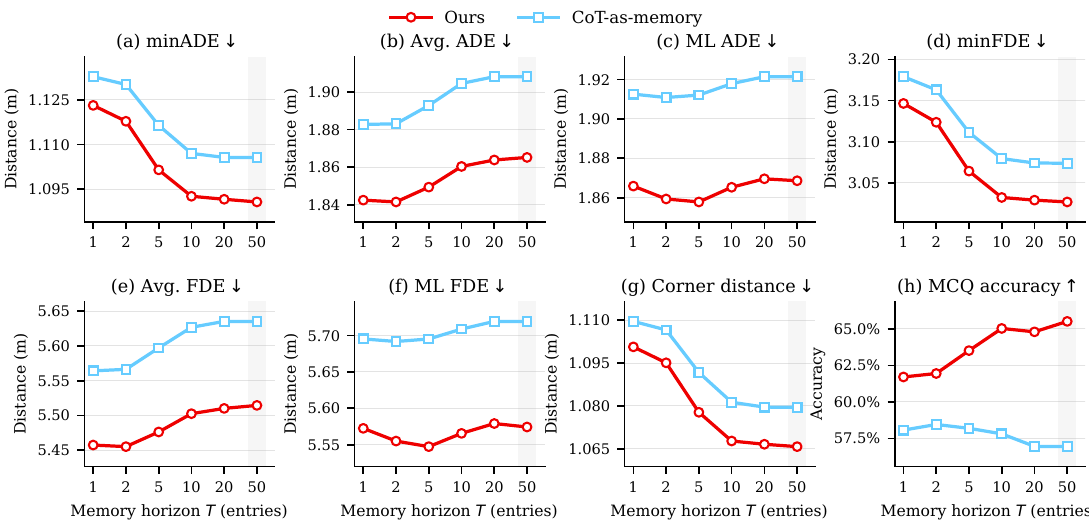}
    \caption{The change of performance with respect to memory horizon during inference time on the general driving dataset.}
    \label{fig:mem-horizon-gen}
\end{figure}

\subsection{Ablation on \texorpdfstring{$\lambda^{\text{VQA}}$}{lambda VQA}}
\label{sec:ablation_lambda}
Tab.~\ref{tab:ablation_lambda} shows the ablation result of using different $\lambda^{\text{VQA}}$ on our test set. The result shows that our method is generally robust to the choice of $\lambda$, with the optimal $\lambda^{\text{VQA}}$ for each metric spanning over the table. Notably, larger $\lambda^{\text{VQA}}$ does slightly better than smaller $\lambda^{\text{VQA}}$. Such performance difference comes from the fact that the cumulative future driving reward is a highly noisy signal. Thus, even if smaller $\lambda^{\text{VQA}}$ encourages the model to focus on driving, its performance will not necessarily improve. Such noisiness is also one of the motivation for \textit{Da Capo} that gives better credit assignment and reduces variance.

\begin{table}[t]
    \centering
    \scriptsize
    \caption{The result of our method on the all-way stop dataset with different $\lambda^{\text{VQA}}$. The performance is generally consistent, and no hyperparameter choice prevail on all metrics.}
    \begin{tabular}{cccccccccc}
    \toprule
       $\lambda^{\text{VQA}}$ & minADE\rlap{$\,\downarrow$} & Avg. ADE\rlap{$\,\downarrow$} & ML. ADE\rlap{$\,\downarrow$} & Stop SR\rlap{$\,\uparrow$} & Go SR\rlap{$\,\uparrow$} & $\Delta_\text{pos}$\rlap{$\,\downarrow$} & $\Delta_{\text{dur}}$\rlap{$\,\downarrow$} & Roll\rlap{$\,\downarrow$} & MCQ Acc.\rlap{$\,\uparrow$} \\
    \midrule
       0.02 & 0.945 & 1.886 & 1.927 & 88.73\% & 44.80\%
            & 1.111 & 1.184 & 5.54\% & 50.43\% \\
       0.05 & 0.943 & 1.877 & 1.913 & 88.67\% & 45.01\%
            & 1.092 & 1.178 & 5.38\% & 50.65\% \\
       0.1  & \textbf{0.932} & 1.886 & 1.898 & 88.71\% & 45.01\%
            & 1.129 & 1.179 & 5.28\% & 51.61\% \\
       0.2  & 0.960 & 1.876 & 1.931 & 89.03\% & 44.94\%
            & 1.099 & 1.194 & 5.34\% & 51.12\% \\
       0.5  & 0.951 & \textbf{1.866} & 1.915 & \textbf{89.26\%} & 45.01\%
            & \textbf{1.070} & 1.183 & 5.26\% & \textbf{51.66\%} \\
       1   & 0.941 & 1.900 & \textbf{1.891} & 88.52\% & 45.04\%
    & 1.119 & \textbf{1.174} & 5.44\% & 50.38\% \\
       2.5 & 0.942 & 1.886 & 1.909 & 88.10\% & \textbf{45.33\%}
    & 1.130 & 1.175 & 5.74\% & 50.80\% \\
       5    & 0.948 & 1.874 & 1.908 & 88.87\% & 45.07\%
            & \textbf{1.070} & 1.175 & \textbf{5.25\%} & 50.47\% \\
    \bottomrule
    \end{tabular}
    \label{tab:ablation_lambda}
\end{table}

\subsection{Ablation on SFT Training Length}
\label{sec:longer}

To test the optimal training length of our model on SFT, we train our 2B checkpoint for up to $4$ epochs on an earlier split of our general driving SFT dataset, and test its performance on a separate 2000-clip validation set on that split. The result is illustrated in Fig.~\ref{fig:length-ablation}, which shows that training for 2 epochs is optimal.

\begin{figure}[t]
    \centering
    \subfigure[minADE]{
    \includegraphics[width=0.3\linewidth]{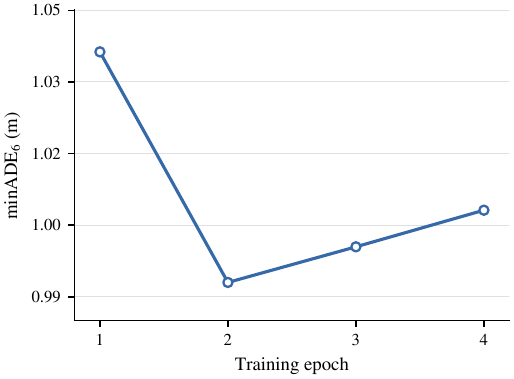}}
    \subfigure[Most likely ADE]{
    \includegraphics[width=0.3\linewidth]{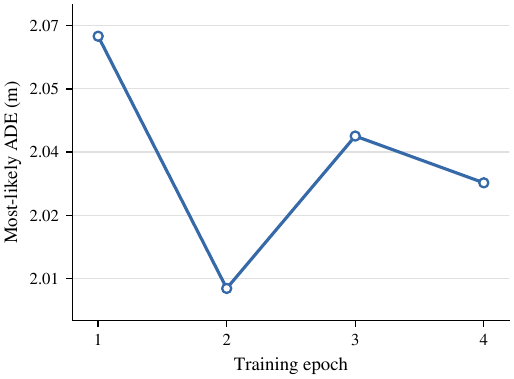}}
    \subfigure[Decode rate]{
    \includegraphics[width=0.3\linewidth]{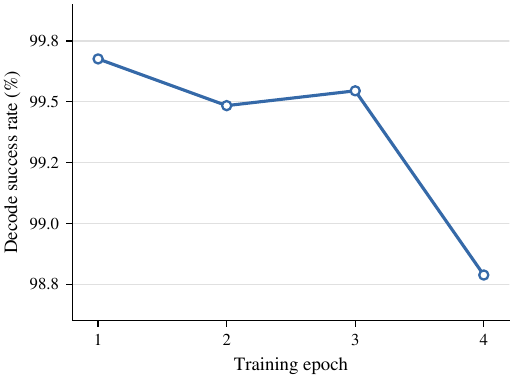}}
    \caption{The performance change with the training length of SFT on the validation set. It is clearly shown that, after 2 epochs, not only does the trajectory move further from ground truth, but the format following ability of the model is also degrading due to exposure bias~\citep{Bengio2015ScheduledSampling}, i.e., more invalid rollouts appear in sampling.}
    \label{fig:length-ablation}
\end{figure}

\section{Computational Resources}
\label{sec:resource}

Our computational resource consumption is as follows:

\textbf{Data curation:} Generating a decision graph takes on average 10.43 model calls, which sums to roughly 15K input tokens and 5K output tokens.

\textbf{SFT experiments:} All SFT experiments are conducted on a single compute node with 8 NVIDIA H100 GPUs (80GB memory), 96 AMD EPYC 9R14 CPUs, and 512GB memory. A 8900-step SFT training on the general driving dataset takes approximately 16 hours.

\textbf{RL experiments:} All RL experiments are conducted with four compute nodes (two for rollout and two for training), each of which is identical to a SFT node. RL experiments typically require 2-3 days to complete.

\textbf{Evaluation:} All evaluations are conducted with 16 single-GPU node with each node identical to the SFT node (except for the number of GPU). Evaluation usually takes 4-6 hours. 

\section{Limitation and Future Works}
\label{sec:limit}

\textbf{Horizon of memories.} Our work mainly studies memories that lasts around 10 seconds. While we empirically verified this to be sufficient for our task of interest (Sec.~\ref{sec:datastats-aws}), long-term regulations such as speed signs may exert influence on driving for minutes to hours. While AD-Memo can straightforwardly address such problem by adding a special buffer for ``long-term memory'' heuristically, an in-depth investigation into such area is still interesting.

\textbf{Extending our finding to action expert.} Our experiment is mainly done with a VLA backbone, which means the action expert in Alpamayo is not involved in this work. This is because in the Alpamayo training pipeline~\citep{nvidia2026alpamayo2super}, the action expert requires extensive training after the VLA backbone is finalized, which is too expensive to run. To see whether the advantage of language memory persists with action expert is an interesting avenue to explore in the near future.

\textbf{Comparison with vector-based baselines.} We choose the same training recipe without memory and with CoT-as-memory as our baseline because direct comparison to vector-based baselines is difficult. Most related papers have no training code or checkpoint released, such as MindVLA-U1~\citep{huang2026mindvla}, enhanced HOPE~\citep{kim2026think}, VLN-AVP~\citep{li2026vln} and AdaDrive~\citep{zhang2025adadrive}. 

For those with code available, COMPACT-VA~\citep{liang2026planning} considers the closest scene to us, and we adopt some of its metrics. However, the memory design is based on a different VLA architecture with the assumption of $20$ or more frames, which is not supported by our Alpamayo base model (pretrained to only receive $4$ frames). Also, it requires a fixed number of frame input with no padding (e.g. $20$), for which many clips in our test set cannot be tested for their limited length (and it is also unfair as we report the average of many frames in a clip, including early ones in the video that must be padded for COMPACT-VA).

For the rest of the baselines,  MoMAD~\citep{song2025don} uses ResNet~\citep{he2016deep} instead of VLA as backbone; to adapt such a method to VLA is beyond our scope (and we already have a better model, Alpamayo 2 Super 34B, as the off-the-shelf baseline compared to MoMAD's original checkpoint). ST-Occ~\citep{leng2025occupancy} is not a driving model and deals with voxels; Orion~\citep{fu2025orion} is based on CARLA dataset, while our work uses real-life clips with a different number of cameras. The dataset difference means that we cannot reliably compute ``collision loss'' and ``boundary loss'' in the paper.

\end{document}